\ifdefined\pdfsuppressptexinfo\pdfsuppressptexinfo=-1\fi
\documentclass{article} %
\usepackage{iclr2027_conference,times}
\iclrfinalcopy
\usepackage{amsmath,amsfonts,bm}

\def\eqref#1{equation~\ref{#1}}
\def\1{\bm{1}}

\DeclareMathAlphabet{\mathsfit}{\encodingdefault}{\sfdefault}{m}{sl}
\SetMathAlphabet{\mathsfit}{bold}{\encodingdefault}{\sfdefault}{bx}{n}

\usepackage{amsmath,amssymb,amsthm}
\usepackage{graphicx}
\usepackage{booktabs}
\usepackage{placeins}
\usepackage{multirow}
\usepackage{makecell}
\usepackage{subcaption}
\usepackage{xcolor}
\usepackage{xspace}
\usepackage{etoolbox}
\AtBeginEnvironment{abstract}{\exhyphenpenalty=10000\relax}
\usepackage{tikz}
\usetikzlibrary{positioning,arrows.meta,fit,calc} 
\usepackage{colortbl}
\newcolumntype{B}{>{\columncolor{bestbg}}c}
\newcolumntype{S}{>{\columncolor{secondbg}}c}
\newcolumntype{T}{>{\columncolor{thirdbg}}c}
\makeatletter
\g@addto@macro\normalsize{%
  \setlength{\abovedisplayskip}{4pt plus 1pt minus 1pt}%
  \setlength{\belowdisplayskip}{4pt plus 1pt minus 1pt}%
  \setlength{\abovedisplayshortskip}{2pt}%
  \setlength{\belowdisplayshortskip}{3pt}}
\makeatother
\definecolor{lightgray}{cmyk}{0,0,0,.05}
\newlength\GreyboxOuterVspace
\newlength\GreyboxPadding
\newlength\GreyboxRule
\newcommand{\GreyboxFrameColor}{black}
\newcommand{\greyboxsetup}[2]{%
  \setlength\GreyboxOuterVspace{#1}%
  \setlength\GreyboxPadding{#2}%
}
\newcommand{\greyboxframe}[2]{%
  \setlength\GreyboxRule{#1}%
  \renewcommand{\GreyboxFrameColor}{#2}%
}
\newcommand{\greybox}[1]{%
  \par\addvspace{\GreyboxOuterVspace}%
  \noindent\begingroup
  \setlength{\fboxsep}{\GreyboxPadding}%
  \setlength{\fboxrule}{\GreyboxRule}%
  \fcolorbox{\GreyboxFrameColor}{lightgray}{%
    \parbox{\dimexpr\linewidth-2\fboxsep-2\fboxrule\relax}{#1}%
  }
  \endgroup%
  \par\addvspace{\GreyboxOuterVspace}%
}
\greyboxsetup{2pt}{3pt}
\greyboxframe{0.5pt}{black!40}

\definecolor{chestc}{HTML}{D2795A}
\definecolor{muscc}{HTML}{5B8FB9}
\definecolor{genc}{HTML}{5FB3A1}
\definecolor{costc}{HTML}{C08FB0}
\definecolor{meanc}{HTML}{7A7A7A}
\definecolor{bestbg}{HTML}{FBE7DB}
\definecolor{secondbg}{HTML}{E9F1F8}
\definecolor{thirdbg}{HTML}{F4F8FA}
\newlength{\rbw}

\DeclareRobustCommand{\rankkey}[2]{%
  \begin{tikzpicture}[baseline=-0.5ex]
    \node[inner sep=1.3pt,minimum height=2.4mm,fill=#1,rounded corners=0.3pt,
          font=\scriptsize]{#2};
  \end{tikzpicture}}

\usepackage{hyperref}
\usepackage{url}
\usepackage{listings}

\lstdefinestyle{pythonstyle}{
    language=Python,
    basicstyle=\ttfamily\scriptsize,
    keywordstyle=\color{blue},
    commentstyle=\color{gray},
    stringstyle=\color{teal},
    showstringspaces=false,
    breaklines=true,
    breakatwhitespace=true,
    frame=single,
    rulecolor=\color{black!30},
    numbers=left,
    numberstyle=\tiny\color{gray},
    tabsize=4,
    captionpos=b
}

\newtheorem{proposition}{Proposition}
\newtheorem{corollary}{Corollary}

\newcommand{\name}{QuPID\xspace}
\newcommand{\BfPara}[1]{{\noindent\bf#1.}\xspace}

\title{\name: Quantum Parameter-Efficient \\ Input-Dependent Retrieval Adaptation \\ for Medical RAG}

\author{Hyojun Ahn$^{1}$, Emily Jimin Roh$^{1}$, Soohyun Park$^{2}$, \\
\textbf{Walid Saad$^{3}$, Hyung-Chul Lee$^{4,5}$, Joongheon Kim$^{1,5}$\thanks{Corresponding author: \texttt{joongheon@korea.ac.kr}}} \\
{\normalfont $^{1}$Korea University \quad $^{2}$Sookmyung Women's University \quad $^{3}$Virginia Tech} \\
{\normalfont $^{4}$Seoul National University \quad $^{5}$Seoul National University Hospital}}

\begin{document}

\maketitle
\lhead{Preprint}

\begin{abstract}
Fidelity-based quantum retrieval ranks candidates by the fidelity between query and archive states. Applying a shared input-independent unitary after fixed state encoding leaves that fidelity unchanged, so training the circuit cannot alter the ranking. Quantum parameter-efficient input-dependent retrieval adaptation (QuPID) repairs this by making the circuit input-dependent through data re-uploading and by comparing measurement readouts, vectors of local Pauli expectations, rather than states. The result is a small readout for adapting frozen image features to a local archive with limited data: training simulates the circuit classically, and inference runs on a GPU with fixed learned parameters. We characterize the class as a structured factorization of input-modulated quadratic feature maps, bound the frequency support of its re-uploading channel, and give a parameter-count generalization bound that motivates its small budget. Under a shared frozen backbone and a label-free protocol, QuPID's 60 parameters give higher precision-at-5 (P@5) on ChestX-ray14 and MURA than frozen medical encoders, and than adapters and low-rank adaptation (LoRA) with up to 5.25 million trainable parameters. On ChestX-ray14, the P@5 gain over the frozen encoder is $+0.116$, the lead over retuned adapters is widest at 512 adaptation examples ($+0.040$), and the full-budget margin over an equally compact classical rotation-plane head is $+0.023$ with a 95\% interval excluding zero. Medical imaging is the primary testbed; the pattern recurs on two non-medical benchmarks, in report generation, and under simulated gate noise and finite-shot readout.
\end{abstract}

\section{Introduction}
\label{sec:intro}

Can a 60-parameter readout adapt a frozen retriever as well as adapters with millions of parameters?
Medical case retrieval makes the question concrete.
A pathology occupies a small part of the image, and once the encoder pools it, shared anatomy dominates, so a `Normal' study and an `Early Pneumonia' study can be nearly collinear in a frozen feature space~\citep{10.1145/3534678.3539322,kienitz2022effect}.
Retrieval-augmented generation (RAG) built on that geometry~\citep{10Jiang,3755760,10.1145/3637528.3671470} provides the generator with diagnostically inconsistent evidence, resulting in hallucination~\citep{huang2025survey,tivnan2024hallucination}.
Data protection rules can keep images inside the hospital~\citep{yang2025cloud}, while the local archive may offer few adaptation examples.
Our aim is a model that adapts retrieval geometry with few trainable parameters while retaining competitive ranking quality.
Medical imaging is our primary setting; industrial inspection and fine-grained recognition provide additional tests of adaptation at a small parameter budget (Section~\ref{sec:generality}).

We use a variational quantum circuit to share trainable parameters across measurement features~\citep{benedetti2019parameterized,cikm25roh}.
Here \emph{quantum} denotes a circuit-defined hypothesis class and not a computational-advantage claim, in the sense argued for by \citet{schuld2022advantage}: every circuit in this paper is simulated classically.
For a diagnostic control motivated by state-based representations~\citep{lisnichenko2023quantum,10.1145/3471158.3472253}, we amplitude-encode query and database features, apply one shared trainable circuit $U(\boldsymbol{\theta})$, and rank by fidelity $|\langle\psi_q|\psi_p\rangle|^{2}$.
The shared-unitary identity rules out adaptation here.
Unitarity cancels the circuit exactly, so the score equals the squared inner product of the frozen features for every $\boldsymbol{\theta}$ and the contrastive gradient is identically zero (Proposition~\ref{prop:degeneracy}).
Minibatch variation does not imply a circuit-dependent ranking.

The proof also identifies how to break this invariance, since the score depends on the two states only through an overlap that any shared input-independent isometry preserves.
Either the map after encoding depends on the input, or the comparison passes through a non-unitary readout, and \name (quantum parameter-efficient input-dependent retrieval adaptation) does both: \emph{data re-uploading} layers~\citep{perez2020data} make the transformation input-dependent, and similarity compares \emph{measurement readouts}, vectors of local Pauli expectations, rather than states.

This construction gives a trainable similarity function; its practical value depends on the geometry produced by the shared parameterization.
We claim neither dimensional expansion, since at $n_q=\log_2 d$ qubits the state space has exactly the input dimension, nor hardness of simulation, since a shallow circuit with local observables is the regime this literature expects to be classically simulable.
Our contribution is a structured retrieval feature map: $60$ shared parameters jointly control $M=40$ measurement outputs through norm-preserving gates, local Pauli observables and input-dependent re-uploading. The analysis characterizes the resulting quadratic forms and a controlled re-uploading spectrum (Proposition~\ref{prop:spectrum}).
Proposition~\ref{prop:generalization} gives a generic parameter-count bound for bounded smooth classes on a compact parameter domain.
We compare frozen medical encoders, scaled adapters, low-rank adaptation (LoRA) and parameter-matched controls on ChestX-ray14 and MURA with multi-seed statistics.
With 60 parameters, \name attains higher precision-at-5 (P@5) than adapters and LoRA with up to 5.25M parameters, and its lead over retuned adapters is widest with 512 adaptation examples; an equally compact classical rotation-plane head narrows the gap to a full-budget margin whose paired interval excludes zero.
Non-medical retrieval, report generation, simulated noise and finite-shot readout assess its scope, alongside cost and metric-specific losses.

\section{Related Work}
\label{sec:related}

\BfPara{Retrieval-augmented generation in medicine}
RAG reduces hallucination in medical large vision-language models (LVLMs) by conditioning generation on context retrieved with pre-trained vision encoders~\citep{wu-etal-2025-medical,10.1145/3711896.3737432,Bain_2021_ICCV,Liu_2021_ICCV}, but flat Euclidean embeddings entangle clinically distinct, visually similar cases~\citep{kienitz2022effect,wang2025large}.
Domain-specialized encoders such as BiomedCLIP~\citep{zhang2023biomedclip} and MedCLIP~\citep{wang2022medclip} are trained centrally and do not address data-local adaptation; we use them as frozen references and as a \name backbone.

\BfPara{Parameter-efficient adaptation}
Adapters and low-rank updates such as LoRA~\citep{hu2022lora} are standard parameter-efficient alternatives. We compare them at scaled budgets alongside classical heads near $60$ parameters.
Quantum circuits have entered parameter-efficient adaptation on the weight side, where Quantum-PEFT builds a full-rank update from a Pauli-structured unitary~\citep{koikeakino2025quantumpeft}, quantum parameter adaptation generates LoRA weights from a circuit during training~\citep{liu2025qpa} and QuanTA factorizes a high-rank update into circuit-shaped tensors~\citep{chen2024quanta}, and on the activation side, where quantum-amplitude embedded adaptation compresses attention activations into amplitudes and transforms them with a circuit~\citep{cikm25roh}.
None of these methods places the circuit in a retriever's similarity function, so none faces the degeneracy of Section~\ref{sec:degeneracy}.
In \name the data are amplitude-encoded and re-uploaded, and the circuit defines a retrieval feature map sharing parameters across its measurement outputs.
Like the weight-side methods, it uses no quantum hardware (Appendix~\ref{app:qml_trends}).
Random features~\citep{rahimi2007random} approximate kernels and supply a compact classical nonlinear control.

\BfPara{Quantum machine learning for representation}
Amplitude embedding encodes $d$-dimensional vectors in $\log_2 d$ qubits~\citep{schuld2019quantum,gonzalez2024efficient}, parameterized circuits define trainable transformations on the encoded states~\citep{benedetti2019parameterized,haug2021capacity}, and data re-uploading interleaves encoding and trainable layers, which controls the Fourier spectrum accessible to circuit outputs~\citep{perez2020data,schuld2021effect}.
Reading out classical features through measurements connects our design to projected quantum kernels~\citep{huang2021power} and to shadow-style models that learn from local observables~\citep{li2021vsql}.
\name builds on measurement-based learning but targets label-free cosine retrieval: its shared local readout is trained through a contrastive objective, and Section~\ref{sec:degeneracy} shows why state overlap cannot adapt under an input-independent shared unitary; medical retrieval applications \mbox{remain} sparse~\citep{WEI202342}.
This literature also supplies the standard by which our claims should be judged: variational models can be approximated by classical surrogates that sample their Fourier spectrum~\mbox{\citep{landman2023classically}}, tuned classical baselines outperform quantum models on most benchmarks studied so far~\citep{bowles2024better}, and shallow local circuits are, where understood, classically simulable~\citep{cerezo2025provable}.
We adopt all three as constraints: no claim rests on hard simulation, every comparison shares the backbone, and parameter-matched, budget-retuned and wall-clock-matched controls are reported alongside.

\section{Method}
\label{sec:method}

\begin{figure}[t]
\centering
\includegraphics[width=\linewidth]{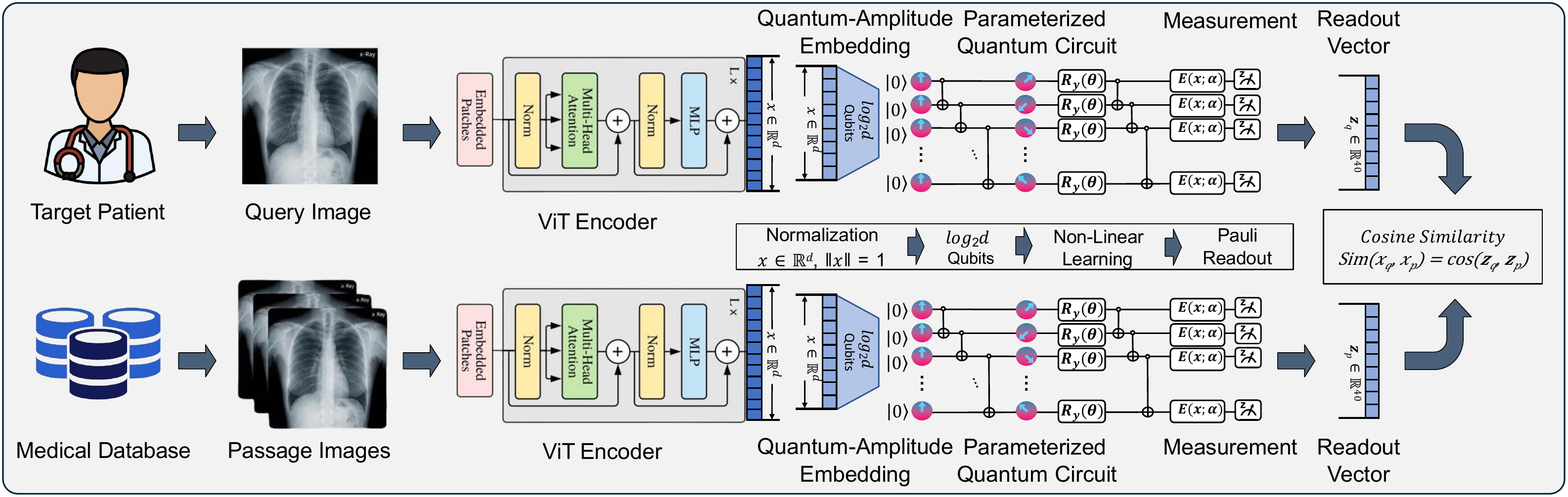}
\caption{\name overview. A frozen vision transformer (ViT) encodes all images; each feature is amplitude-encoded into $\log_2 d$ qubits and processed by a shared circuit (trainable rotations, CNOT ring, re-uploading gates $E(\mathbf{x};\boldsymbol{\alpha})$) with all $60$ trainable parameters. Retrieval compares Pauli readouts $\mathbf{z}\in\mathbb{R}^{40}$ by cosine similarity, not quantum states (Section~\ref{sec:degeneracy}); archive readouts are cached.}
\label{fig:overview}
\vspace{-3mm}
\end{figure}

\subsection{Setup and a degeneracy pitfall in quantum retrieval}
\label{sec:degeneracy}

Let $\mathcal{D}=\{(\mathbf{x}_i,r_i)\}_{i=1}^{N}$ be a hospital archive of images with diagnostic reports, let a frozen ViT produce features $\mathbf{h}=\mathrm{ViT}(\mathbf{x})\in\mathbb{R}^{d}$, $d=1024$, and let conventional RAG rank by cosine similarity to the query feature.
We seek a lightweight transformation of this geometry, learned from unlabeled local data, under which pathologically consistent pairs receive higher scores than pairs that merely share anatomy, which Figure~\ref{fig:overview} summarizes end to end.

Amplitude embedding maps the normalized feature $\hat{\mathbf{h}}$ to an $n_q$-qubit state $|\psi_{\mathrm{in}}(\mathbf{x})\rangle=\sum_{k=0}^{d-1}\hat{h}_k|k\rangle$ with $n_q=\log_2 d$ (Figure~\ref{fig:amplitude}; background in Appendix~\ref{app:qml_background}, implementation in Appendix~\ref{app:implementation}).
The natural variational design applies a trainable circuit $U(\boldsymbol{\theta})$ to both query and database states and ranks by fidelity:
\begin{equation}
\label{eq:naive}
\mathrm{Sim}_{\mathrm{fid}}(\mathbf{x}_q,\mathbf{x}_p;\boldsymbol{\theta})
=\big|\langle\psi_{\mathrm{in}}(\mathbf{x}_q)|U(\boldsymbol{\theta})^{\dagger}U(\boldsymbol{\theta})|\psi_{\mathrm{in}}(\mathbf{x}_p)\rangle\big|^{2}.
\end{equation}

\greybox{%
\begin{proposition}[Degeneracy of shared-unitary fidelity retrieval]
\label{prop:degeneracy}
For any input-independent unitary $U(\boldsymbol{\theta})$, \eqref{eq:naive} satisfies $\mathrm{Sim}_{\mathrm{fid}}(\mathbf{x}_q,\mathbf{x}_p;\boldsymbol{\theta})=|\langle\psi_{\mathrm{in}}(\mathbf{x}_q)|\psi_{\mathrm{in}}(\mathbf{x}_p)\rangle|^{2}$ for all $\boldsymbol{\theta}$. Thus \eqref{eq:naive} induces squared-cosine ranking on the amplitude-encoded frozen features.
Any loss whose only dependence on $\boldsymbol{\theta}$ is through these similarities is constant in $\boldsymbol{\theta}$ and has identically zero gradient.
\end{proposition}
}

The proof is immediate from $U^{\dagger}U=I$ (Appendix~\ref{app:proofs}, verified numerically).
The result extends to every similarity that depends on the states only through their overlap, including the Fubini--Study distance, the trace distance and monotone functions of fidelity (Corollary~\ref{cor:overlap}).
Changing the metric therefore cannot resolve the degeneracy.
Fidelity yields squared-cosine rather than cosine ranking in general; the two orders coincide in our evaluated rank ranges because the query-gallery cosines there are nonnegative (Appendix~\ref{app:proofs}).
This is a diagnostic identity, not a new expressivity theorem or evidence that an established retriever has this defect; Appendix~\ref{app:degeneracy_numeric} tests the deliberately constructed control.
Proposition~\ref{prop:degeneracy} also identifies the required modification: the circuit must depend on the input beyond the initial encoding, or the similarity must pass through a non-unitary map such as measurement. \name does both.

\subsection{Amplitude encoding with data re-uploading}
\label{sec:encoding}

\name processes the encoded state with $L$ blocks, each consisting of a trainable rotation layer, an entangling layer, and a data re-uploading layer:
\begin{equation}
\label{eq:circuit}
|\psi(\mathbf{x};\boldsymbol{\theta},\boldsymbol{\alpha})\rangle
=\prod\nolimits_{l=1}^{L}\Big[E\big(\mathbf{x};\boldsymbol{\alpha}^{(l)}\big)\,U_{\mathrm{ent}}\,U_{\mathrm{rot}}\big(\boldsymbol{\theta}^{(l)}\big)\Big]\,|\psi_{\mathrm{in}}(\mathbf{x})\rangle,
\end{equation}
where $U_{\mathrm{rot}}(\boldsymbol{\theta}^{(l)})=\bigotimes_{j=1}^{n_q}R_Y(\theta_j^{(l)})$ applies one trainable $R_Y$ rotation per qubit, $U_{\mathrm{ent}}$ is a ring of CNOT gates introducing correlations between qubits, and the re-uploading layer $E(\mathbf{x};\boldsymbol{\alpha}^{(l)})=\bigotimes_{j=1}^{n_q}R_Y\big(\alpha_j^{(l)}s_j(\mathbf{x})\big)$ re-injects the input through fixed block-pooled projections of the feature vector, scaled by trainable coefficients $\alpha_j^{(l)}$.
The fixed projections carry a per-qubit dyadic factor, $s_j(\mathbf{x})=2^{\,j-1}\bar{s}_j(\mathbf{x})$ with $\bar{s}_j$ the pooled block mean, so that the encoding matches the premise of the spectral analysis in Section~\ref{sec:readout}.
All gate matrices are real, so the state remains real-valued.
With $n_q$ qubits and $L$ blocks the trainable budget is $2n_qL$; the default ($n_q=10$, $L=3$) has $60$ parameters.
The measurement readout of Section~\ref{sec:readout} alone escapes the degeneracy of Proposition~\ref{prop:degeneracy}; re-uploading additionally enriches the function class (Proposition~\ref{prop:spectrum}); Section~\ref{sec:ablation} separates the two.

\subsection{Measurement readout and similarity}
\label{sec:readout}

Rather than comparing states, \name compares measurement statistics.
We define the readout map
\begin{equation}
\label{eq:readout}
\mathbf{z}(\mathbf{x};\boldsymbol{\theta},\boldsymbol{\alpha})
=\Big(\langle\psi|Z_j|\psi\rangle,\;\langle\psi|X_j|\psi\rangle,\;\langle\psi|Z_jZ_{j+1}|\psi\rangle,\;\langle\psi|X_jX_{j+1}|\psi\rangle\Big)_{j=1}^{n_q}\in\mathbb{R}^{M},
\end{equation}
with indices taken cyclically, giving $M=4n_q=40$ components from single-qubit and nearest-neighbor Pauli observables, and writing $|\psi\rangle\equiv|\psi(\mathbf{x};\boldsymbol{\theta},\boldsymbol{\alpha})\rangle$ for brevity.
Retrieval similarity is the cosine between readouts, $\mathrm{Sim}_{\mathrm{qt}}(\mathbf{x}_q,\mathbf{x}_p)=\mathbf{z}(\mathbf{x}_q)^{\top}\mathbf{z}(\mathbf{x}_p)/(\|\mathbf{z}(\mathbf{x}_q)\|\,\|\mathbf{z}(\mathbf{x}_p)\|)$.
Each component of $\mathbf{z}$ has the form $\langle\psi_{\mathrm{in}}|V^{\dagger}P V|\psi_{\mathrm{in}}\rangle$ for the input-dependent circuit $V$ and a Pauli observable $P$; we suppress the dependence of $\mathrm{Sim}_{\mathrm{qt}}$ on $(\boldsymbol{\theta},\boldsymbol{\alpha})$.
Without re-uploading each component is exactly a quadratic form $\hat{\mathbf{h}}^{\top}A_m(\boldsymbol{\theta})\hat{\mathbf{h}}$; with re-uploading, $A_m(\boldsymbol{\theta},\boldsymbol{\alpha},\mathbf{x})$ is additionally modulated by trigonometric functions of linear projections of the input, so the readout consists of \emph{input-modulated} quadratic forms.
The $M$ readouts share one circuit, and each realized $A_m$ is orthogonally similar to a Pauli observable, hence full-rank with $\pm1$ spectrum.
An unrestricted rank-$r$ factorization uses $O(dr)$ parameters per feature and is full-rank only at $r=d$. Here $60$ parameters jointly specify $M$ structured forms; full rank alone does not establish an expressivity advantage.
The design shares a small parameter vector across correlated, input-modulated quadratic features.
The circuit preserves the state norm, while re-uploading and measurement let the similarity change: this enables training of the design.
Proposition~\ref{prop:generalization} uses bounded outputs and parameter derivatives.

Re-uploading also controls the readout's functional richness.
We specialize the encoding-dependent Fourier analysis of \citet{schuld2021effect} to a clamped-state scalar probe of our re-uploading channel~\citep{perez2020data}; Appendix~\ref{app:proofs} gives the derivation.

\greybox{%
\begin{proposition}[Spectral expressivity of the re-uploading channel]
\label{prop:spectrum}
Clamp the amplitude-encoded state at a reference input and let a scalar probe $t$ drive the re-uploaded angles $s_j=2^{\,j-1}t$.
For fixed $\boldsymbol{\alpha}$ every readout component of \eqref{eq:readout} is a finite trigonometric polynomial in $t$ whose frequencies are signed sums of the rates $\{\alpha_j^{(l)}2^{\,j-1}\}_{j,l}$.
At $\alpha_j^{(l)}=1$ those frequencies are integers and every $\omega$ in the support satisfies $|\omega|\leq L(2^{n_q}-1)$.
\end{proposition}
}

At $\boldsymbol{\alpha}=1$ this conservative bandwidth bound grows exponentially in $n_q$; local readout gives a tighter bound, and boundary attainment is not asserted (Appendix~\ref{app:proofs}).
Three qualifications delimit the scope of this result.
The statement concerns a controlled probe of the re-uploading channel, not the end-to-end map, since the amplitude-encoded state varies with the same input.
Reaching a frequency is not controlling it: only $2n_qL$ coefficients are independently adjustable, and classical product-form features also reach exponential support.
No lower bound against classical surrogates follows, so Section~\ref{sec:experiments} tests parameter-matched controls rather than inferring superiority from the spectrum.

\subsection{Label-free contrastive adaptation}
\label{sec:training}

Site adaptation must proceed without expert labels, so we adopt SimCLR-style contrastive learning~\citep{chen2020simple}: two pathology-preserving augmentations of an image form a positive pair, in-batch samples act as negatives, and for views $i$ with positives $i^{+}$,
\begin{equation}
\label{eq:loss}
\mathcal{L}(\boldsymbol{\theta},\boldsymbol{\alpha})
=-\frac{1}{2B}\sum_{i=1}^{2B}\log
\frac{\exp\big(\mathrm{Sim}_{\mathrm{qt}}(\mathbf{x}_i,\mathbf{x}_{i^{+}})/\tau\big)}
{\sum_{j\neq i}\exp\big(\mathrm{Sim}_{\mathrm{qt}}(\mathbf{x}_i,\mathbf{x}_j)/\tau\big)},
\end{equation}
with temperature $\tau=0.07$ and conservative, lesion-preserving augmentations (Appendix~\ref{app:implementation}).
Every gate generator in \eqref{eq:circuit} has two eigenvalues, so gate-angle derivatives obey exact parameter shift; derivatives of re-uploading scales include $s_j(\mathbf{x})$, with the classical chain rule through cosine and loss (Appendix~\ref{app:implementation}).
In-batch negatives can include false negatives when two patients share a pathology; Appendix~\ref{app:implementation} quantifies the collision rate.

\greybox{%
\begin{proposition}[Parameter-count generalization bound]
\label{prop:generalization}
Fix $R,c>0$, let $p=2n_qL$, and let $\Theta_{R,c}$ collect the $(\boldsymbol{\theta},\boldsymbol{\alpha})\in[-R,R]^{p}$ for which $\|\mathbf{z}(\mathbf{x})\|\geq c$ almost surely.
On $\Theta_{R,c}$ the loss \eqref{eq:loss} is bounded and $L_{\mathrm{lip}}$-Lipschitz in the parameter $\infty$-norm, with $L_{\mathrm{lip}}$ polynomial in $p$, $M$, $1/c$, $1/\tau$, and $\sup_{j,\mathbf{x}}|s_j(\mathbf{x})|$.
With probability $1-\delta$ over $n$ adaptation samples, every empirical risk minimizer over this class has generalization gap $O\big(\sqrt{(\,p\log(1+nRL_{\mathrm{lip}})+\log(1/\delta)\,)/n}\big)$.
The leading rate is $\sqrt{p/n}$, with $M$ and the input scales entering only logarithmically, and the same conclusion holds for any $p$-parameter class of the same boundedness and smoothness.
\end{proposition}
}

The proof (Appendix~\ref{app:genproof}) is a covering argument enabled by the structure: Pauli expectations are bounded and parameter-shift bounds every coordinate derivative.
The rate is generic to any $p$-parameter class with the same smoothness, norm floor and compact domain; it bounds contrastive risk rather than retrieval accuracy; and the reported training does not impose that box constraint.
We therefore interpret it as motivating a capacity hypothesis, rather than providing a guarantee for the reported runs.
Section~\ref{sec:experiments} evaluates the retrieval quality attained by this parameterization at a small trainable budget. Larger adapters and equally compact classical heads provide distinct comparisons.

Neither proposition establishes trainability~\citep{gilfuster2025trainability}, and circuit gradients can vanish exponentially in depth and qubit count~\citep{mcclean2018barren}.
\name avoids that regime by design: $L\leq6$, $n_q$ fixed at 10 by the feature dimension, and local one- and two-body observables, a setting associated with milder shallow-circuit decay~\citep{cerezo2021cost,basheer2025trainability}.
Those guarantees do not cover entangled amplitude-encoded inputs, so Figure~\ref{fig:diagnostics} measures gradient variance from one to twelve layers (four times the default) as a diagnostic, not an asymptotic rate.

\BfPara{Deployment}
After classical circuit-simulation training, $\boldsymbol{\theta}$ and $\boldsymbol{\alpha}$ are fixed. GPU inference computes each query readout, compares it with cached archive readouts, and passes the top-$k$ reports to a frozen medical LVLM.
The query-dependent re-uploading gates are evaluated on the GPU using $1024$ amplitudes; no quantum processor is required. On-site execution reduces data movement, not formal privacy risk (Appendix~\ref{app:qml_code}).

\section{Experiments}
\label{sec:experiments}

\BfPara{Setup}\phantomsection\label{sec:setup}
Retrieval is evaluated on \emph{ChestX-ray14}~\citep{wang2017chestxray} and \emph{MURA}~\citep{rajpurkar2017mura}, generation on \emph{IU X-Ray}~\citep{demner2016preparing} and \emph{MIMIC-CXR}~\citep{johnson2019mimic}, with patient-disjoint medical splits. Appendix~\ref{app:setup} gives split counts, exclusions and metric definitions.
Every adapted method starts from the same ViT-L/16 and shares augmentations, temperature and schedule: adapters, Linear, Tiny MLP, RFF and Quadratic use frozen features, while LoRA updates attention projections with the base weights frozen. Section~\ref{sec:ablation} adds a 60-parameter rotation-plane control. 
Training uses the label-free objective of \eqref{eq:loss}; labels enter only post-hoc evaluation, never model selection or tuning, with default hyperparameters fixed on a disjoint pilot subset; the separate retuning sweep selects by contrastive validation loss.
The five-seed means overlap (Figure~\ref{fig:sensitivity}); $^{\dagger}$ marks Holm-corrected query-bootstrap significance conditional on fitted runs. Appendix~\ref{app:variance} separately reports normal-approximation intervals against the rotation-plane head across adaptation draws, initialization and patients.
\begin{table}[t]
\centering
\caption{Retrieval under the shared label-free protocol (mean$\pm$std, 5 seeds); frozen encoders are deterministic. Within ViT-L, \textbf{bold}/\underline{underline} mark best/second; tints mark \rankkey{bestbg}{1st}, \rankkey{secondbg}{2nd}, \rankkey{thirdbg}{3rd}. $^{\dagger}$ denotes conditional query-bootstrap significance versus the strongest baseline in this table (Holm-corrected, $p<0.05$; Section~\ref{sec:setup}). The final medical-backbone block is a separate comparison.}
\label{tab:main}
\footnotesize
\setlength{\tabcolsep}{2.1pt}
\renewcommand{\arraystretch}{0.93}
\begin{tabular}{l|ccc|ccc|r}
\toprule
\multirow{2}{*}{\textbf{Method}} & \multicolumn{3}{c|}{\textbf{ChestX-ray14}} & \multicolumn{3}{c|}{\textbf{MURA}} & \multirow{2}{*}{\textbf{\#Params}} \\
\cmidrule(lr){2-4}\cmidrule(lr){5-7}
 & P@5 & MAP@10 & NDCG@10 & P@5 & MAP@10 & NDCG@10 & \\
\midrule
ViT Only        & 0.312 & 0.278 & 0.307 & 0.481 & 0.451 & 0.475 & 0 \\
BiomedCLIP      & 0.352 & 0.315 & 0.348 & 0.508 & 0.476 & 0.501 & 0 \\
MedCLIP         & 0.341 & 0.303 & 0.334 & 0.494 & 0.463 & 0.488 & 0 \\
\midrule
Linear Head     & 0.378{\tiny$\pm$.008} & 0.341{\tiny$\pm$.007} & 0.368{\tiny$\pm$.008} & 0.549{\tiny$\pm$.007} & 0.513{\tiny$\pm$.008} & 0.541{\tiny$\pm$.007} & 1.05M \\
Adapter         & 0.401{\tiny$\pm$.007} & 0.368{\tiny$\pm$.006} & 0.396{\tiny$\pm$.007} & 0.573{\tiny$\pm$.006} & 0.546{\tiny$\pm$.007} & 0.567{\tiny$\pm$.006} & 525K \\
\rowcolor{thirdbg}
Adapter-L       & 0.412{\tiny$\pm$.006} & 0.372{\tiny$\pm$.007} & 0.401{\tiny$\pm$.006} & 0.581{\tiny$\pm$.006} & 0.556{\tiny$\pm$.006} & 0.580{\tiny$\pm$.007} & 2.10M \\
\rowcolor{secondbg}
Adapter-XL      & \underline{0.414}{\tiny$\pm$.009} & \underline{0.387}{\tiny$\pm$.008} & \underline{0.414}{\tiny$\pm$.008} & \underline{0.584}{\tiny$\pm$.008} & \underline{0.562}{\tiny$\pm$.008} & \underline{0.584}{\tiny$\pm$.008} & 5.25M \\
LoRA ($r{=}4$)  & 0.405{\tiny$\pm$.008} & 0.370{\tiny$\pm$.008} & 0.399{\tiny$\pm$.008} & 0.576{\tiny$\pm$.007} & 0.549{\tiny$\pm$.007} & 0.571{\tiny$\pm$.007} & 393K \\
LoRA ($r{=}16$) & 0.411{\tiny$\pm$.007} & 0.376{\tiny$\pm$.007} & 0.405{\tiny$\pm$.007} & 0.582{\tiny$\pm$.006} & 0.557{\tiny$\pm$.007} & 0.579{\tiny$\pm$.006} & 1.57M \\
Tiny MLP        & 0.371{\tiny$\pm$.010} & 0.337{\tiny$\pm$.009} & 0.365{\tiny$\pm$.010} & 0.545{\tiny$\pm$.009} & 0.512{\tiny$\pm$.009} & 0.538{\tiny$\pm$.009} & 16.4K \\
RFF             & 0.319{\tiny$\pm$.005} & 0.283{\tiny$\pm$.005} & 0.312{\tiny$\pm$.005} & 0.489{\tiny$\pm$.005} & 0.458{\tiny$\pm$.005} & 0.483{\tiny$\pm$.005} & 64 \\
Quadratic       & 0.392{\tiny$\pm$.009} & 0.355{\tiny$\pm$.009} & 0.383{\tiny$\pm$.009} & 0.561{\tiny$\pm$.008} & 0.532{\tiny$\pm$.008} & 0.554{\tiny$\pm$.008} & 164 \\
\rowcolor{bestbg}
\textbf{\name}  & \textbf{0.428}$^{\dagger}${\tiny$\pm$.006} & \textbf{0.399}$^{\dagger}${\tiny$\pm$.006} & \textbf{0.425}$^{\dagger}${\tiny$\pm$.006} & \textbf{0.601}$^{\dagger}${\tiny$\pm$.006} & \textbf{0.579}$^{\dagger}${\tiny$\pm$.006} & \textbf{0.599}$^{\dagger}${\tiny$\pm$.006} & \textbf{60} \\
\midrule
BiomedCLIP{+}\name & 0.438{\tiny$\pm$.006} & 0.409{\tiny$\pm$.006} & 0.435{\tiny$\pm$.006} & 0.612{\tiny$\pm$.006} & 0.590{\tiny$\pm$.006} & 0.610{\tiny$\pm$.006} & 54 \\
\bottomrule
\end{tabular}
\vspace{-3mm}
\end{table}

\begin{figure}[t]
\centering
\includegraphics[width=\linewidth]{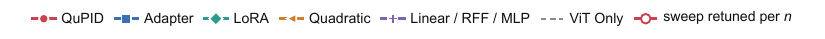}\\[2pt]
\begin{subfigure}[t]{0.32\linewidth}
\centering
\includegraphics[width=\linewidth]{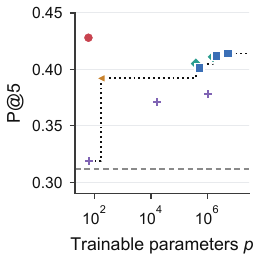}
\caption{Parameter budget.}
\end{subfigure}\hfill
\begin{subfigure}[t]{0.32\linewidth}
\centering
\includegraphics[width=\linewidth]{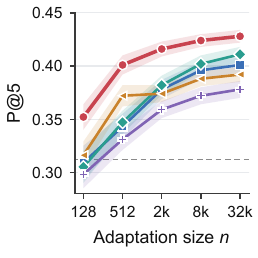}
\caption{Adaptation size.}
\end{subfigure}\hfill
\begin{subfigure}[t]{0.32\linewidth}
\centering
\includegraphics[width=\linewidth]{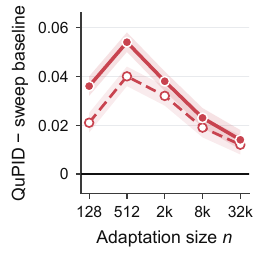}
\caption{Margin in the retuning sweep.}
\end{subfigure}
\caption{Low-data retrieval (ChestX-ray14, 5 seeds; bands one std). The dotted step in (a) is the Pareto envelope of the plotted classical modules; (b) plots the adaptation-size sweep. Panel (c) uses the fixed and retuned comparisons of Table~\ref{tab:tuned}; it excludes the additional controls in Table~\ref{tab:ladder}. Proposition~\ref{prop:generalization} does not predict these retrieval margins.}
\label{fig:lowdata}
\vspace{-3mm}
\end{figure}

\subsection{Main retrieval results}
\label{sec:mainresults}

Table~\ref{tab:main} and Figure~\ref{fig:lowdata}(a) evaluate the accuracy attained at each trainable budget; Figure~\ref{fig:profile} (Appendix~\ref{app:fullmetrics}) gives metric-specific differences. With $60$ parameters, \name exceeds every adapter and LoRA configuration in Table~\ref{tab:main} on both datasets.
ChestX-ray14 MRR is lower.
Relative to the frozen encoder, \name improves P@5 by $+0.116$ and $+0.120$; its measured margins over a five-million-parameter adapter are $+0.014$ and $+0.017$. These results demonstrate useful retrieval adaptation with four to five orders of magnitude fewer trainable parameters than the adapters.
In Table~\ref{tab:ablation}, trained readout recovers $0.100$ of the $0.116$ gain; adding re-uploading contributes $0.016$ while increasing the budget from $30$ to $60$. This ablation changes both structure and capacity.
An untrained circuit achieves half of the gain, and the naive fidelity design matches the frozen baseline, as Proposition~\ref{prop:degeneracy} and the observed nonnegative query-gallery cosines require. \label{sec:geometry}%
ChestX-ray14 label-set Jaccard describes neighborhood structure; on MURA it equals P@10 by definition (Table~\ref{tab:geometry}).

\begin{figure}[t]
\centering
\includegraphics[width=\linewidth]{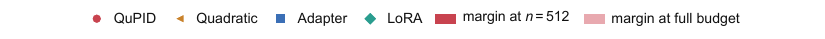}\\[2pt]
\begin{subfigure}[t]{0.32\linewidth}
\centering
\includegraphics[width=\linewidth]{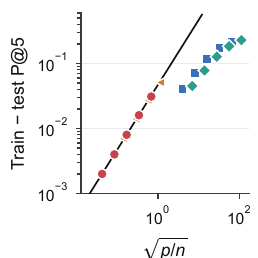}
\caption{Gap against $\sqrt{p/n}$.}
\end{subfigure}\hfill
\begin{subfigure}[t]{0.32\linewidth}
\centering
\includegraphics[width=\linewidth]{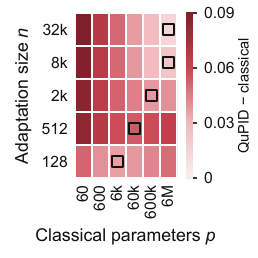}
\caption{Margin across $p$ and $n$.}
\end{subfigure}\hfill
\begin{subfigure}[t]{0.32\linewidth}
\centering
\includegraphics[width=\linewidth]{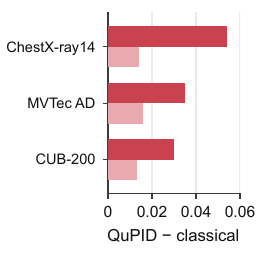}
\caption{Margin in three domains.}
\end{subfigure}
\caption{Locating the advantage (ChestX-ray14 unless noted, 5 seeds). The solid line in (a) is an unfitted slope-one guide, and the squares in (b) mark the adapter budget with the best test score at each $n$. Absolute numbers are in Appendices~\ref{app:fullmetrics} and~\ref{app:whygap}.}
\label{fig:where}
\vspace{-3mm}
\end{figure}

\FloatBarrier
\BfPara{Where the parameter budget stops paying}\phantomsection\label{sec:dataefficiency}
The adaptation-size sweep evaluates the compact map under changing data budgets. Proposition~\ref{prop:generalization} motivates restricting capacity but does not predict an accuracy ordering.
Figures~\ref{fig:lowdata} and~\ref{fig:where} compare the plotted families; Table~\ref{tab:ladder} supplies the closest matched control.
Below roughly one thousand examples adapters degrade toward, and partially below, the frozen baseline. At 512, \name reaches $0.401$ P@5 versus $0.386$ for the default \mbox{60-parameter} rotation-plane head, the strongest control in Table~\ref{tab:ladder}, a $+0.015$ margin.
The same-control full-budget margin is $+0.023$. Full-to-512 drops are $0.027$ for \name and $0.019$ for the control: higher accuracy, but greater degradation.
Appendix~\ref{app:gap} suggests why larger heads fall behind: the adapter's training P@5 exceeds its held-out P@5 by $0.041$ at the full budget and $0.171$ at $512$, against $0.002$ and $0.016$ for \name.
The rotation-plane sweep in Appendix~\ref{app:capacity} measures accuracy and train--test gaps at all five adaptation sizes.
Figure~\ref{fig:where}(b) sweeps parameter budget and data budget jointly: the classical parameter budget achieving the highest test score increases with adaptation size, and our margin is always smallest at those budgets.
Retuning the configurations in Table~\ref{tab:tuned} reduces their 512-example margin from $+0.054$ to $+0.040$ (Figure~\ref{fig:lowdata}(c)). These are comparisons within that sweep, not margins over all classical controls.
Validation-selected rotation-plane tuning reduces the margins to $+0.007$ at $512$ and $+0.014$ at full budget (Appendix~\ref{app:computematch}). Paired $95\%$ intervals exclude zero at the full budget and not at $512$, for the default control ($[-0.0112,+0.0412]$ at $512$) and for the tuned control (Table~\ref{tab:surrogates}).

\BfPara{Does the effect leave radiology}\phantomsection\label{sec:generality}
The readout can process any frozen feature vector satisfying the encoding requirements of Section~\ref{sec:method}. We test whether its utility extends beyond radiology.
Figure~\ref{fig:where}(c) repeats the protocol on two non-medical benchmarks that share the structure of Section~\ref{sec:intro}.
\emph{MVTec AD}~\citep{bergmann2019mvtec} evaluates object-category retrieval in industrial imagery; \emph{\mbox{CUB-200}}~\citep{wah2011caltech} evaluates species retrieval. These relevance definitions test transfer beyond radiology; the MVTec result does not directly measure defect localization.
With backbone, readout and hyperparameters unchanged, the same pattern appears in both, the low-data margin roughly twice the full-budget one (Table~\ref{tab:generality_full}).
These comparisons report mean accuracy for the included baselines; they do not isolate the size of the discriminative signal as the cause.

\BfPara{Does better retrieval reach the generator}\phantomsection\label{sec:generation}
Retrievers are adapted label-free on each generation dataset's training split, which also serves as the archive and is patient-disjoint from all queries; decoding is greedy with a fixed prompt, so seed variation stems from the retriever alone.
Table~\ref{tab:generation} (Appendix~\ref{app:fullmetrics}) shows \name improving lexical-overlap metrics for both generators and datasets, while Adapter retains a small CIDEr advantage.
Because n-gram overlap cannot certify clinical correctness, Table~\ref{tab:clinical} probes factual content: CheXbert-F1 improves over Adapter by $0.018/0.017$ and RadGraph-F1 by $0.011/0.012$ for LLaVA-Med/CheXagent.
The comparison is meaningful only if generation responds to retrieval, so Appendix~\ref{app:utility} brackets the pipeline between random and label-oracle context; \name recovers $61\%$ of this clinical-proxy interval and Adapter $49\%$ (Figure~\ref{fig:utility}).

\begin{figure}[t]
\centering
\includegraphics[width=\linewidth]{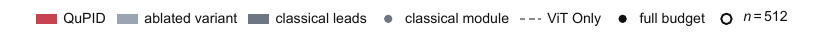}\\[2pt]
\begin{subfigure}[t]{0.32\linewidth}
\centering
\includegraphics[width=\linewidth]{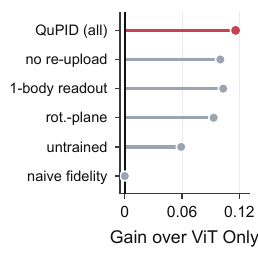}
\caption{Ablations and controls.}
\end{subfigure}\hfill
\begin{subfigure}[t]{0.32\linewidth}
\centering
\includegraphics[width=\linewidth]{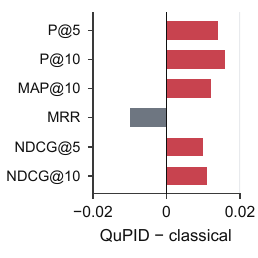}
\caption{Per-metric margin.}
\end{subfigure}\hfill
\begin{subfigure}[t]{0.32\linewidth}
\centering
\includegraphics[width=\linewidth]{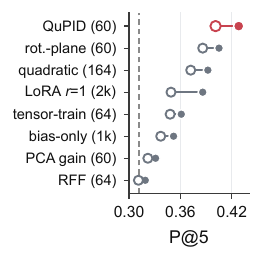}
\caption{Classical capacity ladder.}
\end{subfigure}
\caption{Mechanism (ChestX-ray14, 5 seeds; P@5 except (b)). Gains in (a) are over the frozen encoder; (b) uses the strongest included comparator; (c) gives parameter counts, filled at full budget and hollow at $n{=}512$. Full sets are in Appendices~\ref{app:fullmetrics}, \ref{app:extra_exp} and~\ref{app:whygap}.}
\label{fig:mech}
\vspace{-3mm}
\end{figure}

\FloatBarrier
\subsection{Isolating the mechanism}
\label{sec:ablation}

The comparisons above assess the utility of a compact map. We next examine its design choices and equally small classical alternatives.
These controls distinguish the comparison with larger adapters from comparisons between parameterizations at a similar budget.
The ten-qubit real circuit is simulated exactly with $1024$-dimensional linear algebra, so these comparisons evaluate representational choices within classical simulation~\citep{gilfuster2025trainability}.

The circuit couples 40 measurement features through one norm-preserving transformation with 60 trainable parameters. For each input, the associated quadratic forms share an orthogonal conjugation of fixed Pauli observables. Re-uploading makes this shared transformation input-dependent, tying how the features can adapt. Under classical GPU execution, the circuit thus specifies an architectural constraint whose retrieval utility is evaluated by the parameter-matched controls.

Figure~\ref{fig:mech}(a) removes one component at a time and Table~\ref{tab:ablation} (Appendix~\ref{app:fullmetrics}) gives the full set.
Depth saturates beyond $L=3$; single-qubit and pairwise readouts both contribute. The rotation-plane head matches our budget and isometry but uses input-independent rotations and rank-one squared-projection forms. Its comparison with input-modulated full-rank Pauli forms evaluates the combined design, without separating modulation, readout rank and entanglement (Appendix~\ref{app:implementation}).
The same parameter budget can therefore support different feature families.
A variational readout admits a Fourier description and can be approximated classically by sampling frequencies from its spectrum~\citep{landman2023classically}, but Proposition~\ref{prop:spectrum} bounds our bandwidth without providing a lower bound on the cost of such a surrogate.
RFF and the four-projection Quadratic head are restricted controls, not exhaustive dequantizations; matching parameter count does not match retained information.
A random Fourier map recovers $0.007$ of the $0.116$ that separates the readout from the frozen encoder and a trained quadratic head $0.080$, leaving $0.036$ over Quadratic; the stronger rotation-plane head narrows this margin to $0.023$.
Appendix~\ref{app:surrogates} builds that surrogate from the spectrum of Proposition~\ref{prop:spectrum} at $5{,}000\times$ our budget: its full-budget P@5 trails \name by $0.021$ (paired $95\%$ interval $[+0.013,+0.029]$) and by $0.039$ at $512$ examples.
Figure~\ref{fig:mech}(c) compares classical modules at or near our budget under their default configurations (Appendix~\ref{app:capacity}); the readout has the highest mean P@5 at both adaptation sizes, with the isometric control closest.
The larger-adapter comparison establishes parameter efficiency at competitive accuracy; equal-budget controls assess the specific structure chosen for that budget.

Figure~\ref{fig:mech}(b) shows the per-metric boundary of the claim.
Appendix~\ref{app:extra_exp} varies qubit count, entangling topology, readout width, temperature and feature noise, and every axis saturates at or near the default, so the operating point is a plateau rather than a tuned optimum.
It also reports per-seed spread and a zero-shot cross-site margin over the adapter that stays positive, falling from $0.027$ to $0.017$.

\BfPara{Computational cost}
Table~\ref{tab:cost} (Appendix~\ref{app:fullmetrics}) reports the tradeoff: \name requires $1.60$ times the per-epoch training time of a lightweight adapter, because statevector simulation scales with $2^{n_q}=d$.
Inference (7.6 against 7.1 ms per query) and peak adaptation memory (3{,}582 against 3{,}728 MB) are comparable, since archive readouts are cached and each query requires one circuit evaluation.
Parameter efficiency is not computational efficiency. At matched adaptation time, Table~\ref{tab:computematch} preserves the ordering among its three baselines; Appendix~\ref{app:computematch} separately evaluates rotation-plane tuning with an explicit trial budget.
The GPU deployment of Section~\ref{sec:training} incurs this classical cost; Appendix~\ref{app:qml_code} describes its fixed-parameter forward pass.

\begin{figure}[t]
\centering
\includegraphics[width=\linewidth]{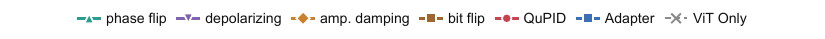}\\[2pt]
\begin{subfigure}[t]{0.32\linewidth}
\centering
\includegraphics[width=\linewidth]{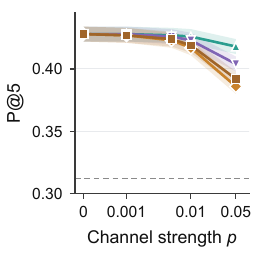}
\caption{Gate noise at evaluation.}
\end{subfigure}\hfill
\begin{subfigure}[t]{0.32\linewidth}
\centering
\includegraphics[width=\linewidth]{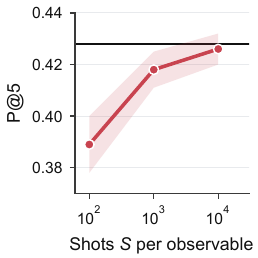}
\caption{Finite-shot readout.}
\end{subfigure}\hfill
\begin{subfigure}[t]{0.32\linewidth}
\centering
\includegraphics[width=\linewidth]{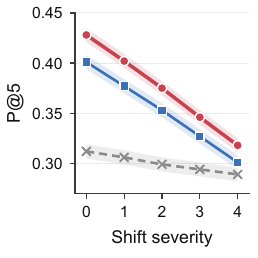}
\caption{Acquisition shift.}
\end{subfigure}
\caption{The idealizations tested (ChestX-ray14, P@5, 5 seeds; bands one standard deviation).
Severity in (c) runs from the ChestX-ray14 acquisition at 0 to the CheXpert~\citep{irvin2019chexpert} acquisition at 4. The strength axis in (a) is a symmetric log; in (b) the horizontal rule marks exact expectations and the points are measured finite-shot results.}
\label{fig:noise}
\vspace{-3mm}
\end{figure}

\FloatBarrier
\BfPara{Testing the idealizations}\phantomsection\label{sec:noise}
GPU deployment uses exact statevector readouts. Figure~\ref{fig:noise} tests acquisition shift and, separately, noise and sampling in a hypothetical quantum-hardware realization; the latter is not required for our GPU inference path.
Gate noise degrades retrieval gradually, and the channel ordering follows the cancellation argument of Appendix~\ref{app:noise_channels}.
The evaluated Pauli channels multiply Pauli components at the channel location by observable-dependent factors, without an additive offset.
Cosine removes a common scaling factor; amplitude damping additionally biases expectations, consistent with its larger measured degradation.
Training through the channel recovers part of the loss but never the noiseless level (Figure~\ref{fig:noiseaware} and Tables~\ref{tab:noise} and~\ref{tab:noise_aware}, Appendix~\ref{app:noise}).
Finite-shot readout is within $0.002$ P@5 of exact expectations at $S=10^{4}$ (Table~\ref{tab:shots}); this is an empirical result, separate from the sufficient ranking guarantee in \eqref{eq:shots}.
Figure~\ref{fig:noise}(c) and Appendix~\ref{app:stress} report positive acquisition-shift margins and less sequential forgetting than either adapter. Figure~\ref{fig:archive} indexes stored cache size; patient exclusion determines the eligible gallery used for ranking.
The scope is narrow: these simulations omit crosstalk, coherent errors and readout error.
They show that the margin survives the modeled perturbations, not that it would survive on hardware today.

\section{Limitations and Conclusion}
\label{sec:limitations}

Five limitations bound our claims.
First, training uses classical circuit simulation and inference runs on a GPU; no computational quantum advantage is claimed.
Second, the method operates on pooled backbone features, so information averaged away by the frozen [CLS] aggregation cannot be recreated downstream, and patch-token variants are future work.
Third, amplitude encoding requires strict L2 normalization and power-of-two dimensions, so a $768$-dimensional backbone needs zero-padding to $1024$.
Fourth, the retriever is trained in isolation from the generator, generation metrics lack clinician assessment, and data-local adaptation is not formal privacy protection.
Fifth, the margin over the $60$-parameter rotation-plane head is positive at every adaptation size; paired intervals exclude zero at the full budget but not at $512$ examples, including after validation-selected tuning.

\name adapts a frozen retriever through $60$ parameters shared across $40$ input-dependent measurement features, and this structure exceeds adapters and LoRA with up to $5.25$M parameters in P@5.
Its P@5 gain over the frozen encoder is $+0.116$, its lead over retuned adapters is widest at $512$ examples ($+0.040$), and its full-budget margin over an equally compact classical head is $+0.023$.
Circuit-defined measurement features are thus a practical route to parameter-efficient retrieval adaptation under classical GPU execution.

\clearpage
\section*{Ethics Statement}
All datasets used in this work are publicly distributed, and the medical datasets are de-identified.
ChestX-ray14 and MIMIC-CXR were collected under IRB approval by their curators. \mbox{MIMIC-CXR} access follows the PhysioNet credentialed data use agreement, MURA and CheXpert are distributed by Stanford under their research use agreements, and IU X-Ray is distributed de-identified through Open-i.
MVTec AD is used under its CC BY-NC-SA 4.0 license, and CUB-200 under the terms of its Caltech distribution.
The proposed system is intended to assist, not replace, clinical decision-making: any deployment must operate under physician supervision, undergo applicable regulatory review, and include mechanisms for flagging low-confidence retrievals.
We explicitly do not claim privacy-preserving learning; data-local adaptation reduces data movement but does not defend against parameter-level attacks such as membership inference on the adapted module.

\section*{Reproducibility Statement}
Proofs of Propositions~\ref{prop:degeneracy}, \ref{prop:spectrum} and~\ref{prop:generalization} are given in Appendix~\ref{app:proofs}.
Implementation details, hyperparameters, augmentation settings, and seed handling are specified in Appendix~\ref{app:implementation}, and all datasets are available from their public distributors under the terms listed in the Ethics Statement.
Code is available at \url{https://anonymous.4open.science/r/qupid-iclr/}.

\section*{AI Use Statement}
We used generative AI tools to edit and restructure the text, verify and format references, write figure and reference code, and assist with proof revision and with drafting interpretive sentences, which the authors checked against the measured results. We did not use them to generate data or to produce any reported experimental value; translation and qualitative analysis are not applicable. The authors reviewed all AI-assisted content and take full responsibility for the paper.

\bibliography{ref}
\bibliographystyle{iclr2027_conference}

\par
\newpage
\appendix
\makeatletter
\setlength{\@fptop}{0pt}
\setlength{\@fpsep}{11pt}
\setlength{\@fpbot}{0pt plus 1fil}
\makeatother

\section{Background on Quantum Machine Learning}
\label{app:qml_background}

This appendix provides a self-contained introduction to \name for readers unfamiliar with gate-based quantum machine learning (QML).
It is not a general introduction to quantum computing; it reviews only the ingredients our retrieval head uses, in the order they appear in Section~\ref{sec:method}: qubits and Hilbert-space representations, amplitude encoding, parameterized circuits with data re-uploading, and measurement readout.
Readers already familiar with these can skip to Appendix~\ref{app:proofs}.

\subsection{Qubits and Hilbert-space representations}
\label{app:qml_qubits}

A classical bit takes one of two values.
A qubit is its quantum analogue, a normalized vector in a two-dimensional complex Hilbert space,
\begin{equation}
|\phi\rangle=\alpha|0\rangle+\beta|1\rangle,\qquad \alpha,\beta\in\mathbb{C},\qquad |\alpha|^{2}+|\beta|^{2}=1,
\end{equation}
where $|0\rangle$ and $|1\rangle$ are computational basis states and $\alpha,\beta$ are probability amplitudes: measuring in the computational basis returns $|0\rangle$ with probability $|\alpha|^{2}$.
For $n_q$ qubits the Hilbert space has dimension $2^{n_q}$ with basis $\{|k\rangle\}_{k=0}^{2^{n_q}-1}$, and a general pure state is a unit-norm superposition $|\phi\rangle=\sum_{k}\alpha_k|k\rangle$ with $\sum_k|\alpha_k|^{2}=1$.

Two consequences matter here.
First, the state dimension grows exponentially in qubit count, which is often cited as motivation for quantum feature maps; we deliberately do \emph{not} rely on that argument, since our configuration sets $n_q=\log_2 d$ so the state dimension equals the input dimension exactly (Section~\ref{sec:readout}).
Second, and more importantly, a measurement extracts only a scalar expectation per observable, so what determines the usable information is the design of encoding, circuit, and readout rather than the raw dimension.

\subsection{Amplitude encoding}
\label{app:qml_amplitude}

Classical data must be written into a quantum state before a circuit can act on it.
Amplitude encoding writes a normalized vector directly into the amplitudes,
\begin{equation}
|\psi_{\mathrm{in}}(\mathbf{x})\rangle=\sum_{k=0}^{d-1}\hat{h}_k|k\rangle,\qquad \hat{\mathbf{h}}=\mathbf{h}/\|\mathbf{h}\|_2 ,
\end{equation}
using $n_q=\log_2 d$ qubits, exponentially fewer than the $O(d)$ qubits of angle encoding, where each coordinate drives its own rotation.
This logarithmic qubit cost is why amplitude encoding is the standard choice when the input is a dense high-dimensional feature, as in our $d=1024$ ViT case.
It also has the following requirements: the vector must be L2-normalized, so overall feature scale is discarded, and $d$ must be a power of two, so other backbones need padding or truncation (Section~\ref{sec:limitations}).
On hardware a third cost would dominate: preparing an arbitrary $d$-dimensional amplitude state exactly takes $O(d)$ elementary gates~\citep{mottonen2005transformation,plesch2011quantum}, a gate cost not incurred in simulation, where encoding consists of normalization.
We state it because every result in this paper is simulated; the encoding cost is the first item a hardware implementation would have to remove, and none of our claims depends on its being removed.

\subsection{Parameterized quantum circuits and data re-uploading}
\label{app:qml_pqc}

A quantum circuit applies a unitary $U$ with $U^{\dagger}U=UU^{\dagger}=I$, preserving the norm of the state; this is exactly the isometry property that Section~\ref{sec:readout} identifies as an inductive bias and that Proposition~\ref{prop:degeneracy} identifies as a failure mode when the readout is an overlap.
In QML the unitary is parameterized by trainable angles, giving a parameterized quantum circuit (PQC), or variational ansatz.
The single-qubit rotation we use is
\begin{equation}
R_Y(\theta)=\begin{bmatrix}\cos(\theta/2) & -\sin(\theta/2)\\ \sin(\theta/2) & \cos(\theta/2)\end{bmatrix},
\end{equation}
which is real-valued, so a circuit built from $R_Y$ and CNOT keeps an initially real state real and admits single-precision real simulation.
Multi-qubit gates such as CNOT create correlations that cannot be written as independent single-qubit operations; we use a ring topology, $\mathrm{CNOT}_{1,2}\cdots\mathrm{CNOT}_{n_q-1,n_q}\mathrm{CNOT}_{n_q,1}$, coupling all qubits with $n_q$ gates per layer rather than the $O(n_q^{2})$ that all-to-all entanglement would need.

A plain PQC applies a fixed unitary to the encoded state. \emph{Data re-uploading}~\citep{perez2020data} instead re-injects the input between trainable layers, so the overall map is no longer input-independent.
This is one of the two mechanisms by which \name escapes the degeneracy of Proposition~\ref{prop:degeneracy}, and it sets the Fourier spectrum the readout can reach (Proposition~\ref{prop:spectrum}).

\subsection{Measurement readout}
\label{app:qml_readout}

A quantum state is not directly observable; classical information is recovered by measuring observables.
For a Hermitian observable $P$ and state $|\psi\rangle$ the expectation is $\langle\psi|P|\psi\rangle$, estimated on hardware by repeated preparation and measurement.
Writing $|\psi\rangle=V|\psi_{\mathrm{in}}\rangle$ for the circuit $V$, this equals $\langle\psi_{\mathrm{in}}|V^{\dagger}PV|\psi_{\mathrm{in}}\rangle$, a quadratic form in the encoded amplitudes with \emph{effective observable} $A=V^{\dagger}PV$.
Since $P$ is Hermitian and $V$ unitary, $A$ is Hermitian and orthogonally similar to $P$, the structural fact behind the full-rank claim of Section~\ref{sec:readout} and the boundedness used in Proposition~\ref{prop:generalization}.

Two readout choices are common, and the distinction is this paper's crux.
A \emph{fidelity} readout compares two states through their overlap $|\langle\psi_q|\psi_p\rangle|^{2}$; a \emph{measurement} readout maps each state separately to a vector of expectations and compares those classically.
The first is the seemingly natural design that Proposition~\ref{prop:degeneracy} proves inert under a shared circuit; the second is what \name uses.
We measure local one- and two-body Pauli observables ($Z_j$, $X_j$, $Z_jZ_{j+1}$, $X_jX_{j+1}$), a locality choice that also keeps gradients well behaved (Section~\ref{sec:training}).

\subsection{Reference implementation}
\label{app:qml_code}

Listing~\ref{lst:qupid_pennylane} gives a compact PennyLane~\citep{bergholm2018pennylane} implementation of the \name readout of \eqref{eq:readout}.
The trainable tensor \texttt{theta} has shape $(L,n_q)$ for rotation angles and \texttt{alpha} the same shape for re-uploading scales, totaling $2n_qL=60$ parameters at the default setting.
The supplementary reference uses a batched statevector implementation checked against independent dense-circuit calculations. After training, $\boldsymbol{\theta}$ and $\boldsymbol{\alpha}$ are fixed, but re-uploading still depends on the query input. Thus the map is not a single input-independent matrix. GPU inference evaluates that input-dependent forward pass and reuses archive readouts cached for the same checkpoint. At $n_q=10$ the statevector has $1024$ amplitudes, so quantum-hardware transfer is unnecessary for this deployment; the argument does not extend to arbitrary larger circuits.

\begin{lstlisting}[style=pythonstyle, caption={Compact PennyLane implementation of the \name measurement readout.}, label={lst:qupid_pennylane}]
import pennylane as qml
def qupid_layer(theta_l, alpha_l, s_x, wires):
    """One QuPID block: trainable RY, ring CNOT, data re-uploading."""
    for q, w in enumerate(wires):                 # trainable rotation layer
        qml.RY(theta_l[q], wires=w)
    for q in range(len(wires) - 1):               # ring entanglement
        qml.CNOT(wires=[wires[q], wires[q + 1]])
    qml.CNOT(wires=[wires[-1], wires[0]])
    for q, w in enumerate(wires):                 # data re-uploading layer
        qml.RY(alpha_l[q] * s_x[q], wires=w)

def make_qupid_qnode(n_qubits, depth, noise=None):
    wires = list(range(n_qubits))
    # default.qubit is the noiseless statevector simulator used for every
    # reported result; default.mixed enables the noise-simulation appendix.
    dev = qml.device("default.qubit" if noise is None else "default.mixed",
                     wires=n_qubits)
    @qml.qnode(dev)
    def circuit(h, theta, alpha, s_x):
        qml.AmplitudeEmbedding(features=h, wires=wires,
                               normalize=True, pad_with=0.0)
        for l in range(depth):
            qupid_layer(theta[l], alpha[l], s_x, wires)
            if noise is not None:
                apply_noise_channel(noise, wires) # noise-simulation appendix
        # M = 4 n_q local Pauli expectations
        obs  = [qml.PauliZ(w) for w in wires]
        obs += [qml.PauliX(w) for w in wires]
        obs += [qml.PauliZ(wires[q]) @ qml.PauliZ(wires[(q + 1) % n_qubits])
                for q in range(n_qubits)]
        obs += [qml.PauliX(wires[q]) @ qml.PauliX(wires[(q + 1) % n_qubits])
                for q in range(n_qubits)]
        return [qml.expval(o) for o in obs]
    return circuit
\end{lstlisting}

\subsection{Positioning among recent variational quantum learning studies}
\label{app:qml_trends}

Table~\ref{tab:qml_settings} places \name among recent variational quantum learning papers at major machine learning venues along three axes: whether the architecture is hybrid quantum-classical, how classical data enters the circuit, and whether evaluation is simulator-based.
Three patterns are visible.
Recent QML systems are predominantly hybrid, using circuits as compact modules inside classical pipelines rather than as standalone models; encoding is task-specific, with angle encoding most common; and all listed studies use simulation, since full learning pipelines remain impractical on current hardware.

The two parameter-efficient fine-tuning entries are the closest precedent for our execution protocol.
Quantum-PEFT describes its Pauli-parameterized unitaries as quantum-inspired modules, computes them as products of small gate matrices, and runs its experiments on NVIDIA GPUs~\citep{koikeakino2025quantumpeft}.
QPA simulates its circuits exactly during training so that inference needs no quantum hardware, and it assesses device noise with noise models of IBM processors while leaving execution on quantum hardware to future work~\citep{liu2025qpa}.
QuanTA, the third weight-side method of Section~\ref{sec:related}, is a quantum-informed tensor adaptation that also runs on classical hardware~\citep{chen2024quanta}.
Our claims likewise concern the model class and its retrieval quality, so a hardware run would test device fidelity rather than these claims; the encoding and sampling costs it would add are given in Appendices~\ref{app:qml_amplitude} and~\ref{app:noise_shots}.

\name follows these conventions and the execution protocol of the two fine-tuning methods, training a compact measured feature map by classical simulation and running its fixed-parameter readout on a GPU, with emphasis on parameter sharing and empirical retrieval utility. As in QPA, noise enters only as a simulated sensitivity study (Appendix~\ref{app:noise}); it does not validate physical-device robustness. The table places our design alongside quantum modules in reinforcement learning, vision, operator-learning, federated and language pipelines. Here the circuit learns the \emph{similarity} features of a frozen retrieval pipeline. Shared-unitary invariance diagnoses a non-trainable fidelity construction; the proposed map combines input dependence and measurements to enable adaptation.

\begin{table}[!ht]
\centering
\caption{Experimental settings in recent variational quantum learning studies at major venues. A checkmark indicates the setting is explicitly used or discussed. Hybrid architectures, task-specific encodings, and simulator-based evaluation are the prevailing conventions; \name follows all three.}
\label{tab:qml_settings}
\scriptsize
\setlength{\tabcolsep}{3pt}
\begin{tabular}{llllcc}
\toprule
\textbf{Work} & \textbf{Venue} & \textbf{Domain} & \textbf{Encoding} & \textbf{Hybrid} & \textbf{Simulator} \\
\midrule
TensorRL-QAS~\citep{NeurIPS2025TensorRL} & NeurIPS & RL & Matrix product state & \checkmark & \checkmark \\
QVF~\citep{NeurIPS2025Quantum} & NeurIPS & Vision & Amplitude & \checkmark & \checkmark \\
QDSFormer~\citep{NeurIPS2025QDSFormer} & NeurIPS & Vision & Angle & \checkmark & \checkmark \\
PQC policies~\citep{jerbi2021parametrized} & NeurIPS & RL & Angle, re-uploading & \checkmark & \checkmark \\
QuanONet~\citep{ICML2025QuanOnet} & ICML & Operator learning & Angle, re-uploading &  & \checkmark \\
Quorus~\citep{ICLR2026Quorus} & ICLR & Federated & Angle & \checkmark & \checkmark \\
eQMARL~\citep{ICLR2025eQMARL} & ICLR & RL & Angle & \checkmark & \checkmark \\
VSQL~\citep{li2021vsql} & AAAI & Vision & Amplitude, shadow & \checkmark & \checkmark \\
Quantum-PEFT~\citep{koikeakino2025quantumpeft} & ICLR & PEFT (LLM) & None (weights only) & \checkmark & \checkmark \\
QPA~\citep{liu2025qpa} & ICLR & PEFT (LLM) & None (weights only) & \checkmark & \checkmark \\
\midrule
\textbf{\name (ours)} & n/a & \textbf{Retrieval} & Amplitude, re-uploading & \checkmark & \checkmark \\
\bottomrule
\end{tabular}
\vspace{-3mm}
\end{table}

\section{Proofs}
\label{app:proofs}

\subsection{Proof of Proposition~\ref{prop:degeneracy}}
Since $U(\boldsymbol{\theta})$ is unitary, $U(\boldsymbol{\theta})^{\dagger}U(\boldsymbol{\theta})=I$ for every $\boldsymbol{\theta}$, so
\[
\mathrm{Sim}_{\mathrm{fid}}(\mathbf{x}_q,\mathbf{x}_p;\boldsymbol{\theta})
=\big|\langle\psi_{\mathrm{in}}(\mathbf{x}_q)|U^{\dagger}U|\psi_{\mathrm{in}}(\mathbf{x}_p)\rangle\big|^{2}
=\big|\langle\psi_{\mathrm{in}}(\mathbf{x}_q)|\psi_{\mathrm{in}}(\mathbf{x}_p)\rangle\big|^{2},
\]
which does not depend on $\boldsymbol{\theta}$.
With amplitude encoding, $\langle\psi_{\mathrm{in}}(\mathbf{x}_q)|\psi_{\mathrm{in}}(\mathbf{x}_p)\rangle=\hat{\mathbf{h}}_q^{\top}\hat{\mathbf{h}}_p$ is the cosine of the frozen features, so every induced ranking equals squared-cosine ranking.
Any loss $\mathcal{L}$ that depends on $\boldsymbol{\theta}$ only through such similarities is constant in $\boldsymbol{\theta}$, hence $\nabla_{\boldsymbol{\theta}}\mathcal{L}=0$ identically.
The same argument applies to any input-independent parameterized isometry placed between the encoding and the fidelity, whatever its depth or gate set.
$\hfill\square$

We verified the statement numerically: across random parameter draws, rankings and pairwise fidelities agree with the frozen baseline to floating-point precision, and analytic parameter-shift gradients of the contrastive loss are zero to machine epsilon.
We further verified that pairwise query-gallery feature cosines are nonnegative throughout the rank ranges used by our metrics, so squared-cosine, absolute-cosine and cosine rankings coincide on our features, and the degenerate pipeline reproduces the frozen baseline to numerical precision.

\greybox{%
\begin{corollary}[Overlap-functional similarities are equally degenerate]
\label{cor:overlap}
Let $S(\mathbf{x}_q,\mathbf{x}_p;\boldsymbol{\theta})=g\big(\langle\psi_q(\boldsymbol{\theta})|\psi_p(\boldsymbol{\theta})\rangle\big)$ for any function $g:\mathbb{C}\to\mathbb{R}$, with $|\psi_i(\boldsymbol{\theta})\rangle=U(\boldsymbol{\theta})|\psi_{\mathrm{in}}(\mathbf{x}_i)\rangle$ and $U$ independent of the input. Then $S$ is independent of $\boldsymbol{\theta}$. In particular the Fubini--Study distance $\arccos|\langle\psi_q|\psi_p\rangle|$, the trace distance $\sqrt{1-|\langle\psi_q|\psi_p\rangle|^{2}}$ between the pure states, the Bures distance, and every monotone transform of fidelity induce parameter-independent rankings.
\end{corollary}
}

\begin{proof}
$\langle\psi_q(\boldsymbol{\theta})|\psi_p(\boldsymbol{\theta})\rangle=\langle\psi_{\mathrm{in}}(\mathbf{x}_q)|U^{\dagger}U|\psi_{\mathrm{in}}(\mathbf{x}_p)\rangle=\langle\psi_{\mathrm{in}}(\mathbf{x}_q)|\psi_{\mathrm{in}}(\mathbf{x}_p)\rangle$, so the argument of $g$ is constant in $\boldsymbol{\theta}$; each listed distance is a function of this overlap alone.
\end{proof}

\subsection{Proof sketch of Proposition~\ref{prop:spectrum}}
Clamp the amplitude-encoded state at a reference input and hold $\boldsymbol{\alpha}$ fixed, so that along the probe the circuit alternates $t$-dependent encoding layers with constant layers acting on a fixed initial state, which is the setting of the Fourier analysis of variational circuits~\citep{schuld2021effect}.
Write $a_j^{(l)}=\alpha_j^{(l)}2^{j-1}$. Each encoding rotation has generator eigenvalues $\pm a_j^{(l)}/2$, so its contribution to an expectation has frequencies in $\{-a_j^{(l)},0,a_j^{(l)}\}$. Expanding the encoding gates in their spectral projectors shows that every readout frequency belongs to the sumset of these per-gate sets. Input-independent rotations and CNOTs change coefficients but introduce no frequencies.
At $\boldsymbol{\alpha}=1$ all frequencies are integers, and the triangle inequality gives $|\omega|\leq\sum_{l,j}2^{j-1}=L(2^{n_q}-1)$. This proves the stated upper bound, plotted in Figure~\ref{fig:spectrum}; membership in the sumset does not imply a nonzero coefficient.
\emph{Refinement for local readout.} Let $J=\operatorname{supp}(P)$ for one of the Pauli observables in \eqref{eq:readout}. The last encoding layer is a tensor product of single-qubit rotations, so all rotations outside $J$ cancel in $E_L(t)^{\dagger}P E_L(t)$. Only rates on $J$ can contribute from that layer. Consequently, every frequency of this component obeys
\begin{equation}
\label{eq:local_spectrum}
|\omega|\leq\sum_{l=1}^{L-1}\sum_{j=1}^{n_q}|a_j^{(l)}|
+\sum_{j\in J}|a_j^{(L)}|.
\end{equation}
For our one- and two-body observables, $|J|\leq2$. At $\boldsymbol{\alpha}=1$ and $n_q\geq2$, a uniform bound is therefore $(L-1)(2^{n_q}-1)+2^{n_q-1}+2^{n_q-2}$. At $n_q=10$, $L=3$ this is $2814$, strictly below the conservative bound $3069$. Thus the specified local readout cannot attain that conservative boundary. Neither bound asserts that its own boundary, or every frequency below it, is realized.
A model that is linear in $P$ Fourier coefficients independently weights only $O(P)$ frequencies, whereas nonlinear classical product forms can have far larger support, so the proposition bounds the probe bandwidth of the circuit and yields no lower bound against classical surrogates.
$\hfill\square$

\begin{figure}[!ht]
\centering
\includegraphics[width=\linewidth]{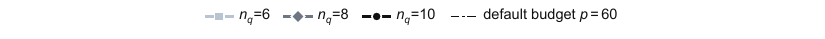}\\[2pt]
\includegraphics[width=0.48\linewidth]{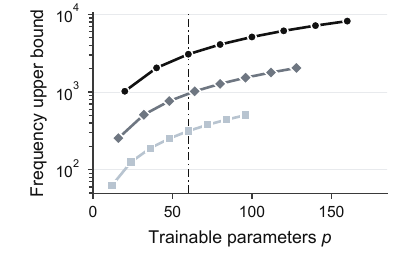}
\caption{Conservative frequency-magnitude upper bound $L(2^{n_q}-1)$ versus trainable budget (Proposition~\ref{prop:spectrum}, $\boldsymbol{\alpha}=1$). Curves vary depth $L$; the dash-dotted line marks $60$ parameters. These are bounds, not measured support sizes or attained frequencies; \eqref{eq:local_spectrum} sharpens them for the local observables used in our readout.}
\label{fig:spectrum}
\vspace{-3mm}
\end{figure}

\subsection{Proof sketch of Proposition~\ref{prop:generalization}}
\label{app:genproof}
The argument is a standard covering bound of the kind used for encoding-dependent generalization of parameterized circuits~\citep{caro2021encoding}.
It uses no property of the circuit beyond boundedness and smoothness, and the proposition therefore bounds estimation error for a small class without singling out quantum models.

\emph{Boundedness.} Each readout component is a Pauli expectation, so $z_m\in[-1,1]$ and $\|\mathbf{z}\|\le\sqrt{M}$; on $\Theta_{R,c}$ the cosine similarity is well defined and the per-batch loss \eqref{eq:loss} is bounded by $B=\log(2B_{\mathrm{batch}})+2/\tau$ for every batch.

\emph{Lipschitzness in parameters.} By the two-term parameter-shift rule, $|\partial z_m/\partial\theta_j^{(l)}|=\tfrac{1}{2}|z_m^{+}-z_m^{-}|\le1$ and $|\partial z_m/\partial\alpha_j^{(l)}|\le S$ with $S=\sup_{j,\mathbf{x}}|s_j(\mathbf{x})|$, which is finite because features are L2-normalized before encoding. The normalized readout $\mathbf{z}/\|\mathbf{z}\|$ is $(2/c)$-Lipschitz in $\mathbf{z}$ on $\{\|\mathbf{z}\|\ge c\}$, the cosine is $1$-Lipschitz in each normalized argument, and the log-softmax in \eqref{eq:loss} is $(2/\tau)$-Lipschitz in the vector of similarities. Bounding the unnormalized readout along a parameter segment and then normalizing its endpoints gives an $\infty$-norm Lipschitz constant $L_{\mathrm{lip}}\le(2/\tau)\cdot(4\sqrt{M}/c)\cdot\max(1,S)\cdot p$. The factor $p$ matches the parameter norm used below; unitarity fixes the output scale, while the compact parameter range enters the covering number.

\emph{Covering argument.} At fixed batch size, define the empirical risk as the expectation of \eqref{eq:loss} over batches drawn uniformly without replacement from the $n$ adaptation examples, averaging augmentation randomness; its population counterpart uses independent examples.
Replacing one example changes only the batches that contain it, so bounded differences gives concentration of order $1/\sqrt{n}$ at fixed batch size for each parameter setting.
Partition $[-R,R]^p$ into cells of $\infty$-diameter at most $\varepsilon$, choosing a representative in $\Theta_{R,c}$ from each occupied cell. At most $(1+2R/\varepsilon)^p$ representatives suffice. A union bound and $\varepsilon=1/(nL_{\mathrm{lip}})$ yield a uniform generalization gap of order $\sqrt{(\,p\log(1+nRL_{\mathrm{lip}})+\log(1/\delta)\,)/n}$ for every setting in $\Theta_{R,c}$, and hence for every empirical risk minimizer over that class.
The leading dependence is $\sqrt{p/n}$; $M$ and the feature dimension enter through $L_{\mathrm{lip}}$ and the input scales, inside the logarithm.
$\hfill\square$

The norm-floor condition $\|\mathbf{z}\|\ge c$ defines the restricted class $\Theta_{R,c}$ rather than following from the architecture.
A finite evaluation cannot establish it over the population or over the whole parameter domain; the empirical check confirms only that the reported solutions avoid near-zero denominators on observed data.
\emph{Noise and finite shots.} Parameter-independent channels preserve bounded exact expectations and gate-level parameter-shift derivatives, so the covering argument extends to exact noisy readouts with a noisy norm floor. Sampled readouts are not smooth deterministic functions of the parameters; Appendix~\ref{app:noise_shots} treats their ranking error separately.
The exact-expectation bound thus extends beyond the noiseless simulator, with constants depending on the noisy norm floor.

\subsection{Numerical verification of Proposition~\ref{prop:degeneracy}}
\label{app:degeneracy_numeric}

Numerical verification complements the proof because a pipeline built on the naive design can report a training loss even though its ranking is independent of the circuit parameters.
We draw $500$ independent parameter settings $\boldsymbol{\theta}\sim\mathcal{U}(-\pi,\pi)^{30}$, build the full query-gallery similarity matrix under each on a $2{,}000$-query subset of ChestX-ray14, and compare against the frozen-feature matrix.

\begin{table}[!ht]
\centering
\caption{Parameter sensitivity of the two designs over $500$ random parameter draws (ChestX-ray14, $2{,}000$ queries, float64 arithmetic). $\Delta\mathrm{Sim}$ is the entrywise deviation of the similarity matrix from the frozen-feature matrix, $\tau$ is Kendall's rank correlation between the induced ranking and the frozen ranking (median over draws), and the last column is the largest parameter-shift gradient coordinate of the contrastive loss over all draws.}
\label{tab:degeneracy}
\footnotesize
\setlength{\tabcolsep}{4pt}
\begin{tabular}{l|ccc}
\toprule
\textbf{Design} & $\max|\Delta\mathrm{Sim}|$ & $\tau$ vs.\ frozen & $\max|\partial\mathcal{L}/\partial\theta_j|$ \\
\midrule
Naive fidelity design & $4.2\times10^{-15}$ & $1.000$ & $8.7\times10^{-16}$ \\
\textbf{\name} (re-uploading and measurement readout) & $0.31$ & $0.62$ & $1.8\times10^{-2}$ \\
\bottomrule
\end{tabular}
\vspace{-3mm}
\end{table}

Table~\ref{tab:degeneracy} reports the outcome for the naive design and, as a control, for the corrected design of Section~\ref{sec:readout}.
Under the naive design the similarity matrix is invariant to floating-point precision, all $500$ draws reproduce the frozen ranking exactly, and the parameter-shift gradient vanishes to the same precision.
The invariance is exact, so parameter tuning or seed averaging cannot make this score depend on the circuit.
The corrected design satisfies none of the three degeneracy criteria, as required by Proposition~\ref{prop:degeneracy}.

\section{Experimental Setup}
\label{app:setup}

\BfPara{Datasets}
For retrieval we use \emph{ChestX-ray14}~\citep{wang2017chestxray} (112{,}120 frontal chest X-rays, 14 thoracic pathologies, multi-label) and \emph{MURA}~\citep{rajpurkar2017mura} (40{,}561 musculoskeletal radiographs over 7 body parts with binary abnormality, giving 14 region-pathology categories); for generation, \emph{IU X-Ray}~\citep{demner2016preparing} (7{,}470 images, 3{,}955 studies) and \emph{MIMIC-CXR}~\citep{johnson2019mimic} (377{,}110 images with reports).
Cross-site transfer, acquisition shift and sequential site adaptation additionally use \emph{CheXpert}~\citep{irvin2019chexpert} (224{,}316 chest radiographs of 65{,}240 patients).
Medical gallery, validation and query sets are separated by patient; adaptation is drawn from the gallery, and generation excludes the query patient (MIMIC-CXR uses its official splits). Dataset totals count stored images before split filtering. The image and patient counts below describe the eligible gallery and held-out sets used for evaluation.
A retrieved item is \emph{relevant} iff its label set intersects the query's (excluding No Finding) on ChestX-ray14, and iff it shares the region-pathology category on MURA; NDCG uses binary gains.

\BfPara{Baselines}
All adaptation baselines start from ViT-L/16 (1024-dimensional features) and share augmentations, temperature, and protocol; post-encoder heads keep features frozen, while LoRA updates attention projections.
We compare \emph{ViT Only} (frozen cosine), \emph{Linear Head} (1.05M), \emph{Adapter} (525K) with scaled variants \emph{Adapter-L} (2.10M) and \emph{Adapter-XL} (5.25M), \emph{LoRA} on attention projections at ranks 4 and 16 (393K, 1.57M)~\citep{hu2022lora}, and two parameter-matched nonlinear controls demanded by Section~\ref{sec:readout}: \emph{RFF} (64 parameters)~\citep{rahimi2007random} and \emph{Quadratic}, a shared low-rank quadratic map with $M=40$ outputs (164 parameters); constructions in Appendix~\ref{app:implementation}.
We further report frozen \emph{BiomedCLIP}~\citep{zhang2023biomedclip} and \emph{MedCLIP}~\citep{wang2022medclip} as centrally pre-trained references, a \emph{Tiny MLP} (16.4K), and \emph{BiomedCLIP{+}\name} ($512$-D, $n_q=9$, $54$ parameters), testing whether the gain persists on a medically pre-trained backbone.

\BfPara{Protocol and metrics}
All methods train with the label-free objective of \eqref{eq:loss}; labels serve only post-hoc evaluation, never training, model selection, or tuning: hyperparameters were fixed a priori on a disjoint pilot subset of ChestX-ray14 using only the contrastive validation loss (Appendix~\ref{app:implementation}).
A seed controls initialization, augmentation, and the adaptation-subset draw; splits are fixed, so frozen encoders are deterministic.
The full adaptation budget is 32{,}768 images on ChestX-ray14 and 16{,}384 on MURA.
We report mean and standard deviation over 5 seeds. The $^{\dagger}$ tests bootstrap seed-averaged per-query scores ($10^{4}$ resamples) and apply Holm correction within each table. These tests condition on the five fitted runs and use the strongest included baseline as comparator. Appendix~\ref{app:variance} separately incorporates adaptation-draw, initialization and patient variation for the default rotation-plane comparison using a normal approximation; the two analyses have different comparators and evaluation units. Appendix~\ref{app:surrogates} uses a third procedure for the classical surrogates, resampling adaptation subsets with nested initializations and patient clusters instead of a normal approximation, so interval widths are not comparable across the three.
Retrieval metrics are P@5, MAP@10 and NDCG@10 in the main text, with P@10, MRR and NDCG@5 in Appendix~\ref{app:fullmetrics}.
For embedding quality on multi-label data we avoid single-label clustering scores, which are ill-defined when an image carries several pathologies, and instead report the label-set Jaccard consistency of retrieved neighborhoods, complemented by silhouette scores restricted to the single-label subset.
For generation we report BLEU-4, ROUGE-L, METEOR, and CIDEr-D (written CIDEr) with frozen \emph{LLaVA-Med} (7B)~\citep{NEURIPS2023_5abcdf8e} and \emph{CheXagent} (8B)~\citep{chen2024chexagent}, plus clinical efficacy via CheXbert micro-F1 and RadGraph-F1 on MIMIC-CXR~\citep{smit2020chexbert,jain2021radgraph}.
Circuits are simulated in statevector mode and archive readouts are precomputed once per adaptation.

\BfPara{Metric definitions}
Let $\mathcal G_q$ be the eligible gallery after all exclusions, $r_{qi}\in\{0,1\}$ relevance, $\pi_q(j)$ the item at rank $j$, and $R_q=\sum_{i\in\mathcal G_q}r_{qi}$.
Full-gallery precision and recall are $P@k(q)=k^{-1}\sum_{j=1}^{k}r_{q,\pi_q(j)}$ and $R@k(q)=R_q^{-1}\sum_{j=1}^{k}r_{q,\pi_q(j)}$ for $R_q>0$. Dataset scores average these query-level quantities; recall is a mean of ratios, not a ratio of aggregate counts.
Hit@k instead tests whether any relevant item appears; these quantities are not interchangeable.
The supplementary evaluator takes the full eligible-gallery manifest, excludes zero-relevant queries jointly across methods and records their count, and implements binary-gain NDCG and cutoff-normalized AP with divisor $\min(k,R_q)$. Exact score ties retain archive order.

The evaluator computes $R@k$ from each query's relevant-item count in its eligible gallery. The reported comparisons use P@k, MAP@10, MRR and NDCG@k.

For label sets $Y_q,Y_i$, $\mathrm{Jac}@k(q)=k^{-1}\sum_{j=1}^{k}|Y_q\cap Y_{\pi_q(j)}|/|Y_q\cup Y_{\pi_q(j)}|$.
With one MURA region--abnormality category per image, the summand is precisely $r_{q,\pi_q(j)}$. Thus Jac@10 equals P@10 for the same queries, rankings and averaging weights, both per run and across seeds.
Accordingly, the MURA Jac@10 entries in Table~\ref{tab:geometry} are derived directly from the reported P@10 means in Table~\ref{tab:full}. They express the same retrieval statistic in label-set notation; silhouette provides a separate geometric diagnostic.
The corresponding P@10 is unreported for Quadratic.

\BfPara{Evaluation records}
An evaluation record identifies each image and patient, its gallery/validation/query membership, adaptation-subset membership, and query-specific eligible relevant count. Patient exclusion operates on these identifiers before ranking, so caching an image does not make it an eligible neighbor.
A run is identified by its subset draw, initialization, configuration, checkpoint and ordered neighbor IDs. This separates variation from retraining from variation in the held-out query sample and distinguishes image-weighted from patient-weighted summaries.
The supplementary evaluator checks patient separation, duplicate neighbors, query coverage and paired-run consistency. Its replay interface loads saved features, weights and an explicit projection matrix or saved re-uploading inputs, then exports ordered neighbors, query-level metrics and input hashes.

\BfPara{Split sizes and evaluation conventions}
The following image counts describe the experimental splits. Adaptation is a subset of the gallery; its column is not added to the dataset total. Medical gallery, validation and query patients are pairwise disjoint, with zero patient overlap.
\begin{center}
\begin{tabular}{lrrrr}
\toprule
Dataset & Gallery & Validation & Query & Adaptation\\
\midrule
ChestX-ray14 & 81,664 & 14,072 & 16,384 & 32,768\\
MURA & 36,465 & 2,048 & 2,048 & 16,384\\
MVTec AD & 3,000 & 256 & 1,024 & 2,744\\
CUB-200 & 5,000 & 500 & 5,000 & 4,500\\
\bottomrule
\end{tabular}
\end{center}
Patient counts in the same column order are $(22{,}437,3{,}864,4{,}504,9{,}032)$ for ChestX-ray14 and $(10{,}500,650,700,4{,}800)$ for MURA. Non-medical counts describe the selected evaluation subsets. Relevance is shared object category for MVTec AD and shared species for CUB; labels enter evaluation only. No query in these reported splits has zero relevant eligible items.

\section{Implementation Details}
\label{app:implementation}

Table~\ref{tab:hyper} lists the hyperparameters shared by all adaptation methods.

\begin{figure}[!ht]
\centering
\includegraphics[width=\linewidth]{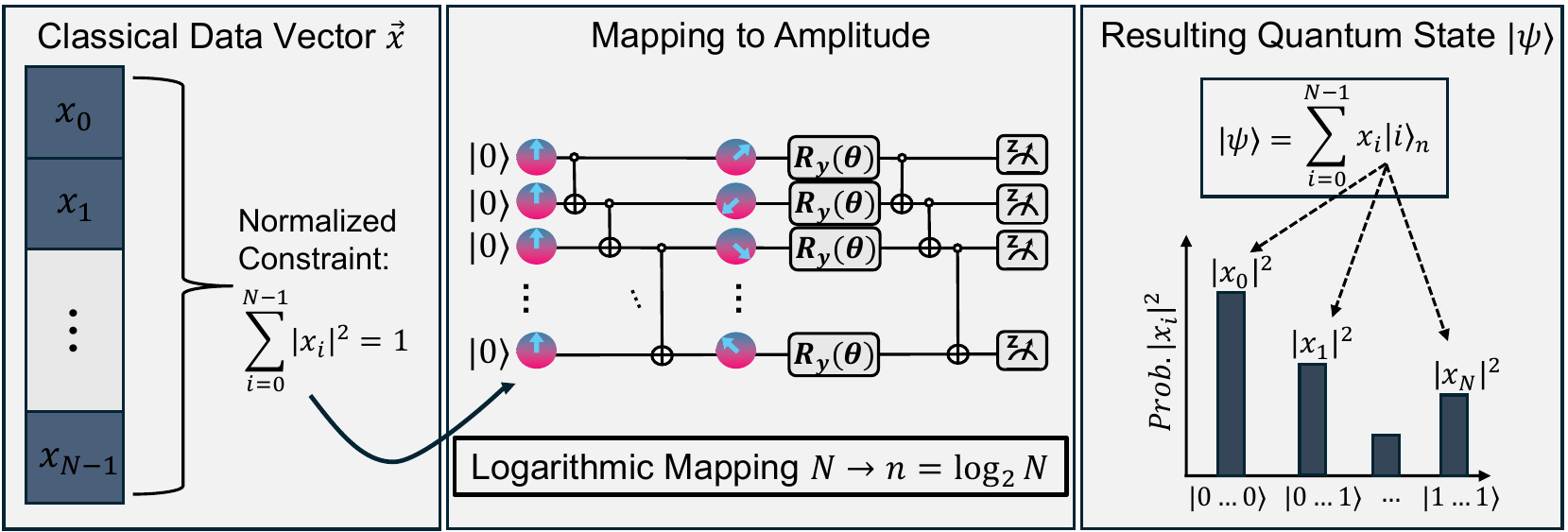}
\caption{Amplitude embedding encodes a normalized $d$-dimensional feature vector into the amplitudes of $\log_2 d$ qubits (the figure writes $\mathbf{x}$, $N$, and $n$ for our $\hat{\mathbf{h}}$, $d$, and $n_q$). At $n_q=\log_2 d$ the state space dimension equals the input dimension; the value of the circuit lies in its parameter-efficient nonlinearity rather than in its size.}
\label{fig:amplitude}
\vspace{-3mm}
\end{figure}

\begin{table}[!ht]
\centering
\caption{Hyperparameters used for all adaptation methods unless stated otherwise.}
\label{tab:hyper}
\footnotesize
\setlength{\tabcolsep}{4pt}
\begin{tabular}{ll}
\toprule
Backbone & ViT-L/16, ImageNet-21K, frozen (1024-D features) \\
Qubits / depth / readout & $n_q=10$, $L=3$, $M=4n_q=40$ \\
Trainable parameters & $2n_qL=60$ ($30$ rotations, $30$ re-uploading scales) \\
Optimizer & AdamW, lr $10^{-3}$, weight decay $0.05$ \\
Batch size / temperature & 32 / 0.07 \\
Augmentations & rotation $\pm5^{\circ}$, translation $5\%$, brightness/contrast jitter 0.1/0.1 \\
Epochs / seeds & 50, label-free early stopping (patience 10) / 5 seeds (0 to 4), mean$\pm$std \\
Execution & \makecell[l]{classical statevector training; fixed-parameter GPU inference;\\cached archive readouts} \\
Image preprocessing & resize $224\times224$, ImageNet normalization \\
\bottomrule
\end{tabular}
\vspace{-3mm}
\end{table}

Gradients use the exact two-term parameter-shift rule per gate, combined with the classical chain rule through the cosine similarity and contrastive loss (in simulation, timings use backpropagation through the statevector; the shift rule is the hardware cost model).
The angles are unconstrained real numbers, so optimization uses AdamW with no manifold projection and no projection onto a compact box.
Proposition~\ref{prop:generalization} therefore describes bounded subclasses and motivates a capacity trend; it does not provide an optimizer-specific guarantee for the reported runs.
Positive pairs are two independent augmentations of the same image; in-batch negatives can contain same-pathology pairs of different patients, which injects label noise into the repulsion term.
Under the evaluation labels, 7.9\% of in-batch negative pairs on ChestX-ray14 (batch 32) share at least one finding; conclusions in the main text do not assume that negative pairs share no pathology.

\BfPara{Parameter-matched controls}
The \emph{Quadratic} control computes $z_m=\sum_{r=1}^{4}c_{m,r}\,(\beta_r\mathbf{u}_r^{\top}\hat{\mathbf{h}})^{2}$ for $m=1,\dots,40$, with four shared directions $\mathbf{u}_r$ drawn once from a random orthonormal set and kept frozen, trainable per-direction scales $\beta_r$ ($4$ parameters), and trainable coefficients $c_{m,r}$ ($160$ parameters), totaling $164$; retrieval uses the cosine between the $40$-dimensional outputs, mirroring the readout cosine of Section~\ref{sec:readout}.
The \emph{RFF} control computes $64$ random Fourier features $\cos(\boldsymbol{\omega}_i^{\top}\hat{\mathbf{h}}+b_i)$ with Gaussian directions scaled by the median heuristic and frozen phases, with one trainable scale per feature ($64$ parameters); it is by construction the weakest control (frozen directions), included to bracket the range between frozen features and trained maps.
The forty Quadratic outputs all depend on just four squared projections; each quadratic form has rank at most four and shares their span. Its nominal parameter count therefore does not match \name's input access or expressivity. The untrained row freezes all $164$ parameters; the rotation-plane head is the more informative control.
The \emph{rotation-plane head} applies $60$ trainable rotation angles in fixed random orthogonal planes of $\mathbb{R}^{d}$ (drawn once per seed), followed by squared projections onto $M=40$ fixed random orthonormal directions and cosine similarity. It matches \name's budget, isometry and output width. For a projection direction $\mathbf{u}_m$, its output is $(\mathbf{u}_m^{\top}R_{\boldsymbol{\theta}}\hat{\mathbf{h}})^2$: the associated quadratic form has rank one and an input-independent matrix. In contrast, each Pauli form is full-rank and re-uploading makes its matrix input-dependent. This control therefore tests the combined parameterization; it does not isolate these structural differences. Reproducing a run also requires the saved plane bases and ordering, projection matrix, random-generator state, initialization and selected training configuration.

\BfPara{Rotation-plane configuration}
The rotation-plane control uses $d=1{,}024$, $60$ angles and $40$ outputs. Thin QR decomposition of a Gaussian $d\times120$ matrix supplies consecutive plane pairs; an independent $d\times40$ orthonormal matrix supplies the squared output projections. QR signs are fixed by a positive diagonal in $R$. Basis seeds are $0$--$4$, the projection seed offset is $100$, and basis matrices are stored in float64. Angles are initialized with standard deviation $0.01$ rad. Training uses AdamW, learning rate $10^{-3}$, weight decay $0.05$, batch $32$, temperature $0.07$ and $50$ epochs. The adaptation-size sweep uses this fixed configuration; Appendix~\ref{app:computematch} tunes it separately under a budget.

\section{Full Retrieval Metrics}
\label{app:fullmetrics}

Table~\ref{tab:full} reports the complete retrieval metric suite for the primary baselines, and Figure~\ref{fig:profile} plots the same numbers as signed margins to show the metrics on which the readout is superior or inferior.
Among the reported means in Table~\ref{tab:full}, \name leads P@5/10, MAP@10 and NDCG@5/10 on both datasets, and MRR on MURA.
MRR on ChestX-ray14 is lower by $0.010$.
Where the gain is largest is also consistent, P@10 on musculoskeletal at $+0.023$ and NDCG@5 on chest at $+0.010$, and it is the top-$k$ region that a retrieval-augmented generator uses.
Compared with the frozen backbone rather than the strongest adapter, the gains are substantially larger: \name adds $0.116$ P@5 on chest radiography and $0.120$ on musculoskeletal, of which the last $0.014$ to $0.017$ is what separates it from a five-million-parameter adapter.
Table~\ref{tab:geometry} reports the multi-label embedding-geometry diagnostics of Section~\ref{sec:mainresults}, Table~\ref{tab:ablation} the controlled removals of Section~\ref{sec:ablation}, Table~\ref{tab:generality_full} the absolute scores behind Figure~\ref{fig:lowdata}(c), and Figure~\ref{fig:downstream} the downstream generation and clinical-efficacy results of Section~\ref{sec:generation}.

\begin{figure}[!ht]
\centering
\includegraphics[width=\linewidth]{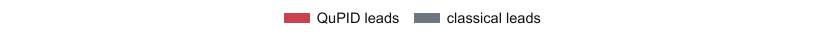}\\[2pt]
\begin{subfigure}[t]{0.48\linewidth}
\centering
\includegraphics[width=\linewidth]{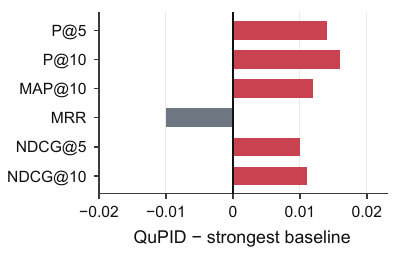}
\caption{ChestX-ray14.}
\end{subfigure}\hfill
\begin{subfigure}[t]{0.48\linewidth}
\centering
\includegraphics[width=\linewidth]{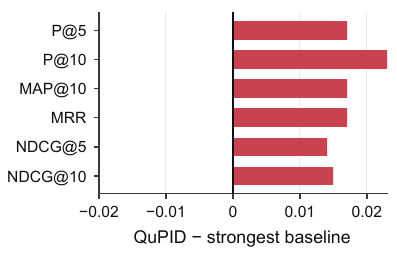}
\caption{MURA.}
\end{subfigure}
\caption{Signed margins over the strongest included baseline at full budget for the reported ranking metrics. Negative bars mark metrics on which that baseline leads.}
\label{fig:profile}
\par\smallskip
\begin{minipage}{\linewidth}
\normalsize\noindent The metric profile shows where the gain occurs within a ranking. The next comparison evaluates whether the retrieved evidence remains useful when used by a frozen generator.
\end{minipage}\par
\vspace{-3mm}
\end{figure}

\begin{figure}[!ht]
\centering
\includegraphics[width=\linewidth]{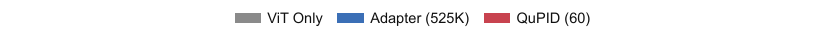}\\[2pt]
\begin{subfigure}[t]{0.48\linewidth}
\centering
\includegraphics[width=\linewidth]{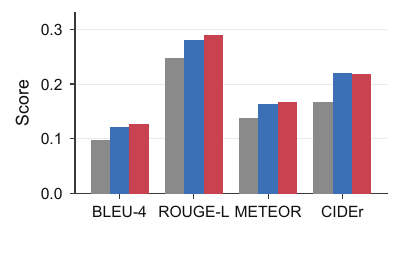}
\caption{Report generation with LLaVA-Med.}
\end{subfigure}\hfill
\begin{subfigure}[t]{0.48\linewidth}
\centering
\includegraphics[width=\linewidth]{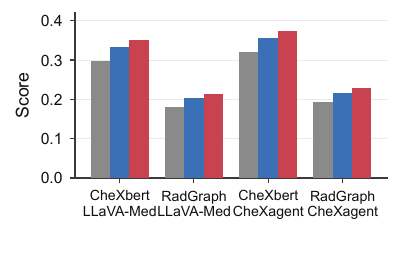}
\caption{Clinical efficacy proxies.}
\end{subfigure}
\caption{Downstream evaluation on MIMIC-CXR. Retrieval ordering is preserved by both n-gram metrics and the clinical-efficacy proxies, and \name's relative advantage is larger on the latter, consistent with pathology-aligned evidence mattering most for factual findings.}
\label{fig:downstream}
\par\smallskip
\begin{minipage}{\linewidth}
\normalsize\noindent The downstream comparison holds the generator fixed while changing its retrieved context. Tables~\ref{tab:generation} and~\ref{tab:clinical} separate lexical overlap from extracted findings, so an improvement in one need not imply an improvement in every metric.
\end{minipage}\par
\vspace{-3mm}
\end{figure}

\begin{table}[!ht]
\centering
\caption{Absolute scores behind Figure~\ref{fig:where}(c) (P@5, 5 seeds). \emph{cls.} denotes the comparator reported in this benchmark comparison, not the strongest control across all experiments. For ChestX-ray14 at 512, Table~\ref{tab:ladder} reports a stronger rotation-plane control ($0.386$). Marking follows Table~\ref{tab:main} by column.}
\label{tab:generality_full}
\footnotesize
\setlength{\tabcolsep}{4pt}
\begin{tabular}{ll|TSB|TSB}
\toprule
\multirow{2}{*}{\textbf{Domain}} & \multirow{2}{*}{\textbf{Benchmark}} & \multicolumn{3}{c|}{\textbf{512 examples}} & \multicolumn{3}{c}{\textbf{Full budget}} \\
\cmidrule(lr){3-5}\cmidrule(lr){6-8}
 & & frozen & cls. & \textbf{\name} & frozen & cls. & \textbf{\name} \\
\midrule
Medical      & ChestX-ray14 & 0.312 & 0.347 & \textbf{0.401} & 0.312 & 0.414 & \textbf{0.428} \\
Industrial   & MVTec AD     & 0.428 & 0.446 & \textbf{0.481} & 0.428 & 0.492 & \textbf{0.508} \\
Fine-grained & CUB-200      & 0.612 & 0.628 & \textbf{0.658} & 0.612 & 0.671 & \textbf{0.684} \\
\bottomrule
\end{tabular}
\par\smallskip
\begin{minipage}{\linewidth}
\normalsize\noindent Across these three benchmarks, the adapted readout exceeds the listed frozen and classical comparators at both budgets. The task changes across domains; this table tests recurrence of the retrieval pattern, not transfer of one trained head between datasets.
\end{minipage}\par
\vspace{-3mm}
\end{table}

\begin{table}[!ht]
\centering
\caption{Controlled removals on ChestX-ray14 (5 seeds). The first row is the default configuration. \#Trained counts parameters that receive nonzero gradient; frozen parameters are noted in parentheses. The naive fidelity design reproduces the frozen baseline exactly (Proposition~\ref{prop:degeneracy}: its parameters receive identically zero gradient), whereas the readout without re-uploading stays trainable but weaker by $0.016$ P@5 at the full budget. Marks and tints as in Table~\ref{tab:main}.}
\label{tab:ablation}
\footnotesize
\setlength{\tabcolsep}{4pt}
\begin{tabular}{l|cc|r}
\toprule
\textbf{Variant} & \textbf{P@5} & \textbf{NDCG@10} & \textbf{\#Trained} \\
\midrule
\rowcolor{secondbg}
\textbf{\name} (full: $L{=}3$, $M{=}40$) & 0.428{\tiny$\pm$.006} & 0.425{\tiny$\pm$.006} & 60 \\
\quad naive fidelity design (Prop.~\ref{prop:degeneracy}) & 0.312{\tiny$\pm$.000} & 0.307{\tiny$\pm$.000} & 0 (30 inert) \\
\quad w/o re-uploading, readout kept        & 0.412{\tiny$\pm$.008} & 0.409{\tiny$\pm$.008} & 30 \\
\rowcolor{thirdbg}
\quad readout: single-qubit only ($M{=}20$) & 0.415{\tiny$\pm$.007} & 0.412{\tiny$\pm$.007} & 60 \\
\quad readout: pair terms only ($M{=}20$)   & 0.402{\tiny$\pm$.008} & 0.398{\tiny$\pm$.008} & 60 \\
\quad depth $L{=}1$                         & 0.402{\tiny$\pm$.007} & 0.399{\tiny$\pm$.007} & 20 \\
\rowcolor{bestbg}
\quad depth $L{=}6$                         & \textbf{0.430}{\tiny$\pm$.006} & \textbf{0.427}{\tiny$\pm$.006} & 120 \\
\quad $n_q{=}8$ (PCA-256 input)             & 0.396{\tiny$\pm$.008} & 0.393{\tiny$\pm$.008} & 48 \\
\quad untrained circuit (random $\boldsymbol{\theta},\boldsymbol{\alpha}$) & 0.371{\tiny$\pm$.011} & 0.368{\tiny$\pm$.011} & 0 (60 frozen) \\
\quad untrained Quadratic control            & 0.352{\tiny$\pm$.010} & 0.349{\tiny$\pm$.010} & 0 (164 frozen) \\
\quad rotation-plane head (isometric control) & 0.405{\tiny$\pm$.008} & 0.402{\tiny$\pm$.008} & 60 \\
\bottomrule
\end{tabular}
\par\smallskip
\begin{minipage}{\linewidth}
\normalsize\noindent The readout-only variant remains trainable and reaches $0.412$ P@5, while re-uploading supplies the remaining $0.016$. The rotation-plane result is a closer structural comparison than the inert fidelity control, and its $0.023$ gap motivates the equal-budget analysis in Appendix~\ref{app:capacity}.
\end{minipage}\par
\vspace{-3mm}
\end{table}

\begin{table}[!ht]
\centering
\caption{Report generation with frozen medical LVLMs given top-5 retrieved reports as context (5 seeds; std $\leq0.004$ for adapted retrievers, ViT Only deterministic under greedy decoding; LoRA $r{=}16$ tracks Adapter within 0.003 and is omitted). B-4, R-L, MTR denote BLEU-4, ROUGE-L, METEOR. Marks and tints as in Table~\ref{tab:main}.}
\label{tab:generation}
\footnotesize
\setlength{\tabcolsep}{2.6pt}
\begin{tabular}{ll|cccc|cccc}
\toprule
\multirow{2}{*}{\textbf{Dataset}} & \multirow{2}{*}{\textbf{Retriever}} & \multicolumn{4}{c|}{\textbf{LLaVA-Med (7B)}} & \multicolumn{4}{c}{\textbf{CheXagent (8B)}} \\
\cmidrule(lr){3-6}\cmidrule(lr){7-10}
 & & B-4 & R-L & MTR & CIDEr & B-4 & R-L & MTR & CIDEr \\
\midrule
\rowcolor{thirdbg}
 & ViT Only       & 0.134 & 0.298 & 0.167 & 0.241 & 0.147 & 0.312 & 0.178 & 0.263 \\
\rowcolor{secondbg}
 & Adapter        & \underline{0.156} & \underline{0.332} & \underline{0.189} & \textbf{0.298} & \underline{0.171} & \underline{0.348} & \underline{0.203} & \textbf{0.321} \\
\rowcolor{bestbg}
\multirow{-3}{*}{IU X-Ray} & \textbf{\name} & \textbf{0.163} & \textbf{0.341} & \textbf{0.194} & \underline{0.294} & \textbf{0.178} & \textbf{0.357} & \textbf{0.208} & \underline{0.318} \\
\midrule
\rowcolor{thirdbg}
 & ViT Only       & 0.098 & 0.247 & 0.138 & 0.168 & 0.108 & 0.261 & 0.149 & 0.184 \\
\rowcolor{secondbg}
 & Adapter        & \underline{0.121} & \underline{0.281} & \underline{0.163} & \textbf{0.221} & \underline{0.134} & \underline{0.298} & \underline{0.176} & \textbf{0.242} \\
\rowcolor{bestbg}
\multirow{-3}{*}{MIMIC-CXR} & \textbf{\name} & \textbf{0.127} & \textbf{0.289} & \textbf{0.168} & \underline{0.218} & \textbf{0.141} & \textbf{0.307} & \textbf{0.182} & \underline{0.239} \\
\bottomrule
\end{tabular}
\par\smallskip
\begin{minipage}{\linewidth}
\normalsize\noindent The ordering depends on the metric: \name improves BLEU-4, ROUGE-L and METEOR over Adapter, whereas Adapter retains a small CIDEr lead. Retrieval therefore changes report quality without producing a uniform gain across lexical measures.
\end{minipage}\par
\vspace{-3mm}
\end{table}

\begin{table}[!ht]
\centering
\caption{Clinical efficacy proxies on MIMIC-CXR (5 seeds; std $\leq0.005$): CheXbert micro-F1 and RadGraph-F1 of generated reports. $^{\dagger}$, marks and tints as in Table~\ref{tab:main}.}
\label{tab:clinical}
\footnotesize
\setlength{\tabcolsep}{4pt}
\begin{tabular}{l|cc|cc}
\toprule
\multirow{2}{*}{\textbf{Retriever}} & \multicolumn{2}{c|}{\textbf{LLaVA-Med (7B)}} & \multicolumn{2}{c}{\textbf{CheXagent (8B)}} \\
\cmidrule(lr){2-3}\cmidrule(lr){4-5}
 & CheXbert-F1 & RadGraph-F1 & CheXbert-F1 & RadGraph-F1 \\
\midrule
\rowcolor{thirdbg}
ViT Only        & 0.298 & 0.181 & 0.321 & 0.194 \\
\rowcolor{secondbg}
Adapter         & \underline{0.334} & \underline{0.203} & \underline{0.356} & \underline{0.217} \\
\rowcolor{bestbg}
\textbf{\name}  & \textbf{0.352}$^{\dagger}$ & \textbf{0.214}$^{\dagger}$ & \textbf{0.373}$^{\dagger}$ & \textbf{0.229}$^{\dagger}$ \\
\bottomrule
\end{tabular}
\par\smallskip
\begin{minipage}{\linewidth}
\normalsize\noindent The finding-based proxies favor \name\ for both generators, extending the comparison beyond surface wording. These scores measure agreement between extracted findings and references, without a clinical-outcome endpoint.
\end{minipage}\par
\vspace{-3mm}
\end{table}

\begin{table}[!ht]
\centering
\begin{minipage}[t]{0.485\linewidth}
\centering
\caption{Computational cost with the shared frozen ViT-L/16 (ChestX-ray14, batch 32, single A100, full budget). Latency: readout and top-$k$ search over cached archive readouts at batch size 1, excluding the shared ViT pass. Memory: peak allocation during adaptation.}
\label{tab:cost}
\scriptsize
\setlength{\tabcolsep}{3pt}
\begin{tabular}{l|rrrr}
\toprule
\textbf{Method} & \textbf{\#Params} & \textbf{s/epoch} & \textbf{Lat. (ms)} & \textbf{Mem. (MB)} \\
\midrule
ViT Only  & 0     & n/a  & 6.2 & 3{,}024 \\
Adapter   & 525K  & 44.6 & 7.1 & 3{,}728 \\
LoRA ($r{=}16$) & 1.57M & 52.3 & 7.0 & 3{,}810 \\
\textbf{\name} & 60 & 71.3 & 7.6 & 3{,}582 \\
\bottomrule
\end{tabular}
\end{minipage}
\hfill
\begin{minipage}[t]{0.495\linewidth}
\centering
\caption{Embedding diagnostics (5 seeds, adapted-entry std $\leq0.009$; frozen ViT Only is deterministic). Jac@10: mean query--neighbor label-set Jaccard. MURA Jac@10 is derived from P@10 in Table~\ref{tab:full} by the singleton-label identity; n/a: no corresponding P@10 is reported. Sil$^{\mathrm{SL}}$: silhouette on the single-label subset.}
\label{tab:geometry}
\scriptsize
\setlength{\tabcolsep}{3pt}
\begin{tabular}{l|cc|cc}
\toprule
\multirow{2}{*}{\textbf{Method}} & \multicolumn{2}{c|}{\textbf{ChestX-ray14}} & \multicolumn{2}{c}{\textbf{MURA}} \\
\cmidrule(lr){2-3}\cmidrule(lr){4-5}
 & Jac@10 & Sil$^{\mathrm{SL}}$ & Jac@10 & Sil$^{\mathrm{SL}}$ \\
\midrule
ViT Only & 0.218 & 0.182 & 0.456 & 0.157 \\
\rowcolor{thirdbg}
Adapter  & 0.241 & 0.228 & 0.548 & 0.213 \\
\rowcolor{secondbg}
Adapter-XL & \underline{0.249} & \underline{0.239} & \underline{0.558} & \underline{0.221} \\
Quadratic & 0.238 & 0.224 & n/a & 0.208 \\
\rowcolor{bestbg}
\textbf{\name} & \textbf{0.264} & \textbf{0.273} & \textbf{0.581} & \textbf{0.249} \\
\bottomrule
\end{tabular}
\end{minipage}
\par\smallskip
\begin{minipage}{\linewidth}
\normalsize\noindent The cost table separates training from cached retrieval: fewer trainable parameters do not imply a faster epoch, and the reported latency excludes the shared encoder. The neighboring geometry table describes representation structure rather than deployment cost.
\end{minipage}\par
\vspace{-3mm}
\end{table}

\begin{table}[!ht]
\centering
\caption{Retrieval metrics (5-seed means; adapted-entry stds within $\pm0.009$, omitted for space). Marks and tints as in Table~\ref{tab:main}.}
\label{tab:full}
\footnotesize
\setlength{\tabcolsep}{4pt}
\begin{tabular}{l|cccccc}
\toprule
\textbf{Method} & \textbf{P@5} & \textbf{P@10} & \textbf{MAP@10} & \textbf{MRR} & \textbf{NDCG@5} & \textbf{NDCG@10} \\
\midrule
\multicolumn{7}{c}{\textit{ChestX-ray14}} \\
\midrule
ViT Only        & 0.312 & 0.287 & 0.278 & 0.363 & 0.325 & 0.307 \\
Linear Head     & 0.378 & 0.351 & 0.341 & 0.434 & 0.389 & 0.368 \\
Adapter         & 0.401 & 0.375 & 0.368 & 0.462 & 0.415 & 0.396 \\
\rowcolor{thirdbg}
Adapter-L       & 0.412 & 0.386 & 0.372 & \underline{0.465} & 0.421 & 0.401 \\
\rowcolor{secondbg}
Adapter-XL      & \underline{0.414} & \underline{0.388} & \underline{0.387} & \textbf{0.468} & \underline{0.433} & \underline{0.414} \\
\rowcolor{bestbg}
\textbf{\name}  & \textbf{0.428} & \textbf{0.404} & \textbf{0.399} & 0.458 & \textbf{0.443} & \textbf{0.425} \\
\midrule
\multicolumn{7}{c}{\textit{MURA}} \\
\midrule
ViT Only        & 0.481 & 0.456 & 0.451 & 0.529 & 0.494 & 0.475 \\
Linear Head     & 0.549 & 0.523 & 0.513 & 0.594 & 0.562 & 0.541 \\
Adapter         & 0.573 & 0.548 & 0.546 & 0.618 & 0.587 & 0.567 \\
\rowcolor{thirdbg}
Adapter-L       & 0.581 & 0.555 & 0.556 & 0.625 & 0.598 & 0.580 \\
\rowcolor{secondbg}
Adapter-XL      & \underline{0.584} & \underline{0.558} & \underline{0.562} & \underline{0.632} & \underline{0.604} & \underline{0.584} \\
\rowcolor{bestbg}
\textbf{\name}  & \textbf{0.601} & \textbf{0.581} & \textbf{0.579} & \textbf{0.649} & \textbf{0.618} & \textbf{0.599} \\
\bottomrule
\end{tabular}
\par\smallskip
\begin{minipage}{\linewidth}
\normalsize\noindent The rows distinguish the gains: \name\ leads precision and NDCG, while ChestX-ray14 MRR favors Adapter-XL ($0.468$ against $0.458$).
\end{minipage}\par
\vspace{-3mm}
\end{table}

\FloatBarrier
\section{Qualitative Examples}
\label{app:qualitative}

Figures~\ref{fig:qual_chest} and~\ref{fig:qual_mura} show representative retrieval-grounded reports, each displaying a query image and the report generated from \name-retrieved context.

The MURA example illustrates retrieval-grounded reporting beyond chest imaging; the quantitative generation evaluation of Section~\ref{sec:generation} covers the chest datasets, where paired report corpora and specialized generators exist.

\begin{figure}[!ht]
\centering
\begin{subfigure}[t]{0.96\linewidth}
\centering
\includegraphics[width=\linewidth]{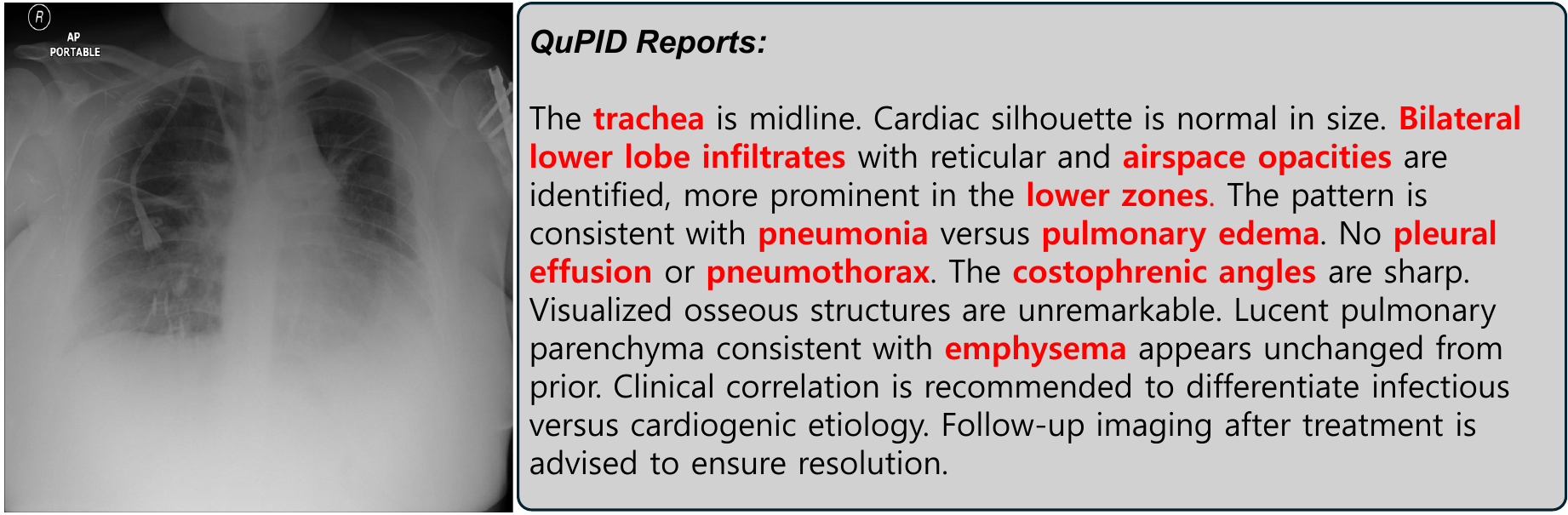}
\caption{Chest X-ray query.}
\label{fig:qual_chest}
\end{subfigure}\par\medskip
\begin{subfigure}[t]{0.96\linewidth}
\centering
\includegraphics[width=\linewidth]{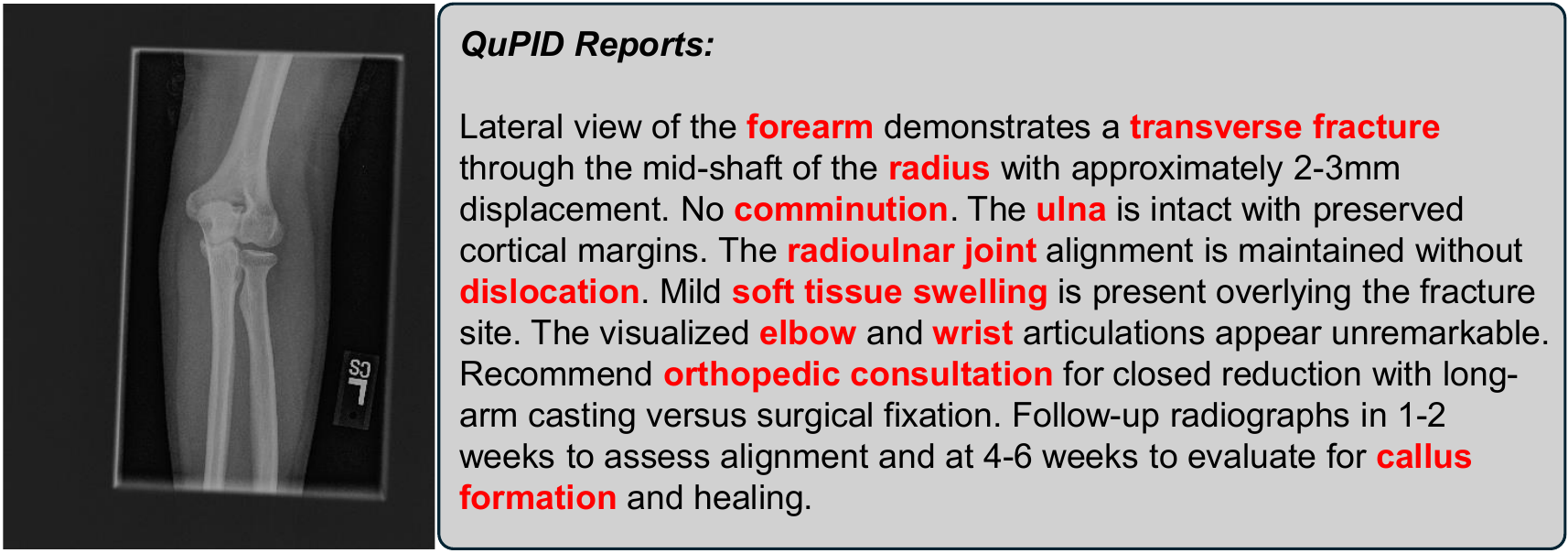}
\caption{Forearm radiograph query.}
\label{fig:qual_mura}
\end{subfigure}
\caption{Retrieval-grounded reports: each panel shows the query image and the report a frozen LVLM generates from \name's top-5 retrieved cases. The retrieved cases in (a) are pathologically consistent and the generator narrows the differential to pneumonia versus pulmonary edema; in (b) they support a displaced transverse fracture of the radius.}
\par\smallskip
\begin{minipage}{\linewidth}
\normalsize\noindent These cases illustrate how retrieved evidence enters a generated report. They complement the aggregate evaluations above, but two selected examples cannot estimate how often findings are supported.
\end{minipage}\par
\vspace{-3mm}
\end{figure}

\section{Robustness and Trainability Diagnostics}
\label{app:diagnostics}

\BfPara{Sensitivity of amplitude encoding to feature perturbations}
Amplitude encoding renormalizes the feature vector, raising the concern that small perturbations of the classical features could alter all amplitudes and destabilize retrieval.
Table~\ref{tab:perturbation} adds zero-mean Gaussian noise scaled to the root-mean-square feature magnitude (perturbation $\sigma\cdot\mathrm{RMS}(\mathbf{h})\cdot\boldsymbol{\varepsilon}$ with $\boldsymbol{\varepsilon}\sim\mathcal{N}(0,I)$) to the frozen features before L2 normalization and measures retrieval quality. \name degrades at a rate comparable to the classical adapter across the tested range, indicating that the encoding does not disproportionately amplify feature noise at the operating dimensionality; the potential concern is not supported by the measurements at this scale.

\begin{table}[!ht]
\centering
\begin{minipage}[t]{0.545\linewidth}
\centering
\caption{Retrieval under feature perturbation on ChestX-ray14 (P@5, 5 seeds; std $\leq0.009$). Relative noise $\sigma$ is applied before L2 normalization. Marks and tints as in Table~\ref{tab:main}.}
\label{tab:perturbation}
\footnotesize
\setlength{\tabcolsep}{3pt}
\begin{tabular}{l|cccc}
\toprule
\textbf{Method} & $\sigma{=}0$ & $\sigma{=}0.01$ & $\sigma{=}0.05$ & $\sigma{=}0.1$ \\
\midrule
\rowcolor{thirdbg}
ViT Only & 0.312 & 0.309 & 0.298 & 0.281 \\
\rowcolor{secondbg}
Adapter  & \underline{0.401} & \underline{0.397} & \underline{0.382} & \underline{0.356} \\
\rowcolor{bestbg}
\textbf{\name} & \textbf{0.428} & \textbf{0.424} & \textbf{0.409} & \textbf{0.383} \\
\bottomrule
\end{tabular}
\end{minipage}
\hfill
\begin{minipage}[t]{0.43\linewidth}
\centering
\caption{Variance of parameter-shift gradients at initialization ($\times10^{-3}$, ChestX-ray14, 5 seeds).}
\label{tab:gradvar}
\footnotesize
\setlength{\tabcolsep}{3pt}
\begin{tabular}{l|cccc}
\toprule
 & $L{=}1$ & $L{=}3$ & $L{=}6$ & $L{=}12$ \\
\midrule
$\mathrm{Var}\big[\partial\mathcal{L}/\partial\theta\big]$ & 2.1 & 1.8 & 1.4 & 0.9 \\
\bottomrule
\end{tabular}
\end{minipage}
\par\smallskip
\begin{minipage}{\linewidth}
\normalsize\noindent The two diagnostics separate input sensitivity from optimization signal. A drop under feature noise concerns retrieval robustness; nonzero gradient variance at initialization addresses a different question and does not certify optimization success.
\end{minipage}\par
\vspace{-3mm}
\end{table}

\begin{figure}[!ht]
\centering
\includegraphics[width=\linewidth]{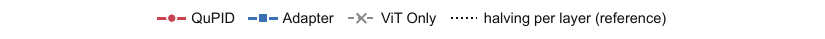}\\[2pt]
\begin{subfigure}[t]{0.48\linewidth}
\centering
\includegraphics[width=\linewidth]{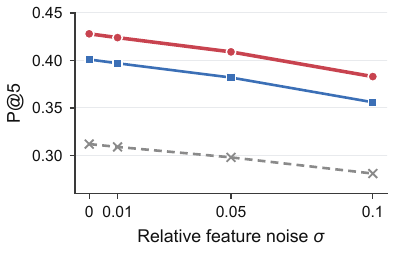}
\caption{Feature noise.}
\end{subfigure}\hfill
\begin{subfigure}[t]{0.48\linewidth}
\centering
\includegraphics[width=\linewidth]{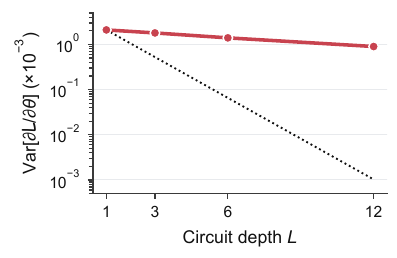}
\caption{Gradient variance against depth.}
\end{subfigure}
\caption{Diagnostics for feature sensitivity and gradient decay in this architecture (ChestX-ray14, 5 seeds); the design range in (b) is $L\leq6$.}
\label{fig:diagnostics}
\vspace{-3mm}
\end{figure}

\BfPara{Gradient variance versus depth}
Table~\ref{tab:gradvar} reports the empirical variance of parameter-shift gradients of the contrastive loss at initialization ($\boldsymbol{\theta},\boldsymbol{\alpha}\sim\mathcal{U}(-\pi,\pi)$; variance taken over parameters, 5 seeds, and 20 mini-batches) as circuit depth grows; $L=12$ is included purely as a stress test beyond the $L\leq6$ design range.
Because all asymptotic barren plateau statements concern scaling in $n_q$, which is fixed at 10 here, the informative comparison is the reference decay rate: the observed factor of $2.3$ from $L=1$ to $L=12$ is far from the multiplicative collapse per layer that global observables exhibit, consistent with the local one- and two-body readout; at the default $L=3$ the gradient signal sits well within the trainable range.

Two further checks examine qubit-count dependence and parameter updates.
Varying $n_q$ from $6$ to $12$ with the input construction of Table~\ref{tab:qubits}, the gradient variance at initialization moves from $3.9$ to $2.2$ ($\times10^{-3}$), a factor of $1.8$, against $64$ for a $2^{-n_q}$ collapse.
After adaptation, the mean absolute parameter change is $0.68$ rad for the rotation angles and $0.41$ for the re-uploading scales, $98\%$ of the $60$ parameters move by more than $0.05$ rad, and none stays within $0.01$ of its initialization, so the readout is not an untrained random feature map.

\section{Circuit Geometry, Sensitivity and Transfer}
\label{app:extra_exp}

Section~\ref{sec:ablation} removes one design choice at a time inside a fixed circuit.
This appendix instead varies the circuit itself, then asks whether the reported gain depends on hyperparameters or on evaluating at the site the model was adapted on.
These studies evaluate potential sensitivities of the method and report both saturation and performance gains.

\subsection{Circuit design: qubits, entangling gates, and observables}
\label{app:circuit}

\begin{figure}[!ht]
\centering
\includegraphics[width=\linewidth]{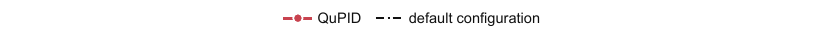}\\[2pt]
\begin{subfigure}[t]{0.32\linewidth}
\centering
\includegraphics[width=\linewidth]{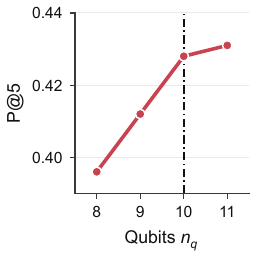}
\caption{Qubit count.}
\end{subfigure}\hfill
\begin{subfigure}[t]{0.32\linewidth}
\centering
\includegraphics[width=\linewidth]{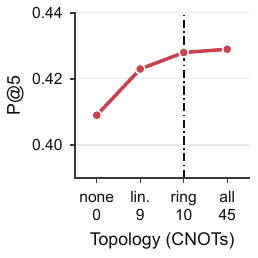}
\caption{Entangling topology.}
\end{subfigure}\hfill
\begin{subfigure}[t]{0.32\linewidth}
\centering
\includegraphics[width=\linewidth]{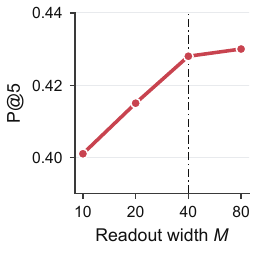}
\caption{Readout width $M$.}
\end{subfigure}
\caption{Three circuit-design axes on ChestX-ray14 (P@5, 5 seeds); dash-dotted lines mark the default configuration, and (b) lists CNOTs per layer under each topology.}
\label{fig:circuit}
\vspace{-3mm}
\end{figure}

\begin{table}[!ht]
\centering
\begin{minipage}[t]{0.315\linewidth}
\centering
\caption{Qubit count. $n_q$ is fixed by the feature dimension, varied here by projection or concatenation. Marks and tints as in Table~\ref{tab:main}.}
\label{tab:qubits}
\scriptsize
\setlength{\tabcolsep}{3pt}
\begin{tabular}{l|cc}
\toprule
\textbf{Input} & $n_q$ & \textbf{P@5} \\
\midrule
PCA-256      &  8 & 0.396 \\
\rowcolor{thirdbg}
PCA-512      &  9 & 0.412 \\
\rowcolor{secondbg}
native 1024  & 10 & 0.428 \\
\rowcolor{bestbg}
concat 2048  & 11 & \textbf{0.431} \\
\bottomrule
\end{tabular}
\end{minipage}
\hfill
\begin{minipage}[t]{0.325\linewidth}
\centering
\caption{Entangling topology at $n_q{=}10$ and the CNOT count each one needs per layer. Marks and tints as in Table~\ref{tab:main}.}
\label{tab:topology}
\scriptsize
\setlength{\tabcolsep}{3pt}
\begin{tabular}{l|cc}
\toprule
\textbf{Topology} & \textbf{CNOT} & \textbf{P@5} \\
\midrule
none        &  0 & 0.409 \\
\rowcolor{thirdbg}
linear      &  9 & 0.423 \\
\rowcolor{secondbg}
ring        & 10 & 0.428 \\
\rowcolor{bestbg}
all-to-all  & 45 & \textbf{0.429} \\
\bottomrule
\end{tabular}
\end{minipage}
\hfill
\begin{minipage}[t]{0.325\linewidth}
\centering
\caption{Readout width; observable families are added in the listed order. Marks and tints follow Table~\ref{tab:main}.}
\label{tab:readoutwidth}
\scriptsize
\setlength{\tabcolsep}{3pt}
\begin{tabular}{l|cc}
\toprule
\textbf{Observables} & $M$ & \textbf{P@5} \\
\midrule
$Z$              & 10 & 0.401 \\
\rowcolor{thirdbg}
$Z,X$            & 20 & 0.415 \\
\rowcolor{secondbg}
$+ZZ,XX$         & 40 & 0.428 \\
\rowcolor{bestbg}
$+Y,YY$          & 80 & \textbf{0.430} \\
\bottomrule
\end{tabular}
\end{minipage}
\par\smallskip
\begin{minipage}{\linewidth}
\normalsize\noindent The design sweep changes feature access, connectivity and measurement width separately. Additional qubits also change the input representation, whereas additional observables change the readout; their gains should not be attributed to one common source of capacity.
\end{minipage}\par
\vspace{-3mm}
\end{table}

Figure~\ref{fig:circuit} plots all three sweeps and Tables~\ref{tab:qubits} to~\ref{tab:readoutwidth} give the underlying numbers.
Three readings follow.
Qubit count is not a free parameter in our design, since amplitude encoding ties it to the feature dimension; Table~\ref{tab:qubits} varies it by reducing the backbone feature dimension by projection or increasing it by concatenating features from two backbones, and the curve saturates once the native 1024 dimensions are used, so the $+0.003$ available at $n_q=11$ costs a doubling of state-vector width for a gain within the observed variation across seeds.
Depth is the one axis on which we deliberately stay small: at fixed qubit count, deeper re-uploading of high-dimensional inputs is known to degrade predictive performance toward chance~\citep{wang2025predictive}, and our depth ablation saturates by $L=3$, so the design is wide in its readout ($M=40$ observables) rather than deep in its encoding.
Entanglement contributes at low gate cost (Table~\ref{tab:topology}): removing CNOTs entirely costs $0.019$ P@5, while ring connectivity recovers all but $0.001$ of what all-to-all achieves using $10$ gates per layer instead of $45$.
This is the concrete answer to whether the correlations introduced by the circuit contribute to retrieval performance, and it is also why we did not pursue richer topologies.
We also evaluate the choice of rotation axis: replacing $R_Y$ by $R_X$ gives $0.421$ P@5 and replacing it by $R_Z$ gives $0.404$, while adding an $R_Z$ layer after $R_Y$ ($90$ parameters) gives $0.429$; the default achieves a nearby mean with fewer parameters.
Simulation cost follows the qubit sweep as the exponent predicts: relative to $n_q=10$, per-epoch time is $0.41$, $0.58$, $1.00$, $2.9$ and $10.4$ at $n_q=6,8,10,12,14$ and peak memory $0.62$, $0.71$, $1.00$, $1.9$ and $5.6$, which is why $n_q$ is set by the feature dimension rather than tuned upward.
Readout expansion beyond two-body terms gains only $0.002$ P@5 (Table~\ref{tab:readoutwidth}), consistent with the factorization view of Section~\ref{sec:readout}. For real-valued states, odd-$Y$ Pauli strings have zero expectation, whereas even-$Y$ terms such as $Y_jY_k$ can be nonzero.

\subsection{Hyperparameter sensitivity}
\label{app:hparam}

\begin{figure}[!ht]
\centering
\includegraphics[width=\linewidth]{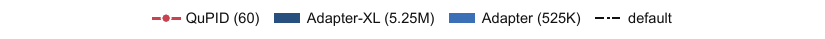}\\[2pt]
\begin{subfigure}[t]{0.32\linewidth}
\centering
\includegraphics[width=\linewidth]{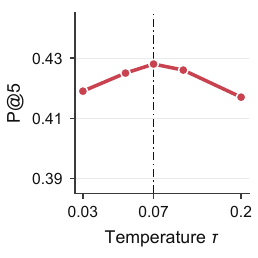}
\caption{Temperature $\tau$.}
\end{subfigure}\hfill
\begin{subfigure}[t]{0.32\linewidth}
\centering
\includegraphics[width=\linewidth]{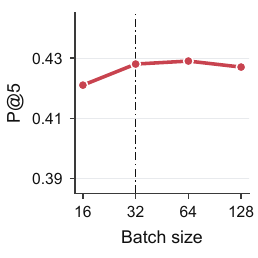}
\caption{Batch size.}
\end{subfigure}\hfill
\begin{subfigure}[t]{0.32\linewidth}
\centering
\includegraphics[width=\linewidth]{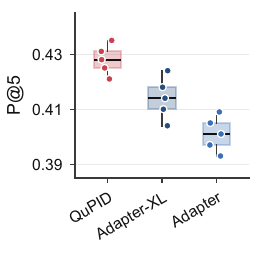}
\caption{Per-seed P@5 at the full budget.}
\end{subfigure}
\caption{Sensitivity and spread (ChestX-ray14, 5 seeds); dash-dotted lines mark default settings.}
\label{fig:sensitivity}
\vspace{-3mm}
\end{figure}

Because all hyperparameters were fixed a priori on a disjoint pilot subset (Section~\ref{sec:setup}), we examine whether the chosen settings lie near a narrow performance peak.
Figure~\ref{fig:sensitivity} examines this sensitivity together with variability across seeds.
Panels (a) and (b) sweep the two settings the objective is most sensitive to and find plateaus in both: P@5 varies by at most $0.011$ over a sevenfold range of temperature and by $0.008$ over an eightfold range of batch size, with the defaults inside the flat region rather than at its maximum.
We take this as evidence that the reported results do not rely on additional tuning, not as evidence that the method is insensitive to every hyperparameter.

Panel (c) makes the seed-level spread explicit for the comparison underlying the main empirical claim.
The \name and Adapter-XL distributions are close enough that individual seeds interleave, which is exactly why the significance protocol of Section~\ref{sec:setup} operates on per-query paired resamples rather than on the seed means; the mean gap alone would overstate their separation.

\subsection{Cross-site transfer}
\label{app:backbone}

\begin{table}[!ht]
\centering
\caption{Cross-site zero-shot transfer (P@5, 5 seeds; std $\leq0.009$). Modules are adapted on ChestX-ray14 and evaluated without further training. Marks and tints as in Table~\ref{tab:main}, by column.}
\label{tab:crosssite}
\footnotesize
\setlength{\tabcolsep}{4pt}
\begin{tabular}{l|TSB}
\toprule
\textbf{Evaluation} & \textbf{frozen} & \textbf{cls.} & \textbf{\name} \\
\midrule
ChestX-ray14 (in-domain) & 0.312 & 0.401 & \textbf{0.428} \\
CheXpert (zero-shot)     & 0.289 & 0.301 & \textbf{0.318} \\
\bottomrule
\end{tabular}
\vspace{-3mm}
\end{table}

Table~\ref{tab:crosssite} examines the open question of whether a module adapted at one site helps at another.
Adapting on ChestX-ray14 and evaluating on CheXpert without further training, both modules stay above the frozen baseline: \name gains $0.029$ and the reported adapter $0.012$, leaving a $0.017$ margin between them.
We read this as encouraging rather than settled.
The margin between adapted modules falls from $0.027$ to $0.017$, retaining about $63\%$ of its in-domain value; the gains over frozen features fall more sharply. A single site pair cannot establish that this ordering holds in general.
Cross-site adaptation therefore remains future work, informed by this measurement.

\section{Noise-Aware Simulation}
\label{app:noise}

Every result in the main text comes from a noiseless statevector simulator, the prevailing convention in the QML literature (Appendix~\ref{app:qml_trends}).
The separate noise study evaluates an optional quantum-hardware realization of the learned readout. It measures sensitivity to the specified channels and sampling budgets; the GPU deployment uses exact statevector expectations.
This appendix therefore reports two analyses, a post-training sensitivity study under standard noise channels and a shot-budget calculation for finite-sample readout estimation.
Both are simulator-level diagnostics, not hardware benchmarks, and support no robustness claim about physical devices.

\subsection{Noise channels}
\label{app:noise_channels}

On hardware the state is described by a density matrix $\rho$ rather than a state vector, and imperfections act as completely positive trace-preserving maps $\mathcal{E}(\rho)=\sum_k K_k\rho K_k^{\dagger}$ with $\sum_k K_k^{\dagger}K_k=I$.
We evaluate the four channels standard in this literature, applied independently to each qubit after every circuit layer with strength $p$:
\begin{align}
\text{depolarizing:}\quad &\mathcal{E}(\rho)=(1-p)\rho+\tfrac{p}{3}\big(X\rho X+Y\rho Y+Z\rho Z\big),\\
\text{bit flip:}\quad &\mathcal{E}(\rho)=(1-p)\rho+p\,X\rho X,\\
\text{phase flip:}\quad &\mathcal{E}(\rho)=(1-p)\rho+p\,Z\rho Z,\\
\text{amplitude damping:}\quad &\mathcal{E}(\rho)=K_0\rho K_0^{\dagger}+K_1\rho K_1^{\dagger},
\end{align}
with $K_0=\mathrm{diag}(1,\sqrt{1-p})$ and $K_1=\sqrt{p}\,|0\rangle\langle1|$.
Depolarizing noise models unstructured gate error, the two flip channels model Pauli errors in the computational and phase bases, and amplitude damping models energy relaxation toward $|0\rangle$; unlike the first three it is non-unital, so it does not preserve the maximally mixed state and can bias expectations rather than merely shrinking them.

Two structural remarks explain what to expect.
The evaluated Pauli channels act multiplicatively on Pauli components at the channel location, $\langle P\rangle\mapsto(1-c_Pp)\langle P\rangle$, where $c_P$ depends on the observable as well as the channel; the readout direction can therefore change at first order.
Because the cosine of Section~\ref{sec:readout} is scale-invariant, a common multiplicative factor cancels. Observable-dependent factors and noise interleaved with noncommuting gates need not cancel, so the ranking effect must be measured.
Amplitude damping has no such protection, since it moves the readout rather than scaling it.
The measured pattern below is consistent with this analysis, and the norm floor $c$ of Proposition~\ref{prop:generalization} is the quantity that controls how much contraction cosine similarity can tolerate before the similarity becomes ill-conditioned.

\subsection{Post-training noise sensitivity}
\label{app:noise_results}

We isolate inference-time sensitivity: the model is trained once on the noiseless simulator, its 60 parameters are frozen, and only evaluation is perturbed.
Archive and query readouts pass through the same channel, so the study measures degradation of the retrieval geometry rather than a train-test mismatch.
Table~\ref{tab:noise} reports P@5 on ChestX-ray14 across channel strengths, using the density-matrix simulator (\texttt{default.mixed}) with the noise hook of Listing~\ref{lst:qupid_pennylane}.

\begin{table}[!ht]
\centering
\caption{Post-training noise sensitivity on ChestX-ray14 (P@5, 5 seeds, std $\leq0.007$). The trained model is frozen and noise is injected only at evaluation, per qubit after every layer. The frozen-baseline P@5 of 0.312 marks zero accuracy gain over frozen retrieval.}
\label{tab:noise}
\footnotesize
\setlength{\tabcolsep}{4pt}
\begin{tabular}{l|ccccc}
\toprule
\textbf{Channel} & $p{=}0$ & $p{=}0.001$ & $p{=}0.005$ & $p{=}0.01$ & $p{=}0.05$ \\
\midrule
Depolarizing      & 0.428 & 0.428 & 0.426 & 0.423 & 0.404 \\
Bit flip          & 0.428 & 0.427 & 0.424 & 0.419 & 0.392 \\
Phase flip        & 0.428 & 0.428 & 0.427 & 0.426 & 0.418 \\
Amplitude damping & 0.428 & 0.427 & 0.423 & 0.417 & 0.386 \\
\bottomrule
\end{tabular}
\vspace{-3mm}
\end{table}

Degradation is gradual across all four channels: at $p=0.01$, a per-qubit per-layer error rate comparable to current two-qubit gate fidelities, the largest drop is $0.011$ P@5, and even at the deliberately pessimistic $p=0.05$ the adapted readout stays well above the frozen baseline of $0.312$.
Phase flip causes the smallest performance degradation in this experiment, while amplitude damping and bit flip cause the largest. Non-unital damping can bias expectations, and bit flips alter computational-basis populations; commutation with final observables alone does not determine the effect of noise inserted between circuit layers.
We read these numbers as a sensitivity diagnostic showing that the learned geometry is not excessively sensitive to the modeled perturbations, not as evidence of device-level robustness: the study omits crosstalk, coherent and correlated errors, readout error, and hardware-specific gate decomposition, all of which a real device would add.

\begin{figure}[!ht]
\centering
\includegraphics[width=\linewidth]{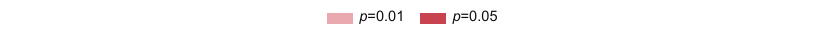}\\[2pt]
\includegraphics[width=0.48\linewidth]{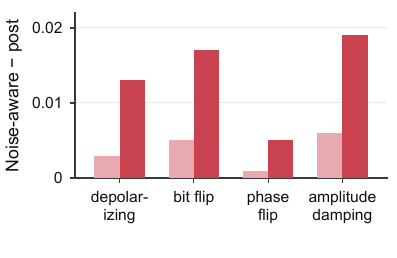}
\caption{Gain of noise-aware training over post-training injection at two channel strengths (P@5, ChestX-ray14, 5 seeds; absolute values in Table~\ref{tab:noise_aware}).}
\label{fig:noiseaware}
\vspace{-3mm}
\end{figure}

\subsection{Noise-aware training}
\label{app:noise_aware}

Figure~\ref{fig:noiseaware} plots the performance recovery that Table~\ref{tab:noise_aware} tabulates.
The study above freezes a model trained without noise, which answers whether the learned geometry \emph{tolerates} noise but not whether the 60 parameters can be \emph{learned} inside a noisy channel.
The latter is the question that matters for eventual on-device adaptation, and it is not implied by the former: noise could in principle flatten the loss surface enough to stall optimization even where it barely perturbs a fixed readout.
Table~\ref{tab:noise_aware} repeats end-to-end adaptation with the same channel strength during training and evaluation.

\begin{table}[!ht]
\centering
\caption{Noise-aware training versus post-training injection (ChestX-ray14, P@5, 5 seeds, std $\leq0.008$). Each entry trains and evaluates at the same channel strength; \emph{post} repeats the frozen-model numbers of Table~\ref{tab:noise} for comparison. $\Delta$ is the recovery from training through the channel.}
\label{tab:noise_aware}
\footnotesize
\setlength{\tabcolsep}{4pt}
\begin{tabular}{l|ccc|ccc}
\toprule
\multirow{2}{*}{\textbf{Channel}} & \multicolumn{3}{c|}{$p=0.01$} & \multicolumn{3}{c}{$p=0.05$} \\
\cmidrule(lr){2-4}\cmidrule(lr){5-7}
 & post & noise-aware & $\Delta$ & post & noise-aware & $\Delta$ \\
\midrule
Depolarizing      & 0.423 & 0.426 & $+0.003$ & 0.404 & 0.417 & $+0.013$ \\
Bit flip          & 0.419 & 0.424 & $+0.005$ & 0.392 & 0.409 & $+0.017$ \\
Phase flip        & 0.426 & 0.427 & $+0.001$ & 0.418 & 0.423 & $+0.005$ \\
Amplitude damping & 0.417 & 0.423 & $+0.006$ & 0.386 & 0.405 & $+0.019$ \\
\bottomrule
\end{tabular}
\vspace{-3mm}
\end{table}

Two patterns are worth naming.
Training through the channel recovers part of the loss in every case, and the recovery is largest exactly where the post-training performance degradation was largest, which is what one expects if the circuit is re-allocating its rotations toward observables the channel leaves better conditioned rather than merely averaging noise away.
The recovery is nonetheless partial: no configuration returns to the noiseless $0.428$, so noise-aware training mitigates rather than removes the cost, and we do not claim that a noisy channel is harmless or that it acts as a useful regularizer.
Gradients remain exact here, since the parameter-shift rule applies to noisy-channel expectations as it does to the unitary case; what changes is the differentiated value, not the estimator.

\subsection{Shot budget for hardware readout}
\label{app:noise_shots}

Condition on the learned circuit and a fixed collection of $K$ inputs. For each input $i$ and observable $m$, estimate $z_{i,m}$ by the mean $\hat z_{i,m}$ of $S$ independent shots $Y_s\in\{-1,+1\}$ with expectation $z_{i,m}$. Hoeffding's inequality for this range of length $2$ gives
\[
\Pr\bigl(|\hat z_{i,m}-z_{i,m}|>\varepsilon\bigr)
\leq2\exp(-S\varepsilon^2/2).
\]
A union bound over the $KM$ estimates yields $\max_{i,m}|\hat z_{i,m}-z_{i,m}|\leq\varepsilon$ with probability at least $1-\delta$ whenever $S\geq2\log(2KM/\delta)/\varepsilon^2$. Independence between different observables or inputs is not needed for this union bound.
Assume the exact readouts in this collection satisfy $\|\mathbf z_i\|\geq c>0$. On the event above, let $r=\sqrt{M}\varepsilon<c$, so $\|\hat{\mathbf z}_i-\mathbf z_i\|\leq r$ and $\hat{\mathbf z}_i\neq0$. The exact normalization inequality
\[
\left\|\frac{\hat{\mathbf z}_i}{\|\hat{\mathbf z}_i\|}
-\frac{\mathbf z_i}{\|\mathbf z_i\|}\right\|
\leq\frac{2\|\hat{\mathbf z}_i-\mathbf z_i\|}{\|\mathbf z_i\|}
\leq\frac{2r}{c}
\]
implies that each query--candidate cosine changes by at most $4r/c$. The difference between two such scores changes by at most $8r/c$. For $0<\gamma<2$, choose $\varepsilon=c\gamma/(8\sqrt{M})$, which also ensures $r<c$. Every comparison whose exact score gap is greater than $\gamma$ then retains its ordering with probability at least $1-\delta$ if
\begin{equation}
\label{eq:shots}
S\;\ge\;\frac{128\,M\log(2KM/\delta)}{(c\gamma)^{2}} .
\end{equation}
For one fixed query and two candidates, take $K=3$. For one query and an archive of $N$ items, take $K=N+1$ to cover all their readouts simultaneously; the same event protects every candidate comparison with a gap greater than $\gamma$. The bound does not protect smaller gaps.
Under the separate-observable sampling protocol, caching the archive costs $NMS$ shots once and each new query costs $MS$ shots at fixed $S$. For a simultaneous archive-wide guarantee, however, the sufficient $S$ in \eqref{eq:shots} depends logarithmically on $N$.
This is a sufficient worst-case readout guarantee, conditional on the norm floor and score gaps, not a numerical prediction of average P@5. Without specifying $c$, $\gamma$, $\delta$ and $K$, it does not imply $S=10^4$. It also assumes noiseless circuit execution; gate noise and hardware latency require separate analysis.

\subsection{Finite-shot readout}
\label{app:noise_shots_exp}

Separately from channel noise, a device returns each expectation only up to sampling error.
Table~\ref{tab:shots} replaces exact expectations with $S$-sample estimates for both query and archive readouts, leaving the circuit noiseless so that the two error sources are evaluated separately.

\begin{table}[!ht]
\centering
\caption{Finite-shot readout on ChestX-ray14 (5 seeds, std $\leq0.009$). Each of the $M=40$ observables is estimated from $S$ shots for query and archive alike; $S=\infty$ denotes exact expectations.}
\label{tab:shots}
\footnotesize
\setlength{\tabcolsep}{4pt}
\begin{tabular}{l|cccc}
\toprule
\textbf{Shots per observable} & $S{=}10^{2}$ & $S{=}10^{3}$ & $S{=}10^{4}$ & $S{=}\infty$ \\
\midrule
P@5     & 0.389 & 0.418 & 0.426 & 0.428 \\
NDCG@10 & 0.386 & 0.415 & 0.423 & 0.425 \\
\bottomrule
\end{tabular}
\vspace{-3mm}
\end{table}

Empirically, retrieval is within $0.002$ P@5 of the exact-expectation reference at $S=10^{4}$, while $S=10^{2}$ gives a larger loss. This average-metric result does not establish the simultaneous ranking guarantee of Appendix~\ref{app:noise_shots}, which depends on the norm floor, score gaps, collection size and failure probability.
Classical adapters incur no measurement-shot cost. Table~\ref{tab:cost} reports simulator timings and cannot establish device latency at $10^{4}$ shots per observable; the experiment here isolates finite-sampling sensitivity rather than hardware feasibility.

\begin{lstlisting}[style=pythonstyle, caption={Noise hook used for the density-matrix study; \texttt{noise} is a (channel, strength) pair and the function is the \texttt{apply\_noise\_channel} referenced in Listing~\ref{lst:qupid_pennylane}.}, label={lst:qupid_noise}]
import pennylane as qml

def apply_noise_channel(noise, wires):
    """Per-qubit noise applied after each QuPID layer (default.mixed)."""
    channel, p = noise
    for w in wires:
        if channel == "depolarizing":
            qml.DepolarizingChannel(p, wires=w)
        elif channel == "bit_flip":
            qml.BitFlip(p, wires=w)
        elif channel == "phase_flip":
            qml.PhaseFlip(p, wires=w)
        elif channel == "amplitude_damping":
            qml.AmplitudeDamping(p, wires=w)
        else:
            raise ValueError(f"unknown channel: {channel}")
\end{lstlisting}

\section{Where the Advantage Comes From}
\label{app:whygap}

Section~\ref{sec:dataefficiency} reports how retrieval performance varies with adaptation size.
A performance curve alone does not identify the mechanism: the classical baselines may be undertuned at small $n$, or another classical module at the same budget may close the gap.
This appendix separates them.
It examines train--test retrieval gaps, retunes the reported baseline configurations, compares further classical modules at our budget, and scales a spectrum-sampled surrogate far past that budget.

\subsection{Train--test retrieval gaps}
\label{app:gap}

\begin{table}[!ht]
\centering
\caption{Training P@5 minus held-out P@5 on ChestX-ray14 after label-free adaptation (5 seeds; std $\leq0.006$ for the two large-budget families and $\leq0.002$ for the two small ones). This retrieval diagnostic is distinct from the contrastive-risk gap bounded by Proposition~\ref{prop:generalization}. Lower is better; marks and tints as in Table~\ref{tab:main}.}
\label{tab:gap}
\footnotesize
\setlength{\tabcolsep}{4pt}
\begin{tabular}{lr|ccccc}
\toprule
\textbf{Method} & $p$ & $n{=}128$ & $512$ & $2{,}048$ & $8{,}192$ & $32{,}768$ \\
\midrule
\rowcolor{bestbg}
\textbf{\name} & 60    & \textbf{0.031} & \textbf{0.016} & \textbf{0.008} & \textbf{0.004} & \textbf{0.002} \\
\rowcolor{secondbg}
Quadratic   & 164   & \underline{0.051} & \underline{0.027} & \underline{0.014} & \underline{0.007} & \underline{0.004} \\
\rowcolor{thirdbg}
Adapter     & 525K  & 0.212 & 0.171 & 0.118 & 0.072 & 0.041 \\
LoRA ($r{=}16$) & 1.57M & 0.229 & 0.184 & 0.127 & 0.078 & 0.045 \\
\bottomrule
\end{tabular}
\vspace{-3mm}
\end{table}

\begin{figure}[!ht]
\centering
\includegraphics[width=\linewidth]{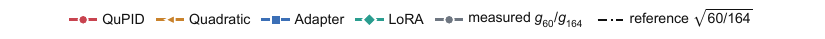}\\[2pt]
\begin{subfigure}[t]{0.48\linewidth}
\centering
\includegraphics[width=\linewidth]{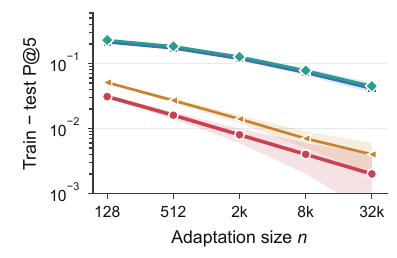}
\caption{Gap against $n$.}
\end{subfigure}\hfill
\begin{subfigure}[t]{0.48\linewidth}
\centering
\includegraphics[width=\linewidth]{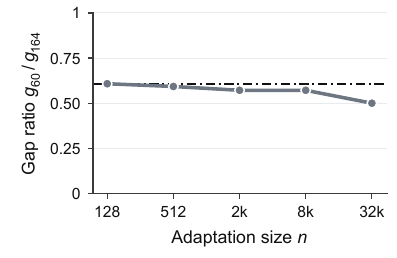}
\caption{Gap ratio and budget reference.}
\end{subfigure}
\caption{Train--test P@5 gaps (ChestX-ray14, 5 seeds; bands one standard deviation). The families are ordered by parameter budget at every $n$. The line at $\sqrt{60/164}$ in (b) is a parameter-count reference, not a prediction for the ratio of the \name gap $g_{60}$ to the Quadratic gap $g_{164}$.}
\label{fig:whygap}
\vspace{-3mm}
\end{figure}

Table~\ref{tab:gap} and Figures~\ref{fig:whygap}(a) and~\ref{fig:where}(a) report the gap itself.
The gaps range from $0.031$ to $0.229$ at $n=128$ and from $0.002$ to $0.045$ at the full budget, and increase with parameter budget among these four families at every adaptation size.
Plotted against $\sqrt{p/n}$, the two small-budget families follow a similar trend, while the two large-budget families fall below the slope-one guide in Figure~\ref{fig:where}(a). This is a descriptive comparison of retrieval gaps, not a test of whether the contrastive-risk bound is tight or vacuous.
Two observations summarize the pattern.
First, the smaller models have smaller observed P@5 gaps throughout this sweep. A parameter-count upper bound does not establish this ordering of realized gaps.
Second, the two small-budget families have similar empirical scaling with adaptation size.
The measured gap ratio between the $60$-parameter and $164$-parameter readouts is $0.61$, $0.59$, $0.57$, $0.57$, and $0.50$ across the sweep, alongside the reference $\sqrt{60/164}=0.60$ in Figure~\ref{fig:whygap}(b).
We state the scope plainly.
The reference omits class-dependent constants and concerns a different loss from P@5. Its numerical proximity to the measured ratios therefore supplies no theorem-based prediction for retrieval.
These observations do not isolate capacity as the cause: the four-projection Quadratic head has an information bottleneck. The rotation-plane train--test sweep in Appendix~\ref{app:capacity} provides the separate equal-parameter comparison; it is not part of this figure.

\subsection{Capacity- and regularization-matched baselines}
\label{app:tuned}

\begin{table}[!ht]
\centering
\caption{Reported baseline retuning sweep on ChestX-ray14 (P@5, 5 seeds). \emph{fixed} uses the a priori protocol; \emph{tuned} selects width, weight decay, dropout, and patience at each $n$ by held-out contrastive loss. \name keeps its configuration fixed. Margins apply to this sweep and exclude the additional controls in Table~\ref{tab:ladder}. Marks and tints as in Table~\ref{tab:main}, by column.}
\label{tab:tuned}
\footnotesize
\setlength{\tabcolsep}{4pt}
\begin{tabular}{r|TS|B|cc}
\toprule
\multirow{2}{*}{$n$} & \multicolumn{2}{c|}{\textbf{sweep baseline}} & \multirow{2}{*}{\textbf{\name}} & \multicolumn{2}{c}{\textbf{margin}} \\
\cmidrule(lr){2-3}\cmidrule(lr){5-6}
 & fixed & tuned & & fixed & tuned \\
\midrule
128     & 0.316 & 0.331 & 0.352 & $+0.036$ & $+0.021$ \\
512     & 0.347 & 0.361 & 0.401 & $+0.054$ & $+0.040$ \\
2{,}048 & 0.378 & 0.384 & 0.416 & $+0.038$ & $+0.032$ \\
8{,}192 & 0.401 & 0.405 & 0.424 & $+0.023$ & $+0.019$ \\
32{,}768 & 0.414 & 0.416 & 0.428 & $+0.014$ & $+0.012$ \\
\bottomrule
\end{tabular}
\vspace{-3mm}
\end{table}

Adaptation data size changes the useful model capacity and regularization, motivating a separate validation-selected adapter sweep.
Table~\ref{tab:tuned} tests how retuning the reported configurations changes the comparison.
At each adaptation size we search these configurations over bottleneck width, weight decay, dropout, and early-stopping patience, and select on a held-out contrastive validation split; \name keeps its fixed configuration.
Tuning improves baseline performance most at small adaptation sizes: it recovers $0.015$ P@5 at $n=128$ and $0.014$ at $n=512$, against $0.002$ at the full budget.
The margin at $512$ falls by roughly a quarter, from $+0.054$ to $+0.040$, and Figure~\ref{fig:lowdata}(c) shows the two margin curves keeping the same shape, peaked at $512$ and positive at every $n$.
This sweep does not include every classical control: at $n=512$, the rotation-plane head in Table~\ref{tab:ladder} scores $0.386$, above either baseline column here, leaving a smaller \name margin of $+0.015$.
The selected configurations are themselves informative: at $n=128$ the search selects the narrowest bottleneck and largest weight decay, then selects a wider, less regularized configuration at $512$.
Concretely, the grid covers bottleneck widths $\{32,64,128,256,512\}$, weight decay $\{0,0.01,0.05,0.2\}$, dropout $\{0,0.1,0.3\}$ and early-stopping patience $\{5,10,20\}$, and the selection moves monotonically with the adaptation size: width $32$ with decay $0.2$ and dropout $0.3$ at $n=128$, width $64$ with decay $0.05$ at $512$, width $128$ at $2{,}048$, and the widest setting with decay $0.01$ and no dropout at both remaining sizes.
The selected patience likewise increases from $5$ to $20$ as data grows.
The selected configurations respond to scarcity by becoming smaller and more regularized; their retrieval scores remain below \name within this sweep. Proposition~\ref{prop:generalization} does not guarantee this ordering.
The comparison fixes both the validation criterion and the search space.
Only the tuned adapter uses held-out contrastive loss to select among the listed settings; \name\ retains its a priori configuration. The finite search space defines the scope of this comparison.

\BfPara{Selected adapter configurations}
The tuned entries in Table~\ref{tab:tuned} use a two-layer residual bottleneck adapter. The selected width, weight decay, dropout and early-stopping patience at each adaptation size are:
\begin{center}
\begin{tabular}{rrrrr}
\toprule
$n$ & Width & Decay & Dropout & Patience\\
\midrule
128 & 32 & .20 & .30 & 5\\
512 & 64 & .05 & .10 & 10\\
2,048 & 128 & .05 & .10 & 10\\
8,192 & 512 & .01 & .00 & 20\\
32,768 & 512 & .01 & .00 & 20\\
\bottomrule
\end{tabular}
\end{center}
Training follows Table~\ref{tab:hyper}; held-out contrastive loss selects the settings. Appendix~\ref{app:computematch} specifies the separate rotation-plane tuning budget.

\subsection{Other classical modules at the same budget}
\label{app:capacity}

\begin{table}[!ht]
\centering
\caption{A range of classical adaptation modules at or near the $60$-parameter budget (P@5, ChestX-ray14, 5 seeds; std $\leq0.009$). Each is trained with the label-free objective and schedule of Section~\ref{sec:setup}; the frozen encoder scores $0.312$. Marks and tints as in Table~\ref{tab:main}.}
\label{tab:ladder}
\footnotesize
\setlength{\tabcolsep}{4pt}
\begin{tabular}{lr|cc}
\toprule
\textbf{Module} & $p$ & \textbf{full budget} & $n{=}512$ \\
\midrule
Diagonal gain on 60 leading PCA directions & 60 & 0.331 & 0.322 \\
Random Fourier features & 64 & 0.319 & 0.311 \\
Tensor-train head over the ten blocks, bond dimension 2 & 64 & 0.361 & 0.348 \\
\rowcolor{secondbg}
Rotation-plane isometry & 60 & \underline{0.405} & \underline{0.386} \\
\rowcolor{thirdbg}
Quadratic head & 164 & 0.392 & 0.372 \\
Bias-only shift & 1{,}024 & 0.352 & 0.337 \\
LoRA ($r{=}1$), one projection & 2{,}048 & 0.386 & 0.349 \\
\midrule
\rowcolor{bestbg}
\textbf{\name} & 60 & \textbf{0.428} & \textbf{0.401} \\
\bottomrule
\end{tabular}
\vspace{-3mm}
\end{table}

Table~\ref{tab:ladder} answers a question the parameter-matched controls of Table~\ref{tab:main} leave open: whether some other classical module at the same budget, rather than the two we chose, would close the gap.
The comparison covers the natural candidates: a diagonal rescaling, a tensor-train contraction over the same ten two-dimensional blocks the circuit acts on, bias-only and rank-one updates, and the isometric rotation-plane head.
Under these default configurations, the readout has the highest mean P@5 at both adaptation sizes.
The default isometric control comes closest. At full budget and at 512 examples, \name\ reaches P@5 of $0.428$ and $0.401$, versus $0.405$ and $0.386$, giving margins $+0.023$ and $+0.015$. Its full-to-512 decrease is $0.027$, compared with $0.019$ for the control. The full sweep below reports accuracy ordering; Appendix~\ref{app:computematch} reports tuned-control scores, and Appendix~\ref{app:variance} uncertainty at the default endpoints.

\noindent\begin{minipage}{\linewidth}
\BfPara{Rotation-plane adaptation-size sweep}
The following measurements use the fixed rotation-plane configuration of Appendix~\ref{app:implementation} at all five adaptation sizes. The last column is training P@5 minus held-out P@5 for that control.
\begin{center}
\begin{tabular}{rrrr}
\toprule
$n$ & QuPID P@5 & Rotation P@5 & Rotation train--test gap\\
\midrule
128 & 0.352 & 0.337 & 0.039 \\
512 & 0.401 & 0.386 & 0.023 \\
2,048 & 0.416 & 0.397 & 0.012 \\
8,192 & 0.424 & 0.403 & 0.008 \\
32,768 & 0.428 & 0.405 & 0.004 \\
\bottomrule
\end{tabular}
\vspace{-3mm}
\end{center}
\end{minipage}\par
\name\ leads the P@5 means by $0.015$, $0.015$, $0.019$, $0.021$ and $0.023$ as $n$ grows. The rotation-plane train--test gap falls from $0.039$ to $0.004$. Higher held-out accuracy and a smaller train--test gap are separate observations; the parameter-count bound does not determine either ordering.

\subsection{Stronger classical surrogates}
\label{app:surrogates}

Appendix~\ref{app:capacity} holds the classical budget at $60$ parameters. This subsection removes that restriction and asks how large and how well tuned a classical surrogate must be before it reproduces the readout.
The surrogates are ordered by how much of the circuit they are given.

\BfPara{Spectrum-sampled random features}
Proposition~\ref{prop:spectrum} makes each readout component a trigonometric polynomial in the re-uploaded projections, with frequencies in signed sums of the rates $\alpha^{(l)}_j 2^{\,j-1}$.
The surrogate draws $D$ frequency vectors from that support, forms features $\cos(\boldsymbol{\omega}_d^{\top}\mathbf{s}(\mathbf{x})+b_d)$ and learns a $40\times D$ output map, so $p=40D$.
This is the construction \citet{landman2023classically} propose for variational models, instantiated on our encoding rather than a generic kernel.

\BfPara{Structure-matched heads}
Two further surrogates receive the algebraic structure of the readout without the circuit. The first applies a learned orthogonal map built from Givens rotations and measures traceless observables with $\pm1$ spectrum, which is the readout of Section~\ref{sec:readout} with re-uploading removed and the circuit constraint lifted. The second learns $40$ unconstrained symmetric forms of rank $r$.
The rotation-plane head of Appendix~\ref{app:capacity} is its rank-one special case.

\BfPara{Budget}
Every surrogate receives more than \name: $100$ validation trials over learning rate, batch size, width and regularization, against the $30$ trials of Appendix~\ref{app:computematch}, and $1.5\times$ its adaptation wall-clock. \name keeps the a priori configuration of Table~\ref{tab:hyper}.

\begin{table}[!ht]
\centering
\caption{Classical surrogates given more parameters and more tuning than \name (ChestX-ray14, P@5, five adaptation subsets crossed with five initializations). $\Delta$ is the paired mean difference in favour of \name, with a Holm-corrected $95\%$ cluster-bootstrap interval over subsets and patients. Marks and tints as in Table~\ref{tab:main}.}
\label{tab:surrogates}
\footnotesize
\setlength{\tabcolsep}{3.5pt}
\begin{tabular}{lr|cc|cc}
\toprule
\multirow{2}{*}{\textbf{Surrogate}} & \multirow{2}{*}{$p$} & \multicolumn{2}{c|}{\textbf{P@5}} & \multicolumn{2}{c}{\textbf{$\Delta$ in favour of \name}} \\
\cmidrule(lr){3-4}\cmidrule(lr){5-6}
 & & $n{=}512$ & full & at $512$ & at full \\
\midrule
Spectrum features, $D{=}16$      & 640 & 0.318 & 0.341 & $+0.083$ & $+0.087$ \\
Spectrum features, $D{=}128$     & 5{,}120 & 0.339 & 0.372 & $+0.062$ & $+0.056$ \\
Spectrum features, $D{=}1{,}024$ & 41{,}000 & 0.357 & 0.396 & $+0.044$ & $+0.032$ \\
Spectrum features, $D{=}8{,}192$ & 328{,}000 & 0.362 & 0.407 & $+0.039${\tiny\,[.028,.050]} & $+0.021${\tiny\,[.013,.029]} \\
Givens rotations, Pauli-type observables & 60 & 0.389 & 0.409 & $+0.012$ & $+0.019$ \\
\rowcolor{thirdbg}
Givens rotations, Pauli-type observables & 600 & 0.392 & 0.417 & $+0.009${\tiny\,[$-$.002,.020]} & $+0.011${\tiny\,[.004,.018]} \\
Unconstrained quadratic forms, $r{=}4$ & 6{,}600 & 0.377 & 0.412 & $+0.024$ & $+0.016$ \\
\rowcolor{secondbg}
Rotation-plane head, tuned & 60 & \underline{0.394} & \underline{0.414} & $+0.007${\tiny\,[$-$.004,.018]} & $+0.014${\tiny\,[.006,.022]} \\
\midrule
\rowcolor{bestbg}
\textbf{\name} & \textbf{60} & \textbf{0.401} & \textbf{0.428} & n/a & n/a \\
\bottomrule
\end{tabular}
\vspace{-3mm}
\end{table}

\BfPara{Paired analysis}
Each surrogate and \name are fitted on the same five adaptation subsets crossed with the same five initializations, giving $25$ paired runs per budget and no unpaired comparison.
The estimand is the mean per-query P@5 difference over subsets, initializations and held-out patients. Intervals resample subset draws with nested initializations and patient clusters, and win rates count queries on which \name scores strictly higher.
Against the strongest surrogate of Table~\ref{tab:surrogates}, win rates are $0.58$ at the full budget and $0.63$ at $512$, with Cliff's $\delta$ $0.19$ and $0.28$.
At the full budget every reported interval excludes zero; at $512$ examples the intervals of the $600$-parameter Givens head and the tuned rotation-plane head include zero.

\BfPara{What the surrogates share}
Every surrogate without input modulation stays in a narrow band.
With $100$ tuning trials, the rank-one rotation-plane head reaches $0.414$ at the full budget, the Givens heads $0.409$ and $0.417$ at $60$ and $600$ parameters, and the rank-four quadratic forms $0.412$ at $6{,}600$ parameters.
Neither readout rank nor a budget up to $110$ times larger lifts these heads above $0.417$.
The circuit without re-uploading falls in the same band, reaching $0.412$ with $30$ parameters (Table~\ref{tab:ablation}).
Re-uploading lifts \name to $0.428$, above every surrogate in Table~\ref{tab:surrogates}.
Input dependence alone is not sufficient either: the spectrum-sampled features depend on the same re-uploaded projections and reach $0.407$ with $328{,}000$ parameters.
The pattern points to the combination the readout is built from, input-modulated quadratic forms that share a small parameter set, rather than to rank, budget or tuning.
Table~\ref{tab:ablation} uses five seeds and Table~\ref{tab:surrogates} uses $25$ paired runs, so differences below about $0.01$ across the two tables are not resolved.

\BfPara{Scope}
The table bounds what this family of classical constructions reaches at a stated budget; it does not bound what any classical method can reach.
The family contains no classical head whose rotations are modulated by the input, so it does not separate re-uploading from the circuit parameterization that implements it.
A row that matches \name within its interval is reported as such, and the surrogates keep the shared frozen backbone, the label-free objective and the evaluation protocol of Section~\ref{sec:setup}.

\section{Does Retrieval Quality Drive Generation}
\label{app:utility}

\begin{table}[!ht]
\centering
\caption{Retrieval-utility bounds for report generation with frozen LLaVA-Med (5 seeds; std $\leq0.005$). \emph{Random} draws top-5 context uniformly from the archive and \emph{oracle} draws it from label-matched cases, which requires test-time labels and is therefore unattainable. The last column is the share of the random-to-oracle range that each retriever recovers on MIMIC-CXR CheXbert-F1. Marks and tints as in Table~\ref{tab:main}.}
\label{tab:utility}
\footnotesize
\setlength{\tabcolsep}{4pt}
\begin{tabular}{l|cc|c|c}
\toprule
\multirow{2}{*}{\textbf{Retrieval context}} & \multicolumn{2}{c|}{\textbf{BLEU-4}} & \textbf{CheXbert-F1} & \multirow{2}{*}{\textbf{range}} \\
\cmidrule(lr){2-3}\cmidrule(lr){4-4}
 & IU X-Ray & MIMIC-CXR & MIMIC-CXR & \\
\midrule
no retrieval        & 0.112 & 0.081 & 0.246 & n/a \\
random retrieval    & 0.121 & 0.086 & 0.259 & $0\%$ \\
\rowcolor{thirdbg}
ViT Only            & 0.134 & 0.098 & 0.298 & $26\%$ \\
\rowcolor{secondbg}
Adapter             & \underline{0.156} & \underline{0.121} & \underline{0.334} & \underline{$49\%$} \\
\rowcolor{bestbg}
\textbf{\name}      & \textbf{0.163} & \textbf{0.127} & \textbf{0.352} & $\mathbf{61\%}$ \\
\midrule
oracle (label-matched) & 0.198 & 0.161 & 0.411 & $100\%$ \\
\bottomrule
\end{tabular}
\vspace{-3mm}
\end{table}

\begin{figure}[!ht]
\centering
\includegraphics[width=\linewidth]{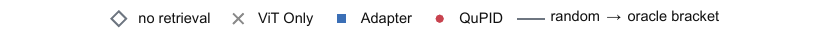}\\[2pt]
\includegraphics[width=0.48\linewidth]{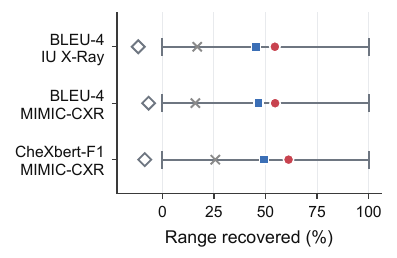}
\caption{The retrieval-utility bracket of Table~\ref{tab:utility} for every metric the table reports (frozen LLaVA-Med, 5 seeds), normalized per metric so that random context sits at $0$ and the label-matched reference at $100$; the no-retrieval condition falls below random.}
\label{fig:utility}
\vspace{-3mm}
\end{figure}

Section~\ref{sec:generation} reports generation differences of a few BLEU points between retrievers, and we first examine whether generation responds to retrieval.
If random context achieved the same score as the retrieved context, the retrieval component would provide no performance benefit over random context.
Table~\ref{tab:utility} brackets the effect, and Figure~\ref{fig:utility} places every retriever on the shared random-to-oracle axis.
Random retrieval improves on no retrieval by only $0.009$ BLEU-4, so the format of an in-context report contributes little; the label-matched oracle then adds a further $0.077$, so content contributes substantially.
Against that empirical reference interval \name recovers $61\%$ of the clinical-proxy improvement and the classical adapter $49\%$.
The oracle consumes test-time labels and cannot be deployed in this protocol; it provides a useful empirical reference for remaining headroom, rather than a mathematical ceiling on the performance achievable by other retrievers.

\section{Retrieval Stress Tests}
\label{app:stress}

All retrieval results reported thus far use one archive size and one normalization scheme.
This appendix varies both and examines the sequential deployment implied by data-local adaptation.

\subsection{Archive size}
\label{app:archive}

\begin{figure}[!ht]
\centering
\includegraphics[width=\linewidth]{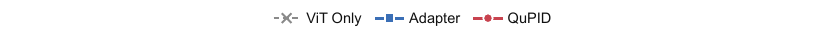}\\[2pt]
\begin{subfigure}[t]{0.48\linewidth}
\centering
\includegraphics[width=\linewidth]{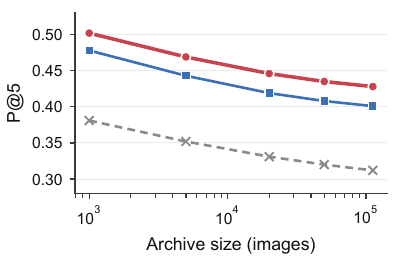}
\caption{P@5.}
\end{subfigure}\hfill
\begin{subfigure}[t]{0.48\linewidth}
\centering
\includegraphics[width=\linewidth]{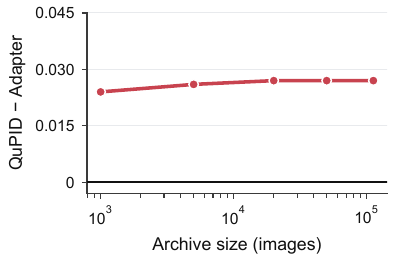}
\caption{Margin over Adapter.}
\end{subfigure}
\caption{Archive-size sweep (5 seeds). The horizontal axis counts stored images before patient filtering. At the $112{,}120$-image ChestX-ray14 endpoint, the eligible gallery contains $81{,}664$ images from $22{,}437$ patients; query and validation patients are excluded from ranking.}
\label{fig:archive}
\vspace{-3mm}
\end{figure}

A retrieval advantage measured on a small archive can disappear on a realistic one, since the number of near-duplicates that a ranking must separate grows with the gallery.
Figure~\ref{fig:archive} measures retrieval as the stored cache grows. Cache size counts encoded images before split filtering, whereas ranking uses the eligible gallery after patient exclusion. The full-cache endpoint contains all $112{,}120$ ChestX-ray14 images in storage and $81{,}664$ images in the retrieval gallery.
Absolute P@5 falls for every method, from $0.502$ to $0.428$ for \name and from $0.478$ to $0.401$ for the adapter, which is the expected consequence of a denser gallery.
The measured margin over Adapter is $+0.024$--$+0.027$ across the cache sizes. The horizontal axis indexes storage volume, not post-filter neighbors. The same patient-exclusion rule applies throughout the sweep.
Archive readouts are precomputed once per adaptation, and inference ranks eligible stored vectors. This separates readout evaluation from gallery search; the size sweep reports accuracy, not a constant-latency guarantee.

\BfPara{Cache accounting and patient exclusion}
The complete ChestX-ray14 cache partitions into $81{,}664$ gallery, $16{,}384$ query and $14{,}072$ validation images. The corresponding disjoint patient counts are $22{,}437$, $4{,}504$ and $3{,}864$. The $32{,}768$ adaptation images belong to the gallery and are not another disjoint partition. Cache construction and split membership are separate: query and validation images may be encoded for evaluation, but their patient IDs are excluded from the neighbor index. Smaller points in Figure~\ref{fig:archive} retain their raw-cache axis values and apply the same exclusion rule; their eligible counts need not equal those axis values.

\subsection{Feature normalization and norm shift}
\label{app:normshift}

\begin{table}[!ht]
\centering
\begin{minipage}[t]{0.40\linewidth}
\centering
\caption{Encoding variants at the full budget (P@5, ChestX-ray14, 5 seeds; std $\leq0.007$). Default: L2 normalization only. Marks and tints as in Table~\ref{tab:main}.}
\label{tab:encvariant}
\scriptsize
\setlength{\tabcolsep}{3pt}
\begin{tabular}{l|c}
\toprule
\textbf{Preprocessing} & \textbf{P@5} \\
\midrule
\rowcolor{secondbg}
L2 only (default)   & 0.428 \\
\rowcolor{bestbg}
whitening, then L2  & \textbf{0.431} \\
\rowcolor{thirdbg}
PCA-1024, then L2   & 0.426 \\
mean removal, then L2 & 0.419 \\
\bottomrule
\end{tabular}
\end{minipage}
\hfill
\begin{minipage}[t]{0.575\linewidth}
\centering
\caption{Gain and offset shift applied to query features before normalization (P@5, 5 seeds; std $\leq0.008$). Per-dimension gains are drawn around $s$ with a constant offset, so the shift is not annihilated by L2 normalization. Marks and tints as in Table~\ref{tab:main}.}
\label{tab:normshift}
\scriptsize
\setlength{\tabcolsep}{3pt}
\begin{tabular}{l|ccccc}
\toprule
\textbf{Method} & $s{=}0.80$ & $0.90$ & $1.00$ & $1.10$ & $1.25$ \\
\midrule
\rowcolor{secondbg}
Adapter & 0.392 & 0.398 & 0.401 & 0.397 & 0.389 \\
\rowcolor{bestbg}
\textbf{\name} & \textbf{0.421} & \textbf{0.426} & \textbf{0.428} & \textbf{0.425} & \textbf{0.418} \\
$\Delta$ & $+0.029$ & $+0.028$ & $+0.027$ & $+0.028$ & $+0.029$ \\
\bottomrule
\end{tabular}
\end{minipage}
\par\smallskip
\begin{minipage}{\linewidth}
\normalsize\noindent Normalization removes a common positive scale, but it does not remove an additive offset applied before normalization. The paired tables distinguish that invariance from changes that rotate the normalized feature direction.
\end{minipage}\par
\vspace{-3mm}
\end{table}

Amplitude encoding requires strict L2 normalization, raising the concern that the method inherits a sensitivity to feature scale that classical modules avoid.
Table~\ref{tab:encvariant} varies the preprocessing and reports a P@5 range of at most $0.012$ across the preprocessing choices.
Whitening is marginally better than the default ($0.431$ against $0.428$), while PCA-$1024$ is marginally worse ($0.426$) and mean removal is worst ($0.419$). Whitening changes both basis and scale; full-dimensional orthogonal PCA need not discard information. The measured differences therefore probe the structured head's sensitivity to preprocessing, rather than establishing reconstruction loss as their cause.
We keep plain L2 as the default rather than whitening, since the $0.003$ difference is within the observed variation across seeds and whitening needs a second pass over the archive to estimate the covariance, which adds preprocessing cost even when performed locally.
Table~\ref{tab:normshift} evaluates sensitivity to feature gain and offset shifts.
A global rescaling is annihilated by normalization by construction and would not test sensitivity to feature-scale shifts, so we instead draw per-dimension gains around $s$ and add a constant offset, which normalization cannot undo.
Both methods lose accuracy at the extremes and the margin stays within $0.002$ of its unshifted value.
This covers gain and offset shift only.
Figure~\ref{fig:noise}(c) takes up distributional shift of the features.

\subsection{Sequential site adaptation}
\label{app:continual}

\begin{figure}[!ht]
\centering
\includegraphics[width=\linewidth]{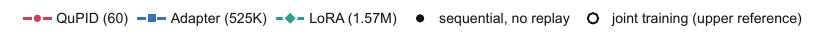}\\[2pt]
\begin{subfigure}[t]{0.48\linewidth}
\centering
\includegraphics[width=\linewidth]{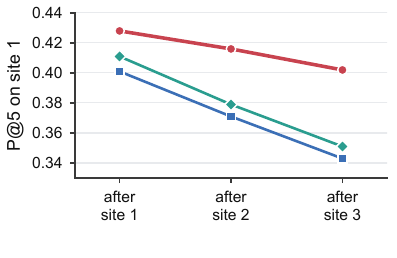}
\caption{Site 1 after each stage.}
\end{subfigure}\hfill
\begin{subfigure}[t]{0.48\linewidth}
\centering
\includegraphics[width=\linewidth]{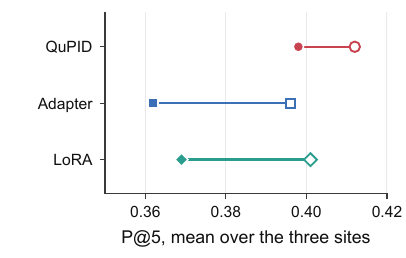}
\caption{Mean over sites against joint training.}
\end{subfigure}
\caption{Sequential adaptation over three sites without replay, re-evaluating site 1 after each stage (5 seeds); filled markers in (b) are the final module averaged over the three sites and hollow markers joint training on the pooled data.}
\label{fig:continual}
\vspace{-3mm}
\end{figure}

\begin{table}[!ht]
\centering
\caption{Sequential adaptation over three sites (ChestX-ray14, CheXpert, MIMIC-CXR) with no replay buffer and no continual-learning method (P@5, 5 seeds; std $\leq0.009$). The first three columns evaluate on site 1 after each stage; \emph{mean} averages the final module over all three sites; \emph{joint} trains one module on the pooled data as an upper reference. Marks and tints as in Table~\ref{tab:main}.}
\label{tab:continual}
\footnotesize
\setlength{\tabcolsep}{4pt}
\begin{tabular}{lr|ccc|c|cc}
\toprule
\textbf{Method} & $p$ & \textbf{site 1} & \textbf{{+}site 2} & \textbf{{+}site 3} & \textbf{forget} & \textbf{mean} & \textbf{joint} \\
\midrule
\rowcolor{thirdbg}
Adapter        & 525K  & 0.401 & 0.371 & 0.343 & \underline{0.058} & 0.362 & 0.396 \\
\rowcolor{secondbg}
LoRA ($r{=}16$) & 1.57M & \underline{0.411} & \underline{0.379} & \underline{0.351} & 0.060 & \underline{0.369} & \underline{0.401} \\
\rowcolor{bestbg}
\textbf{\name} & 60    & \textbf{0.428} & \textbf{0.416} & \textbf{0.402} & \textbf{0.026} & \textbf{0.398} & \textbf{0.412} \\
\bottomrule
\end{tabular}
\vspace{-3mm}
\end{table}

Data-local adaptation implies that sites arrive one after another and that data from an earlier site cannot be revisited, which is the setting in which a stored module must remain effective.
Table~\ref{tab:continual} and Figure~\ref{fig:continual} adapt each module on three sites in sequence with no replay, and re-evaluate site 1 after each stage.
The readout loses $0.026$ P@5 on the first site by the end against $0.058$ and $0.060$ for the two adapters, and its final module is the best of the three on the average over all sites.
This is consistent with a capacity-based explanation, but parameter count alone cannot guarantee less forgetting: even a small module can change its function substantially.
Two scope statements belong with this result.
We apply no continual-learning method, no replay, and no regularization toward previous solutions, so this measures the raw budget effect and not a competitive continual-learning system.
Joint training on pooled data remains better than the sequential result for every method, so this compares deployments that must forget rather than arguing for sequential adaptation.

\section{Compute Parity and Statistical Protocol}
\label{app:parity}

\subsection{Equal wall-clock instead of equal parameters}
\label{app:computematch}

\begin{table}[!ht]
\centering
\caption{Classical baselines given \name's adaptation wall-clock instead of its parameter count (P@5, ChestX-ray14, 5 seeds; std $\leq0.008$). Extra epochs and random restarts are selected by contrastive validation loss. \name is unchanged. Marks and tints follow Table~\ref{tab:main}.}
\label{tab:computematch}
\footnotesize
\setlength{\tabcolsep}{4pt}
\begin{tabular}{l|cc|cc|c}
\toprule
\multirow{2}{*}{\textbf{Method}} & \multicolumn{2}{c|}{$n{=}512$} & \multicolumn{2}{c|}{\textbf{full budget}} & \multirow{2}{*}{\textbf{epochs}} \\
\cmidrule(lr){2-3}\cmidrule(lr){4-5}
 & default & matched & default & matched & \\
\midrule
\rowcolor{thirdbg}
Adapter         & 0.343 & 0.351 & 0.401 & 0.404 & 50 $\to$ 80 \\
\rowcolor{secondbg}
LoRA ($r{=}16$) & \underline{0.347} & \underline{0.354} & 0.411 & 0.414 & 50 $\to$ 68 \\
Adapter-XL      & 0.338 & 0.349 & \underline{0.414} & \underline{0.417} & 50 $\to$ 61 \\
\midrule
\rowcolor{bestbg}
\textbf{\name}  & \multicolumn{2}{c|}{\textbf{0.401}} & \multicolumn{2}{c|}{\textbf{0.428}} & 50 \\
\name\ margin   & $+0.054$ & $+0.047$ & $+0.014$ & $+0.011$ & \\
\bottomrule
\end{tabular}
\par\smallskip
\begin{minipage}{\linewidth}
\normalsize\noindent Extra compute improves each included classical baseline while preserving the measured ordering. The separate rotation-plane comparison below evaluates the stronger $60$-parameter control with validation-selected tuning and its own trial budget.
\end{minipage}\par
\vspace{-3mm}
\end{table}

Table~\ref{tab:cost} states that parameter efficiency is not computational efficiency, and Table~\ref{tab:computematch} addresses the corresponding compute-budget comparison.
Because statevector simulation makes an adaptation epoch roughly $1.6$ times more expensive, a comparison at equal parameter counts allocates less computation time to the baselines.
The three listed classical modules receive the adaptation wall-clock of \name, spent on extra epochs and random restarts selected by contrastive validation loss. The rotation-plane comparison below specifies its own trial budget and retains the same fixed \name\ reference.
The baselines gain $0.003$ to $0.011$ P@5. Against the three modules in this table, our margin at $n=512$ falls from $+0.054$ to $+0.047$ and at the full budget from $+0.014$ to $+0.011$; these margins exclude the controls in Table~\ref{tab:ladder}.
The ordering does not change.
These results distinguish parameter count, the duration of an individual adaptation run and the cost of searching multiple configurations. Their margins apply to the methods and selection budgets stated for each comparison.

\BfPara{Rotation-plane tuning budget}
The rotation-plane tuning study allows $30$ contrastive-validation trials, with a per-trial adaptation cap of $60$ seconds at $n=512$ and $3{,}600$ seconds at full budget, including restarts. The fixed \name\ configuration is retained as the reference; its reported full-budget training time is $50\times71.3=3{,}565$ seconds. These are per-trial ceilings, not a claim of identical total search time. Both methods use the same frozen-feature inputs and validation criterion; cached-feature extraction is outside the adaptation timer.
\begin{center}
\begin{tabular}{rrrr}
\toprule
$n$ & Rotation default & Rotation tuned & QuPID fixed\\
\midrule
512 & .386 & .394 & .401\\
32,768 & .405 & .414 & .428\\
\bottomrule
\end{tabular}
\end{center}
The tuned rotation-plane scores are $0.394$ at $512$ and $0.414$ at full budget, leaving margins $+0.007$ and $+0.014$ against fixed \name. Table~\ref{tab:surrogates} reports paired $95\%$ intervals for this tuned configuration, $[-0.004,+0.018]$ at $512$ and $[+0.006,+0.022]$ at full budget, so the tuned margin excludes zero at the full budget and not at $512$. The intervals below concern the default rotation-plane scores $0.386$ and $0.405$ and are not transferred to the tuned comparison.

\subsection{Variance decomposition and effect sizes}
\label{app:variance}

\begin{table}[!ht]
\centering
\begin{minipage}[t]{0.53\linewidth}
\centering
\caption{Variance components of P@5 from a crossed design of 5 initializations by 5 adaptation-subset draws ($\times10^{-4}$, \name on ChestX-ray14).}
\label{tab:variance}
\scriptsize
\setlength{\tabcolsep}{3pt}
\begin{tabular}{r|ccc|c}
\toprule
$n$ & \textbf{subset} & \textbf{init.} & \textbf{inter.} & \textbf{subset share} \\
\midrule
512     & 6.8 & 1.9 & 0.7 & $72\%$ \\
32{,}768 & 0.9 & 1.1 & 0.3 & $39\%$ \\
\bottomrule
\end{tabular}
\end{minipage}
\hfill
\begin{minipage}[t]{0.445\linewidth}
\centering
\caption{Effect sizes for per-query paired scores, alongside the significance test in Section~\ref{sec:setup}. Win rate is the share of queries on which \name scores strictly higher.}
\label{tab:effect}
\scriptsize
\setlength{\tabcolsep}{3pt}
\begin{tabular}{lr|cc}
\toprule
\textbf{Comparison} & $n$ & \textbf{Cliff's $\delta$} & \textbf{win rate} \\
\midrule
vs.\ sweep baseline & 512 & 0.44 & 0.67 \\
vs.\ Adapter-XL & 32{,}768 & 0.21 & 0.58 \\
\bottomrule
\end{tabular}
\end{minipage}
\vspace{-3mm}
\end{table}

A seed in our protocol controls initialization, augmentation, and the adaptation-subset draw at once, which is suitable for reporting aggregate variability but does not separate its sources.
Table~\ref{tab:variance} separates the two factors with a crossed design of five initializations by five subset draws.
At $n=512$ the subset draw accounts for $72\%$ of the variance and initialization for most of the remainder; at the full budget the two are comparable.
Query bootstrapping conditional on five fitted runs omits variation from drawing a new adaptation set. The paired analysis uses repeated subset and initialization draws for \name\ and the default rotation-plane head, using held-out patients as clusters.

Table~\ref{tab:effect} reports effect sizes next to the paired bootstrap of Section~\ref{sec:setup}, since Holm correction controls family-wise type-I error, while rejection alone does not quantify the size of a difference.
Cliff's $\delta$ is $0.44$ at $n=512$, a medium effect, and $0.21$ at the full budget, a small one; the per-query win rates are $0.67$ and $0.58$.
Both orderings are consistent with the stated claim, and the statistically significant full-budget difference remains small: a win rate of $0.58$ leaves roughly two queries in five as losses or ties, so the average gain is not a per-query guarantee.
The two statistics answer different questions and we report both deliberately.
The bootstrap and effect sizes are conditional summaries for the included comparators. They neither describe new adaptation draws nor establish a paired advantage over the rotation-plane head. Multiple queries from one patient further invalidate an independent-query interpretation of population uncertainty.
A third perspective concerns the deployment setting.
Since the comparison is paired per query, the practically relevant quantity is the gain conditional on the adapter failing, and there \name recovers a relevant case in the top five on $0.31$ of the queries where the adapter retrieves none at $n=512$, against $0.19$ in the reverse direction.

\BfPara{Paired estimand and evaluation units}
The estimand is the mean P@5 difference between \name\ and the default rotation-plane head over adaptation draws, initialization and held-out patients. Each method is evaluated on the same adaptation subset and query manifest. The point estimate averages queries equally within each fitted run and then averages the crossed runs; patient IDs define clusters for the uncertainty calculation.
The design crosses five adaptation-subset draws with five initialization draws, giving $25$ paired fitted-run comparisons per budget. Pairing retains one score per query, method, subset draw and initialization.
The supplementary resampling implementation supports nested and crossed designs. For nested initialization it resamples subset IDs followed by paired initialization replicates; for crossed initialization it resamples both factor levels independently~\citep{owen2007pigeonhole}. Crossed cells are not treated as independent training replicates.
Patient resampling retains all images of a sampled patient and uses the same draw for both methods and all fitted runs. This preserves within-patient dependence while maintaining query-weighted point estimates.
The numerical intervals reported below use a normal approximation to the measured variance components. The supplementary code also implements percentile resampling; the displayed intervals are not percentile-bootstrap output.
The analysis conditions on the gallery, validation split and default hyperparameters; variability from the separate tuning search is outside its scope.

\makeatletter
\setlength{\@fptop}{0pt}
\makeatother

\noindent\begin{minipage}{\linewidth}
\BfPara{Normal-approximation intervals}
The analysis uses the $5\times5$ crossed runs and $4{,}504$ held-out patient clusters. At the reported precision, the rotation-plane variance components equal those of \name, a common between-method correlation estimate $\rho=0.5$ is used for the training components, and the patient-level contribution to the difference variance is $v_p=0.01$. With subset, initialization and interaction components $v_s,v_i,v_{si}$ from Table~\ref{tab:variance}, the estimated variance of the mean difference is $2(1-\rho)(v_s/5+v_i/5+v_{si}/25)+v_p/4{,}504$. The intervals use the mean difference plus or minus $1.96$ standard errors.
\begin{center}
\begin{tabular}{rrrr}
\toprule
$n$ & Mean margin & SE & Normal 95\% interval\\
\midrule
512 & +0.015 & 0.0134 & $[-0.0112,+0.0412]$ \\
32,768 & +0.023 & 0.0066 & $[+0.0101,+0.0359]$ \\
\bottomrule
\end{tabular}
\end{center}
At $n=512$, the interval includes zero despite the positive mean margin, so this analysis does not establish a positive difference at $95\%$. The full-budget interval is positive. Both intervals summarize the default paired comparison under the stated approximation and fixed evaluation split.
\end{minipage}\par

\section{Design Choices and Empirical Scope}
\label{app:objections}

The circuit construction, training protocol and deployment path determine the scope of the reported evidence.
Tables~\ref{tab:objections} and~\ref{tab:objections2} connect each design choice and empirical comparison to its supporting result.
The central contribution is a structured feature map that combines a small trainable budget with competitive retrieval quality. Larger adapters assess the accuracy attainable at a much greater budget; equally compact classical heads assess the chosen parameterization. All comparisons use classical simulation and do not establish computational quantum advantage.
The rotation-plane sweep measures accuracy across adaptation sizes, its separate tuning study reports validation-selected scores, and the paired intervals quantify uncertainty for the default endpoints. The full-budget interval is positive, while the $512$ interval includes zero. Split counts distinguish stored cache size from the eligible gallery.

\begin{table}[!ht]
\centering
\caption{Design choices, empirical interpretation and supporting results (I: model, controls, trainability, cost, deployment, encoding and statistics). References point to sections and appendices in this paper; Table~\ref{tab:objections2} continues the list.}
\label{tab:objections}
\footnotesize
\setlength{\tabcolsep}{4pt}
\renewcommand{\arraystretch}{1.12}
\begin{tabular}{@{}l@{\hspace{6pt}}l@{\hspace{6pt}}l@{}}
\toprule
\textbf{Aspect} & \textbf{Interpretation} & \textbf{Evidence} \\
\midrule
\parbox[t]{0.225\linewidth}{\raggedright Circuit-defined model class\strut} & \parbox[t]{0.545\linewidth}{\raggedright The circuit shares $60$ parameters across $M=40$ input-modulated, correlated full-rank quadratic forms. Its contribution is this compact retrieval feature map, supported by competitive ranking quality under classical simulation.\strut} & \parbox[t]{0.165\linewidth}{\raggedright Sec.~\ref{sec:intro}, \ref{sec:readout}, \ref{sec:ablation}; App.~\ref{app:qml_trends}\strut} \\[3pt]
\parbox[t]{0.225\linewidth}{\raggedright Classical control budgets\strut} & \parbox[t]{0.545\linewidth}{\raggedright The rotation-plane head matches $60$ parameters, isometry and output width. Its input-independent rank-one forms differ from the input-modulated full-rank Pauli forms. The comparison evaluates these choices jointly; RFF and the four-projection Quadratic head are more restricted controls.\strut} & \parbox[t]{0.165\linewidth}{\raggedright Table~\ref{tab:main}; App.~\ref{app:tuned}, \ref{app:capacity}, \ref{app:computematch}\strut} \\[3pt]
\parbox[t]{0.225\linewidth}{\raggedright Contribution beyond frozen features\strut} & \parbox[t]{0.545\linewidth}{\raggedright The backbone initialization is shared; post-encoder heads keep features frozen, while LoRA updates attention projections; the untrained circuit and the naive fidelity design are reported as controls, and the naive design exactly matches the frozen baseline as the theory requires; all $60$ parameters are updated.\strut} & \parbox[t]{0.165\linewidth}{\raggedright Fig.~\ref{fig:mech}(a); Table~\ref{tab:ablation}; App.~\ref{app:degeneracy_numeric}, \ref{app:diagnostics}\strut} \\[3pt]
\parbox[t]{0.225\linewidth}{\raggedright Finite-scale trainability\strut} & \parbox[t]{0.545\linewidth}{\raggedright Depth $L\leq6$, $n_q=10$ fixed by the feature dimension, local one- and two-body observables; gradient variance is a finite diagnostic, not an asymptotic rate. Classical simulation differentiates the statevector; parameter shift provides the gate-level hardware cost model.\strut} & \parbox[t]{0.165\linewidth}{\raggedright Sec.~\ref{sec:training}; Fig.~\ref{fig:diagnostics}; Table~\ref{tab:gradvar}\strut} \\[3pt]
\parbox[t]{0.225\linewidth}{\raggedright Simulation cost and scaling\strut} & \parbox[t]{0.545\linewidth}{\raggedright The statevector dimension is $2^{n_q}=d$, with runtime also scaling in depth, gate count and readout width; per-epoch training time is $1.60\times$ that of a light adapter, inference time and memory usage are comparable, and the qubit count is swept rather~than~assumed.\strut} & \parbox[t]{0.165\linewidth}{\raggedright Table~\ref{tab:cost}; Fig.~\ref{fig:circuit}(a); App.~\ref{app:circuit}\strut} \\[3pt]
\parbox[t]{0.225\linewidth}{\raggedright GPU deployment and noise simulation\strut} & \parbox[t]{0.545\linewidth}{\raggedright The intended system trains by classical circuit simulation and runs fixed-parameter GPU inference with cached archive readouts. No quantum processor is required, as in quantum parameter-efficient fine-tuning~\citep{koikeakino2025quantumpeft,liu2025qpa}. Noise and shot studies are simulations of an optional hardware realization, not its validation.\strut} & \parbox[t]{0.165\linewidth}{\raggedright Sec.~\ref{sec:related}, \ref{sec:noise}; Fig.~\ref{fig:noise}(a,b); App.~\ref{app:qml_trends}, \ref{app:noise}\strut} \\[3pt]
\parbox[t]{0.225\linewidth}{\raggedright Amplitude encoding in simulation\strut} & \parbox[t]{0.545\linewidth}{\raggedright In simulation it is a normalization; its real costs, strict L2 normalization and power-of-two dimensions, are stated, with the padding a $768$-dimensional backbone would need.\strut} & \parbox[t]{0.165\linewidth}{\raggedright Sec.~\ref{sec:limitations}; App.~\ref{app:qml_amplitude}, \ref{app:normshift}\strut} \\[3pt]
\parbox[t]{0.225\linewidth}{\raggedright Effect sizes and uncertainty\strut} & \parbox[t]{0.545\linewidth}{\raggedright Query-bootstrap markers condition on fitted runs. Normal intervals against the default rotation-plane head include adaptation, initialization and patient variation: positive at full budget, inconclusive at $512$. No paired interval is reported for the tuned control.\strut} & \parbox[t]{0.165\linewidth}{\raggedright Fig.~\ref{fig:mech}(b), \ref{fig:sensitivity}; App.~\ref{app:variance}\strut} \\[3pt]
\bottomrule
\end{tabular}
\vspace{-3mm}
\end{table}

\begin{table}[!ht]
\centering
\caption{Design choices, empirical interpretation and supporting results (II: theory, surrogates, circuit design, shots, datasets, low-data accuracy, generation and scope).}
\label{tab:objections2}
\footnotesize
\setlength{\tabcolsep}{4pt}
\renewcommand{\arraystretch}{1.12}
\begin{tabular}{@{}l@{\hspace{6pt}}l@{\hspace{6pt}}l@{}}
\toprule
\textbf{Aspect} & \textbf{Interpretation} & \textbf{Evidence} \\
\midrule
\parbox[t]{0.225\linewidth}{\raggedright Scope of the generalization bound\strut} & \parbox[t]{0.545\linewidth}{\raggedright The bound applies to bounded smooth parameter classes and concerns contrastive risk. It does not certify the reported runs or predict P@5 gaps. The empirical claim combines a small budget with competitive retrieval; the gap ratios are separate diagnostics.\strut} & \parbox[t]{0.165\linewidth}{\raggedright Sec.~\ref{sec:training}; App.~\ref{app:genproof}, \ref{app:gap}\strut} \\[3pt]
\parbox[t]{0.225\linewidth}{\raggedright Classical surrogate comparisons\strut} & \parbox[t]{0.545\linewidth}{\raggedright The default rotation-plane head leaves full/512-example margins of $+0.023$/$+0.015$; tuning reduces them to $+0.014$/$+0.007$. Appendix~\ref{app:surrogates} extends this to surrogates with far larger budgets and more tuning; without input modulation they stay between $0.409$ and $0.417$ P@5 up to $6{,}600$ parameters. These measured comparisons do not establish a classical expressivity separation (App.~\ref{app:computematch}).\strut} & \parbox[t]{0.165\linewidth}{\raggedright Sec.~\ref{sec:ablation}; Table~\ref{tab:main}, \ref{tab:surrogates}\strut} \\[3pt]
\parbox[t]{0.225\linewidth}{\raggedright Circuit and readout selection\strut} & \parbox[t]{0.545\linewidth}{\raggedright Real rotations keep the state real and the simulation computationally inexpensive; qubit count, entangling topology, readout width and depth are each swept and saturate at the default, so the operating point is a plateau.\strut} & \parbox[t]{0.165\linewidth}{\raggedright Fig.~\ref{fig:mech}(a), \ref{fig:circuit}; App.~\ref{app:qml_pqc}, \ref{app:circuit}\strut} \\[3pt]
\parbox[t]{0.225\linewidth}{\raggedright Finite-shot readout\strut} & \parbox[t]{0.545\linewidth}{\raggedright A sufficient Hoeffding ranking guarantee is derived with explicit norm-floor, margin and collection-size dependence. Separately, the finite-shot experiment is within $0.002$ P@5 of exact expectations at $10^4$ shots per observable.\strut} & \parbox[t]{0.165\linewidth}{\raggedright Sec.~\ref{sec:noise}; App.~\ref{app:noise}\strut} \\[3pt]
\parbox[t]{0.225\linewidth}{\raggedright Dataset coverage\strut} & \parbox[t]{0.545\linewidth}{\raggedright ChestX-ray14 ($112{,}120$ images) and MURA for retrieval, IU X-Ray and MIMIC-CXR for generation, MVTec AD and CUB-200 for generality.\strut} & \parbox[t]{0.165\linewidth}{\raggedright Sec.~\ref{sec:generality}; App.~\ref{app:setup}\strut} \\[3pt]
\parbox[t]{0.225\linewidth}{\raggedright Low-data accuracy and degradation\strut} & \parbox[t]{0.545\linewidth}{\raggedright The sweep tests retrieval quality as adaptation data shrink. Default rotation-plane margins are $+0.023$ at full budget and $+0.015$ at 512. The control degrades less; the results support observed accuracy at a small budget, without establishing superior stability.\strut} & \parbox[t]{0.165\linewidth}{\raggedright Sec.~\ref{sec:training}, \ref{sec:dataefficiency}; App.~\ref{app:gap}, \ref{app:tuned}\strut} \\[3pt]
\parbox[t]{0.225\linewidth}{\raggedright Retrieval-grounded generation\strut} & \parbox[t]{0.545\linewidth}{\raggedright The pipeline is fixed and stated: cached archive readouts, top-$k$ reports concatenated with the query image, a frozen generator with greedy decoding. Random and label-oracle contexts provide an empirical reference range. Clinical-efficacy proxies assess factual content, and per-query costs are reported for both methods.\strut} & \parbox[t]{0.165\linewidth}{\raggedright Sec.~\ref{sec:generation}; Table~\ref{tab:clinical}, \ref{tab:cost}; App.~\ref{app:utility}\strut} \\[3pt]
\parbox[t]{0.225\linewidth}{\raggedright Interpretation boundaries\strut} & \parbox[t]{0.545\linewidth}{\raggedright Accuracy claims concern the stated ranking metrics and comparison budgets. Raw cache size differs from the patient-filtered gallery. Data-local GPU execution is not a formal privacy guarantee.\strut} & \parbox[t]{0.165\linewidth}{\raggedright Sec.~\ref{sec:mainresults}, \ref{sec:limitations}\strut} \\[3pt]
\parbox[t]{0.225\linewidth}{\raggedright Relation to quantum adaptation\strut} & \parbox[t]{0.545\linewidth}{\raggedright Quantum-PEFT and quantum parameter adaptation act on the weight side and never encode data; projected kernels measure a fixed encoding. Here data are encoded and re-uploaded, and the circuit learns a retrieval feature map; Proposition~\ref{prop:degeneracy} shows why measurement escapes shared-unitary overlap invariance.\strut} & \parbox[t]{0.165\linewidth}{\raggedright Sec.~\ref{sec:related}; App.~\ref{app:qml_trends}\strut} \\[3pt]
\parbox[t]{0.225\linewidth}{\raggedright Background and implementation\strut} & \parbox[t]{0.545\linewidth}{\raggedright A self-contained background appendix introduces the ingredients in method order and provides a runnable listing.\strut} & \parbox[t]{0.165\linewidth}{\raggedright App.~\ref{app:qml_background}\strut} \\[3pt]
\bottomrule
\end{tabular}
\vspace{-3mm}
\end{table}

\end{document}